\documentclass[]{onethree}

\usepackage[toc,page,header]{appendix}
\usepackage{minitoc}
\usepackage{amsmath,amssymb}
\usepackage{bm}
\usepackage{wrapfig}
\usepackage{enumitem}
\usepackage[utf8]{inputenc} %
\usepackage[T1]{fontenc}    %
\usepackage{hyperref}       %
\usepackage{url}            %
\usepackage{booktabs}       %
\usepackage{amsfonts}       %
\usepackage{nicefrac}       %
\usepackage{microtype}      %
\usepackage{lipsum}		%
\usepackage{graphicx}
\usepackage{natbib}
\usepackage{doi}
\usepackage{booktabs,tabularx,subcaption,multirow,xcolor,colortbl,adjustbox,longtable}

\usepackage{amsthm}
\usepackage{float}
\newtheorem{lemma}{Lemma}
\newtheorem{theorem}{Theorem}

\usepackage{pdfpages}

\title{Learning from Fleet Experience: \\ Collective Experience Learning across 96 Robots}

\author{Tsinghua University}
\author{AgiBot}
\author{Infinigence AI}

\affiliation{For the complete list of authors, please refer to the \protect\nameref{sec:contibution} section}

\checkdata[Resources]{
  \href{https://cel-robot-learning.github.io}{\raisebox{-1ex}{\includegraphics[height=4.2ex]{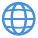}}~Project Page}\hspace{2em}

}

\begin{document}
\includepdf[
    pages=1,
    fitpaper=false,
    pagecommand={\thispagestyle{empty}}
]{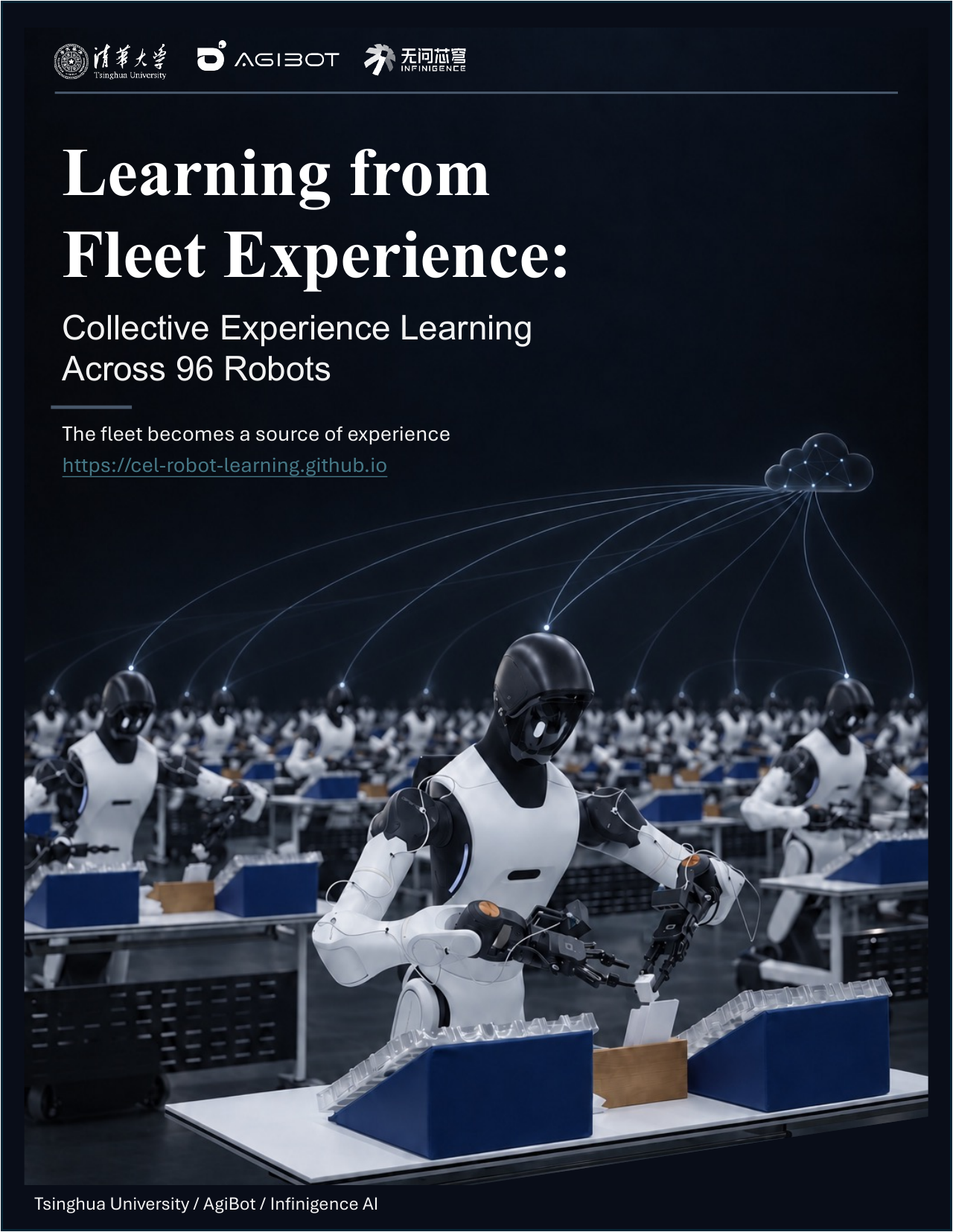}

\includepdf[
    pages=1,
    fitpaper=false,
    pagecommand={\thispagestyle{empty}}
]{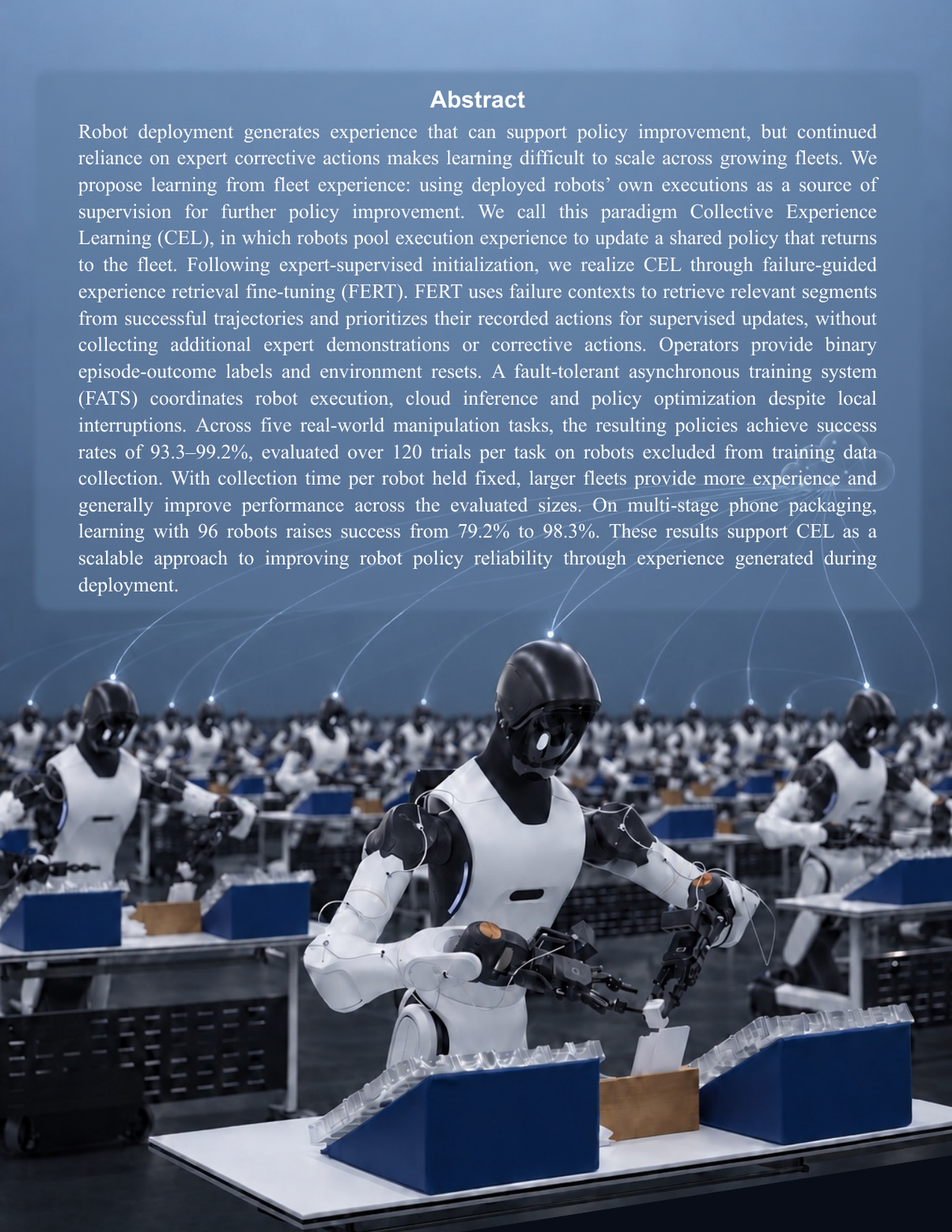}

\setcounter{page}{2}


\newpage

\tableofcontents
\newpage

\section{Introduction}
\begin{figure}[htbp]
  \centering
  \includegraphics[width=\linewidth]{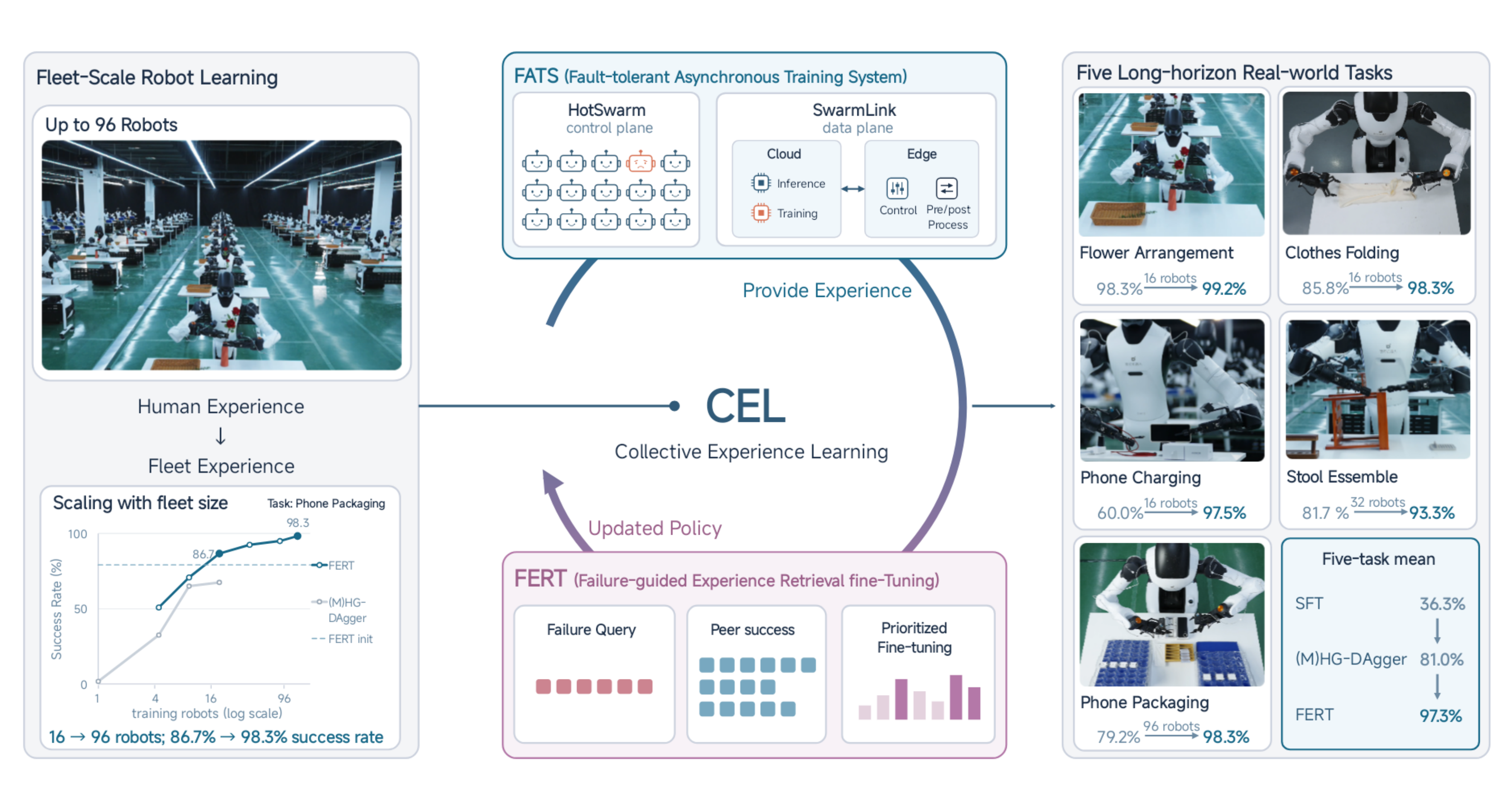}
  \caption{Overview of Collective Experience Learning (CEL). We propose a new training paradigm CEL, so that robots improve from fleet experience rather than continued human experience, enabling large-scale online training with up to 96 robots. We develop a training algorithm FERT under CEL paradigm, which uses failure contexts to retrieve relevant successful segments and prioritize their recorded actions for supervised updates. To support scalable online training, FATS coordinates fault-tolerant asynchronous execution, cloud inference and policy optimization. Across five real-world manipulation tasks, policies improve after FERT, and success rates generally scale with fleet size.}
  \label{fig:overview}
\end{figure}

Robot foundation models provide initial capabilities by translating visual and language understanding into physical action~\citep{brohan2022rt, brohan2023rt, kim2024openvla, team2024octo, black2024pi0}, but these capabilities do not automatically yield reliable execution~\citep{o2024open, khazatsky2024droid, bu2025agibot}. Variations in object properties, contact conditions and individual robots expose policy weaknesses, particularly when errors accumulate over multi-stage tasks~\citep{laskey2017dart, zhao2023learning}. Real-world execution also generates experience that can support further policy improvement~\citep{chen2025conrft, intelligence2025pi06vlalearnsexperience}. As deployment expands, parallel robot operation provides a growing source of such experience~\citep{levine2018learning, kalashnikov2018scalable, kalashnikov2021mt, pan2026sop, wang2026lwd}. The challenge is to turn this experience into more reliable task execution.

A central constraint is the human effort required to supply action supervision. Expert demonstrations and corrective interventions have proved useful for improving policies during real-world execution~\citep{ross2011reduction,kelly2019hg,luo2024hilserl,liu2025robot,intelligence2025pi06vlalearnsexperience,hoque2021thriftydagger}, and recent fleet-learning systems combine parallel collection with human supervision~\citep{hoque2023fleet,pan2026sop,wang2026lwd}. When policy improvement depends on new expert actions, teleoperators usually remain available to monitor execution and take control when corrections are needed. This capacity is difficult to expand alongside robot interaction~\citep{hoque2023fleet,hoque2021thriftydagger}, motivating a learning process that draws subsequent action supervision from the fleet itself.

We propose \textbf{learning from fleet experience}: shifting the source of action supervision for continued policy improvement from newly collected expert actions to deployed robots' own executions. We call this paradigm \onethreeblue{Collective Experience Learning} (CEL), in which robots pool execution experience to improve a shared policy that returns to the fleet, as shown in ~\Cref{fig:overview}. This closes the loop between distributed execution and shared policy improvement, allowing experience generated by individual robots to benefit the fleet. Following expert-supervised initialization, our realization of CEL requires no additional expert demonstrations or corrective actions.

The algorithmic challenge is to identify which experience addresses the policy's current weaknesses. Collecting more successful trajectories alone does not determine which behaviours should receive greater emphasis. We develop \textbf{Failure-guided Experience Retrieval fine-Tuning} (FERT), which uses contexts from failed executions to retrieve relevant segments from a shared buffer of successful trajectories. The retrieved segments receive increased sampling priority, and their recorded actions serve as targets for supervised policy updates. Successful executions thus supply actions to imitate, while failures guide the allocation of learning effort. In our experiments, operators provide binary episode-outcome labels and reset environments; the actions used for FERT updates come from autonomous robot execution.

The system challenge is to sustain this learning process when robots progress at different rates and pause for resets, maintenance or recovery~\citep{pan2026sop,wang2026lwd}. We develop a \textbf{Fault-tolerant Asynchronous Training System} (FATS) that coordinates robot execution, cloud inference and policy optimization without requiring a common collection schedule. Its data plane, SwarmLink, pools inference requests and sends each episode to training as one record once it is labeled. Its control plane, HotSwarm, restarts failed robot workers in place while available robots continue collecting experience.

We evaluate CEL through FERT and FATS across five real-world manipulation tasks involving deformable objects, precise alignment, tool use and multi-stage execution, with training fleets of up to \onethreeblue{96 robots}. Each evaluated policy is tested over 120 trials, split equally between two robots excluded from training data collection. At the largest evaluated fleet size for each task, the resulting policies achieve \textbf{success rates of 93.3--99.2\%}. With collection time per robot held fixed, larger fleets provide more experience and generally improve performance across the evaluated sizes. On multi-stage phone packaging, learning with 96 robots raises success from 79.2\% to 98.3\%. During the 96-robot run, local fault events occurred on average every 22 minutes, yet fleet-level experience collection continued, with fault-related downtime accounting for only 1.3\% of total robot time. Together, these results support CEL as a scalable learning paradigm in which expanding robot deployment provides a growing source of experience for improving shared policies, without requiring a corresponding expansion of expert action supervision.

\begin{figure}[htbp]
    \centering
    \includegraphics[width=0.85\linewidth]{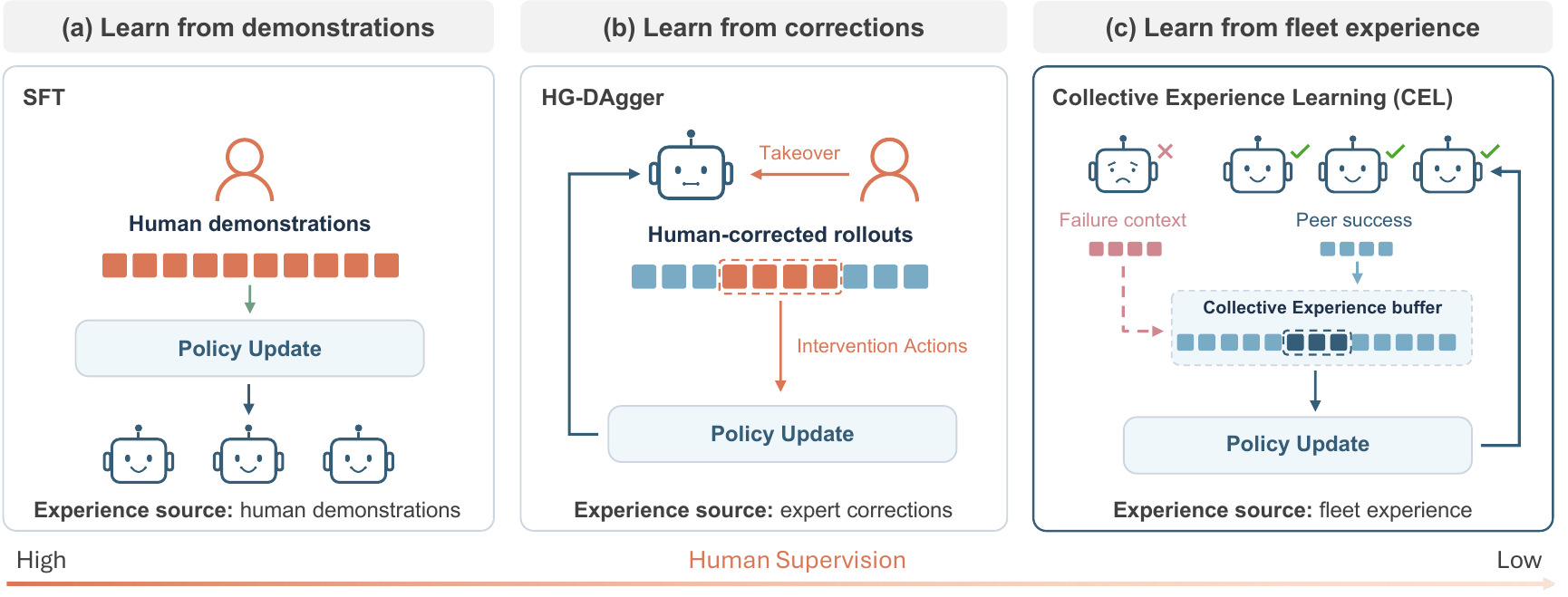}
    \caption{Sources of action supervision for robot policy learning: human demonstrations, expert corrections and fleet experience. CEL uses robot-generated actions for further policy improvement after expert-supervised initialization. Our implementation retains human-provided episode-outcome labels and environment resets.}
    \label{fig:cel}
\end{figure}

\section{Preliminaries}
\label{sec:preliminaries}

\subsection{Robot policy}
\label{sec:vla_training}

Vision-language-action models combine visual observations, language instructions and action generation~\cite{kim2024openvla,black2024pi0}. We consider a robot policy conditioned on visual observations, proprioceptive measurements and a language instruction. Let $o_t$ denote the robot observation at timestep $t$ and $g$ the task instruction. We write the policy as $\pi_\theta(A_t\mid o_t,g)$, where $A_t=(a_t,\ldots,a_{t+H-1})$ is an action chunk with prediction horizon $H$ and $\theta$ denotes the policy parameters.

Our implementation uses $\pi_0$, a flow-based vision-language-action model that generates continuous action chunks from these inputs~\cite{black2024pi0}.

\subsection{Problem setting}
\label{sec:problem_setting}

We consider policy learning during deployment for a task specified by a
language instruction $g$, using experience collected during
robot fleet operation. Robots assigned to the same task
use a shared policy, although individual robots may execute
different policy versions as learning proceeds.

Each execution produces a trajectory
$\tau=\{(o_t,a_t)\}_{t=0}^{T-1}$, where $o_t$ contains multi-view
images and proprioceptive measurements, $a_t$ is the executed
action, and $T$ is the episode length. A binary outcome
$y(\tau)\in\{0,1\}$ indicates task success. We suppress the
dependence of $y$ on $g$ for brevity. This outcome evaluates the execution without providing expert action targets. In our FERT experiments, operators supply outcome labels and reset environments; the actions used for subsequent learning come from robot execution.

For a fixed task $g$, our objective is to maximize the
probability of successful execution, which is the measure of policy reliability used in our task evaluations:
\begin{equation}
    J(\theta)
    =
    \mathbb{E}_{\tau\sim\pi_\theta}
    \left[y(\tau)\right].
    \label{eq:task_objective}
\end{equation}
Here, $\tau\sim\pi_\theta$ denotes a trajectory generated by executing the policy in the environment. We suppress task conditioning for brevity.

Experience collected across the fleet comprises execution
trajectories and their associated outcomes, potentially
generated by different policy versions. The learning algorithm
determines which trajectories are retained and how they
contribute to policy updates.

Starting from a task-initialized policy, CEL uses fleet-generated experience to improve
the policy without collecting additional expert demonstrations
or corrective actions.
Expert data may be used to establish the initial policy.
CEL specifies how fleet-generated experience is shared for continued policy improvement. FERT is the concrete learning algorithm studied here, and FATS supplies the execution and communication infrastructure. The task-success objective states what the policy should improve; it does not prescribe a particular reinforcement-learning or supervised-learning update.

\section{Collective Experience Learning}
\label{sec:cel}

Collective Experience Learning (CEL) is a learning paradigm in which robots share experience generated during deployment to improve a common policy. Updated policies return to the fleet for subsequent execution, connecting experience collection and shared policy improvement in a closed loop. \Cref{fig:cel} illustrates the different sources of experience for policy learning: learning from human demonstrations, learning from expert corrections, and learning from fleet experience generated by deployed robots. The requirement for expert action supervision decreases from left to right.

We realize CEL through an algorithm and a training system. Following expert-supervised initialization, failure-guided experience retrieval fine-tuning (FERT) draws action targets from successful robot trajectories and uses failure contexts to prioritize relevant successful experience for supervised policy updates, without collecting additional expert demonstrations or corrective actions. In this realization, expanding experience collection does not require a corresponding expansion of expert teleoperation. Binary episode-outcome labels distinguish successful and failed executions; humans provide these labels and reset environments, while subsequent action supervision comes from robot executions. The fault-tolerant asynchronous training system (FATS) coordinates robot execution, cloud inference and policy optimization across robot and computing resources. Its asynchronous data paths and local recovery support the learning loop despite differences in execution speed and interruptions to individual workers.

\begin{figure}[htbp]
    \centering
    \includegraphics[width=0.9\linewidth]{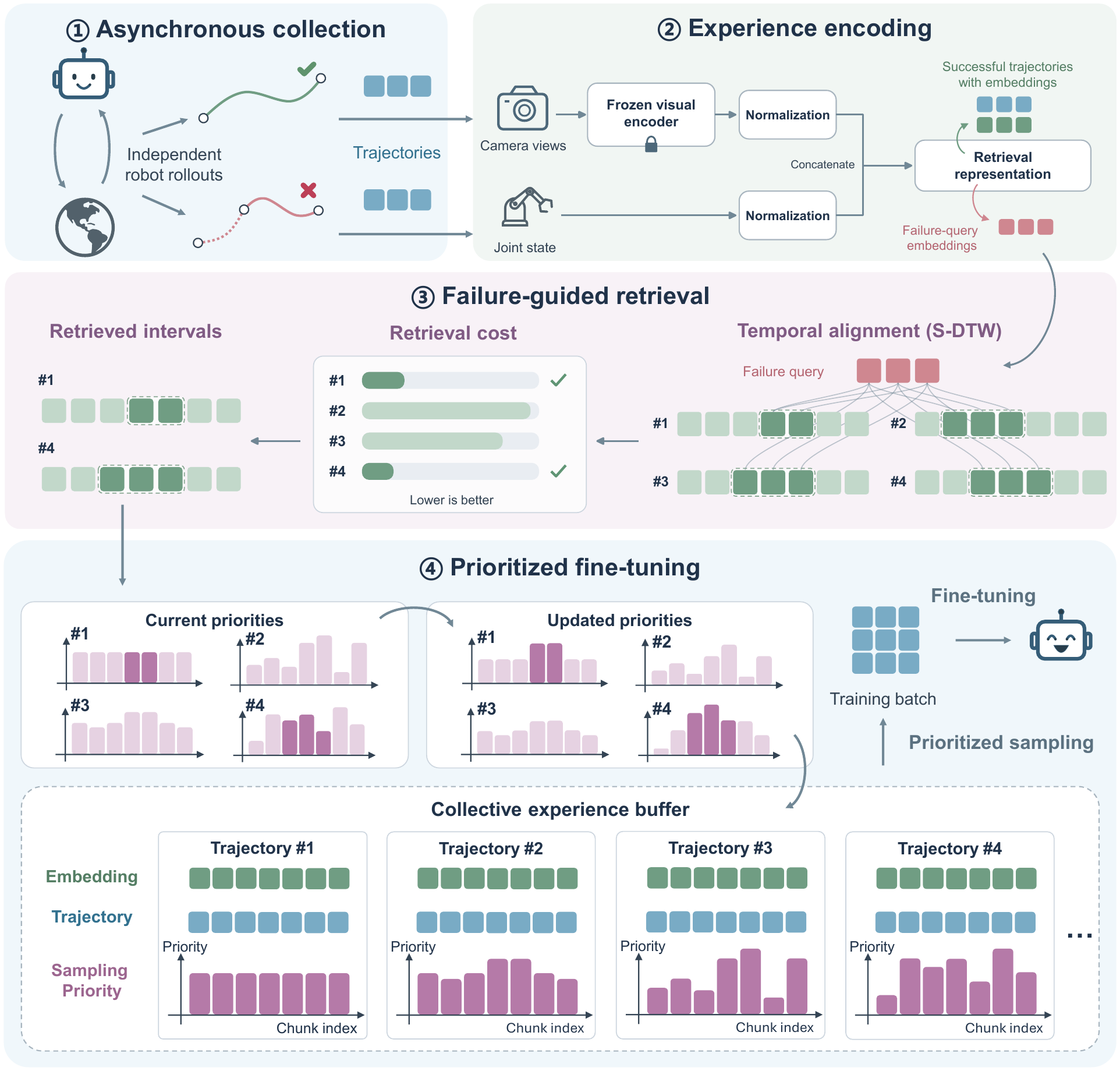}
    \caption{\textbf{Failure-guided experience retrieval fine-tuning.} Robots asynchronously collect autonomous experience with binary task-outcome labels. Observations are encoded using visual and proprioceptive features. Failure contexts retrieve related segments from successful trajectories, and these segments receive increased sampling priority during policy fine-tuning. Updated policies return to the fleet to generate further experience.}
    \label{fig:FERT}
\end{figure}
\subsection{Algorithm: Failure-guided Experience Retrieval Fine-Tuning (FERT)}
\label{sec:algorithm}

FERT realizes CEL by updating the policy from experience generated during deployment, without collecting further expert demonstrations or corrective actions. It organizes autonomous execution data for retrieval and supervised policy fine-tuning. Binary task outcomes distinguish successful executions from failures, but do not directly specify which actions would resolve the policy's difficulties. FERT addresses this problem by using successful experience as a source of action targets and failure contexts as queries for selecting relevant successful segments. Increasing the sampling priority of these segments directs learning towards behaviour related to the policy's current difficulties. Successes therefore supply the actions to imitate, while failures determine where learning concentrates.

In our experiments, FERT starts from the selected policy obtained through SFT and multi-robot HG-DAgger (\Cref{app:policy_preparation}). This separates expert-supervised initialization from subsequent learning using fleet-generated actions. 
FERT comprises four logical steps: asynchronous experience collection, experience encoding, failure-guided retrieval and prioritized fine-tuning (\Cref{fig:FERT}). Collection proceeds concurrently with optimization, while incoming experience is incorporated through buffer insertion, retrieval and priority updates between training updates. Updated policies return to the fleet and generate further experience.

\subsubsection{Step 1: Asynchronous experience collection}

FERT continuously collects experience through autonomous
fleet execution. Each completed trajectory $\tau$ receives
a binary task outcome $y(\tau)$, following the setting in
\Cref{sec:problem_setting}. 

Successful trajectories collected across the fleet are retained
in a collective experience buffer $\mathcal{D}^{+}$, which supplies
action targets for policy fine-tuning. Failed trajectories are used as temporary retrieval queries; their actions are not used as supervision. This separation allows FERT to use failures to guide learning without imitating unsuccessful actions. The active buffer retains a rolling window of successful experience, bounding the history available for training and retrieval.

Completed trajectories are received asynchronously, allowing robots to contribute experience at different rates while policy optimization continues. Buffer insertion, retrieval and priority updates are performed between training updates. Data balancing regulates uneven accumulation across training workers.

\subsubsection{Step 2: Experience encoding}

To retrieve useful behaviour across heterogeneous executions, FERT requires a representation that captures both the scene and the robot's configuration. We therefore combine visual features with proprioceptive state. The visual encoder $f_{\bar{\theta}}$ is a frozen copy of vision tower in the vision-language-action model, initialized from the checkpoint at the start of this stage. Freezing the encoder keeps incoming failure queries and cached successful trajectories in a consistent representation space as the policy is updated.

For each camera view, the encoder produces a $16\times16$ grid of visual tokens. We apply spatial average pooling with a $4\times4$ kernel and stride four, retaining a $4\times4$ grid of pooled features per view. These features are flattened and concatenated in a fixed camera order:
\begin{equation}
    h_t=\text{Concat}_{c=0}^{C-1}\left[\text{vec}\left(\text{Pool}\left(f_{\bar{\theta}}(I_t^{(c)})\right)\right)\right].
\end{equation}
Here, $C$ is the number of camera views and $I_t^{(c)}$ is the preprocessed image from camera $c$. Retaining a coarse spatial grid preserves spatial information while reducing the embedding length.

Let $s_t$ denote the normalized joint state. The retrieval representation is
\begin{equation}
    z_t=\left[\frac{h_t}{\max(\|h_t\|_2,\epsilon)};\frac{s_t}{\max(\|s_t\|_2,\epsilon)}\right],
    \label{eq:fert_representation}
\end{equation}
where $\epsilon$ prevents division by zero. Separate normalization places the visual and state components on comparable scales. This combination provides both scene and posture context when comparing experience collected across physical robots.

Embeddings of complete successful trajectories are computed upon ingestion and cached with their trajectories. For a failed trajectory, only the failure-context query clip defined in Step 3 is encoded.

\subsubsection{Step 3: Failure-guided retrieval}
\label{sec:ret}

A binary failure label does not precisely locate the onset of failure. FERT therefore uses a short observation sequence preceding the recorded endpoint as a retrieval query, providing temporal context for finding related successful behaviour without requiring frame-level failure annotations. A failure context from one robot can retrieve successful
experience contributed by other robots through the collective experience buffer.

We construct a query $Q_\tau$ from observations preceding the endpoint of a failed trajectory $\tau$:
\begin{equation}
    e=T-\delta,\qquad b=\max(0,e-L),\qquad Q_\tau=(o_b,\ldots,o_{e-1}),
    \label{eq:fert_query}
\end{equation}
where $L$ is the maximum query length and $\delta$ is an endpoint offset. The offset allows the query to exclude terminal frames that may reflect a delayed failure indication or behaviour after the task has already failed. Empty queries are skipped. We use $L=90$ and $\delta=10$ frames, corresponding to up to three seconds of context ending approximately one-third of a second before the recorded endpoint at $30\,\mathrm{Hz}$.

Given a failure query $Q_{\tau_f}=(o_b,\ldots,o_{e-1})$, we retrieve temporally aligned segments from the experience buffer $\mathcal{D}^{+}$ using slope-constrained subsequence dynamic time warping (S-DTW)~\citep{memmel2025strap}. Let $z_i^{\tau_f}$ be the embeddings of the failure query, where $i=0,\ldots,M-1$ and $M=e-b$. For each available successful trajectory $\tau_s\in\mathcal{D}^{+}$, let $(z_0^{\tau_s},\ldots,z_{N_{\tau_s}-1}^{\tau_s})$ denote its cached embeddings. We compute the framewise distance matrix
\begin{equation}
    d_{i,j}^{\tau_s}=\left\|z_i^{\tau_f}-z_j^{\tau_s}\right\|_2,\qquad 0\leq i<M,\quad 0\leq j<N_{\tau_s}.
    \label{eq:fert_frame_distance}
\end{equation}
The distance is evaluated on the concatenated representation in \Cref{eq:fert_representation}, jointly incorporating visual and proprioceptive similarity.

An alignment is represented by a sequence of candidate indices $p=(p_0,\ldots,p_{M-1})$, where query frame $i$ is matched to candidate frame $p_i$. We constrain the alignment to satisfy
\begin{equation}
    0\leq p_0,\qquad p_{M-1}<N_{\tau_s},\qquad p_i-p_{i-1}\in\{1,2\},\quad i=1,\ldots,M-1.
    \label{eq:fert_alignment_constraint}
\end{equation}
Thus, each query frame matches exactly one candidate frame, and successive matches advance by one or two frames in the candidate. The starting and ending positions are unrestricted. This permits bounded temporal stretching of the query within a longer successful sequence while preserving temporal order.

For each candidate, S-DTW finds the minimum-cost alignment through dynamic programming. Let $D_{i,j}^{\tau_s}$ be the minimum cumulative distance for matching query frames $0,\ldots,i$, with frame $i$ aligned to candidate frame $j$. We recursively define
\begin{equation}
    D_{i,j}^{\tau_s}=
    \begin{cases}
        d_{0,j}^{\tau_s}, & i=0,\\[2pt]
        d_{i,j}^{\tau_s}+\min\left\{D_{i-1,j-1}^{\tau_s},D_{i-1,j-2}^{\tau_s}\right\}, & i>0,
    \end{cases}
    \label{eq:fert_sdtw}
\end{equation}
where out-of-range or unreachable entries have infinite cost. Initializing the first row with the framewise distances allows retrieval to start at any candidate frame without penalizing the preceding portion of the trajectory.

The best endpoint and normalized retrieval cost are
\begin{equation}
    j_{\tau_s}^{*}=\arg\min_j D_{M-1,j}^{\tau_s},\qquad c_{\tau_s}=\frac{D_{M-1,j_{\tau_s}^{*}}^{\tau_s}}{M}.
    \label{eq:fert_matching_cost}
\end{equation}
The cost represents the average distance over the $M$ aligned frame pairs. Backtracking through the selected predecessors yields the optimal path $p^{*}$ and its enclosing half-open interval $[l_{\tau_s},r_{\tau_s})$, where $l_{\tau_s}=p_0^{*}$ and $r_{\tau_s}=p_{M-1}^{*}+1$ in $\tau_s$. The interval includes intermediate candidate frames skipped by two-frame transitions, although only aligned frame pairs contribute to the retrieval cost.

For retrieval event $n$, we retain one best interval for each $\tau_s\in\mathcal{D}^{+}$ and select the $K$ candidates with the lowest costs. Let $\mathcal{K}_n$ denote this selected set. If fewer than $K$ eligible candidates are available, all are retained. Each selected interval is then expanded to include adjacent actions and form the final retrieved set:
\begin{equation}
    R_n=\left\{\left[\max(0,l_{\tau_s}-h_{\mathrm{pre}}),\min(N_{\tau_s},r_{\tau_s}+h_{\mathrm{post}})\right),\ \tau_s\in\mathcal{K}_n\right\}.
    \label{eq:fert_interval_extension}
\end{equation}
Expansion does not change the retrieval costs or ranking. We typically use $K=10$ and $h_{\mathrm{pre}}=h_{\mathrm{post}}=30$ frames. The retrieved set is passed to the sampling-priority update in Step 4. If no valid retrieval is available, existing priorities remain unchanged.

To prevent retrieval from slowing policy updates, we downsample both the query sequence $z_i^{\tau_f}$ and the successful sequence $z_j^{\tau_s}$ before retrieval. We typically use a downsampling factor of three, yielding a theoretical $9\times$ speedup. In practice, this achieves approximately a $6\times$ speedup.

\subsubsection{Step 4: Prioritized fine-tuning}
\label{sec:weighting}

Successful trajectories supply candidate action targets, but their relevance to the policy's current difficulties varies. FERT converts retrieved intervals into sampling priorities, allowing recurring failure contexts to direct more updates towards related successful behaviour.

Let $A_i$ denote the successful action chunk beginning at dataset index $i$, with nominal horizon $H$, and let $\textrm{idx}\left(A_i\right)$ denote the frame interval it covers. We update the chunk’s priority as
\begin{equation}
    w_i^{(n)}=\rho w_i^{(n-1)}+\sum_{R\in R_n}\frac{\left|\textrm{idx}\left(A_i\right)\cap R\right|}{H},
    \label{eq:fert_weight_update}
\end{equation}
where $\rho\in[0,1]$ controls the retention of previous priorities. Overlap is computed within the same trajectory, so that $\left|\textrm{idx}\left(A_i\right)\cap R\right|$ is zero when $A_i$ and $R$ belong to different $\tau_s$.

The overlap fraction assigns full increments to chunks contained in a retrieved interval and partial increments to chunks crossing its boundaries. Repeated retrieval reinforces successful actions associated with recurring failure contexts, while exponential decay reduces the influence of older retrieval events. Retrieval costs determine which intervals are selected but do not scale the priority increments. The update is applied only when retrieval produces at least one valid match. Newly arriving successful samples receive a constant initial priority $w_{\mathrm{init}}$.

We sample action chunks from $\mathcal{D}^{+}$ in proportion to their priorities:
\begin{equation}
    P(i)=\frac{w_i}{\sum_{j\in\mathcal{D}^{+}}w_j},\qquad i\in\mathcal{D}^{+}.
    \label{eq:fert_sampling_distribution}
\end{equation}
We express prioritized fitting of successful action chunks through a negative log-likelihood objective. For a sampled batch $\mathcal{B}$, this objective is:
\begin{equation}
    \widehat{\mathcal{L}}_{\mathrm{FERT}}(\theta;\mathcal{B})=-\frac{1}{|\mathcal{B}|}\sum_{i\in\mathcal{B}}\log\pi_\theta(A_i\mid o_i,g_i),
    \label{eq:fert_policy_update}
\end{equation}
where $A_i$ is the recorded successful action chunk conditioned on observation $o_i$ and task $g_i$. Binary outcomes therefore guide learning through successful data selection and failure-conditioned prioritization, without requiring a learned critic, value targets or advantage estimates.

Since chunks are sampled in proportion to their priorities, the expected training objective is
\begin{equation}
    \mathcal{L}_{\mathrm{FERT}}(\theta)=-\sum_{i\in\mathcal{D}^{+}}\frac{w_i}{\sum_{j\in\mathcal{D}^{+}}w_j}\log\pi_\theta(A_i\mid o_i,g_i).
    \label{eq:fert_population_loss}
\end{equation}
The priorities are fixed during each update and enter through sampling. Our flow-based implementation uses the base policy's conditional flow-matching loss as a tractable surrogate for likelihood-based fitting, as detailed in \Cref{app:fert_details}. Training updates the policy used for inference, and subsequent robot execution generates new successful trajectories and failure contexts.

\subsubsection{Understanding FERT}
\label{sec:understanding_fert}

We study the local effect of successful replay and retrieval weighting using the likelihood-based update above. Under a current-policy replay approximation and the retrieval-alignment conditions below, successful replay provides an improving update direction and retrieval adds a positive first-order contribution. The analysis compares a single sufficiently small mean-gradient update with uniform successful replay under these conditions.

\paragraph{Notation and assumptions.}

Let $J(\theta)$ be the task-success objective defined in \Cref{eq:task_objective}, and let $\theta_0$ denote the policy parameters immediately before an update. Let $\mathbb{E}_{\mathcal{D}^{+}}$ denote expectation under uniform successful replay, and normalize the retrieval priorities as $\omega=w/\mathbb{E}_{\mathcal{D}^{+}}[w]$. The mean log-likelihood gradients are
\begin{equation}
    \begin{aligned}
        \mathbf{G}_{\mathrm{unif}}&=\mathbb{E}_{\mathcal{D}^{+}}\left[\nabla_\theta\log\pi_{\theta_0}(A\mid o)\right],\\
        \mathbf{G}_{\mathrm{FERT}}&=\mathbb{E}_{\mathcal{D}^{+}}\left[\omega\nabla_\theta\log\pi_{\theta_0}(A\mid o)\right].
    \end{aligned}
    \label{eq:fert_mean_gradients}
\end{equation}
Here, gradients are evaluated at $\theta_0$. Minimizing the negative log-likelihood in \Cref{eq:fert_population_loss} gives the mean update $\theta_{\mathrm{FERT}}=\theta_0+\eta\mathbf{G}_{\mathrm{FERT}}$; uniform replay gives $\theta_{\mathrm{unif}}=\theta_0+\eta\mathbf{G}_{\mathrm{unif}}$, using the same step size $\eta$.

We assume that the policy and expected return are differentiable, the relevant gradient expectations are finite, and differentiation can be interchanged with expectation. The environment dynamics are independent of the policy parameters. During an update, the replay distribution and priorities remain fixed, and $\nabla J$ is $\lambda_J$-Lipschitz in a neighbourhood containing both updates.

Write $\mathbf{v}=\nabla J(\theta_0)$. The first-order contribution of a replay sample is $u=\mathbf{v}^{\top}\nabla_\theta\log\pi_{\theta_0}(A\mid o)$. Here, $u$ denotes sample contribution within the theoretical analysis. We use two quantities to describe the update:
\begin{equation}
    \mu=\mathbb{E}_{\mathcal{D}^{+}}[u],\qquad \kappa=\operatorname{Cov}_{\mathcal{D}^{+}}(\omega,u).
    \label{eq:fert_alignment}
\end{equation}
Here, $\mu$ measures the alignment of successful replay with the expected-return objective, and $\kappa$ measures the additional alignment introduced by retrieval. Neither quantity needs to be estimated by FERT.

\begin{lemma}[Successful replay alignment]
\label{lemma:align}
For current-policy successful replay, with $J(\theta_0)>0$ and $\mathbf{v}\neq0$,
\begin{equation}
    \mathbf{G}_{\mathrm{unif}}=\frac{\mathbf{v}}{T_cJ(\theta_0)},\qquad \mu=\frac{\|\mathbf{v}\|_2^2}{T_cJ(\theta_0)}>0,
    \label{eq:fert_success_alignment}
\end{equation}
where $T_c$ is the episode length in action chunks. Thus, fitting successful experience produces an update aligned with improving expected return. The replay model and proof are given in \Cref{app:lemma1}.
\end{lemma}

\begin{lemma}[Retrieval alignment]
\label{lemma:retrieval}
Consider the stage-wise robot task model specified in \Cref{app:lemma2}, with independent stages and fixed relative likelihoods among successful actions within each stage. If retrieval assigns at least as much weight to stages with higher local failure rates, with a strict increase for at least one pair of stages with different failure rates, then
\begin{equation}
    \kappa>0.
\end{equation}
Under this setting, successful experience from more failure-prone stages has greater first-order benefit to the expected-return objective. The model and proof are given in \Cref{app:lemma2}.
\end{lemma}

\begin{theorem}[Local performance guarantee]
\label{thm:imp}
Under the conditions of \Cref{lemma:align,lemma:retrieval}, a sufficiently small positive step size gives
\begin{equation}
    J(\theta_{\mathrm{FERT}})>J(\theta_{\mathrm{unif}})>J(\theta_0).
    \label{eq:fert_combined_guarantee}
\end{equation}
The improvement from FERT admits the following lower bound:
\begin{equation}
    J(\theta_{\mathrm{FERT}})-J(\theta_0)\geq\eta\underbrace{\mu}_{\textrm{successful replay}}+\eta\underbrace{\kappa}_{\textrm{retrieval contribution}}-\underbrace{\frac{\lambda_J\eta^2}{2}\|\mathbf{G}_{\mathrm{FERT}}\|_2^2}_{\textrm{finite-step penalty}}.
    \label{eq:fert_improvement_decomposition}
\end{equation}
The first term captures the benefit of learning from successful experience, the second captures the additional benefit of experience retrieval, and the final term bounds the higher-order effects of a finite update. The proof and explicit step-size conditions are provided in \Cref{app:thm1}.
\end{theorem}

Overall, \Cref{thm:imp} characterizes a local benefit of retrieval-weighted likelihood updates under the stated replay and alignment conditions. The retrieval-alignment assumption specifies when prioritization helps. The practical benefit of the retrieval mechanism is evaluated through the ablation in \Cref{sec:exp-abl}, using repeated updates with historical experience and the flow-matching implementation.

\subsection{Fault-tolerant asynchronous training system}
\label{sec:infra}

\begin{table}[!b]
    \centering
    \small
    \caption{Interaction and experience paths in FATS. Sizes describe model-input images and a complete episode record, respectively.}
    \label{tab:fats-flows}
    \begin{tabularx}{\linewidth}{@{}lXX@{}}
        \toprule
        & Interaction & Experience \\
        \midrule
        Unit & Observation and action chunk & Complete labeled episode \\
        Size & 0.45\,MB of images per model input & 1.58\,GB per Phone Packaging episode \\
        Timing & Each inference request & After labeling, about every 4\,min per robot \\
        Requirement & Reply before the robot's next interaction & Complete and labeled before use in training \\
        Path & Robot observations to cloud inference, actions back to the robot & Robot site to the experience store of a training worker \\
        \bottomrule
    \end{tabularx}
\end{table}

Learning from experience generated during deployment requires a distributed loop that connects robot execution, experience collection and policy improvement while accommodating interruptions to individual robots.
FATS provides this system layer: each robot executes action chunks computed by cloud inference,
    and its experience contributes to updates of the policy that inference serves.
In our deployment this loop spans 96 robots and six cloud servers,
    connected through a single 30\,Gbps edge-cloud link (\Cref{app:infra_deployment}).
Two properties of this setting determine how the system must be built.

The first property is that interaction and experience have different timing requirements (\Cref{tab:fats-flows}).
A robot needs the action chunk for its current observation before its next interaction,
    so delays in cloud inference or communication slow execution.
Training instead consumes complete episodes, whose records are much larger but need not arrive immediately.
Both share the edge-cloud link,
    so the system must move experience without making robot execution wait for it.

The second property is that physical robots have availability constraints beyond those of computing resources: resets, battery depletion and hardware faults can interrupt their contribution to learning.
In the 96-robot run, fault incidents affected 22 robots,
    and some robot was interrupted every 22\,min on average (\Cref{sec:exp-system}).
Resets pause robots far more often,
    so the set of robots producing experience changes continuously.
Most robot faults are transient, e.g., a network interruption, a low battery or a triggered joint limit.
Although robots and accelerators are both schedulable resources, robot recovery must preserve the association with a physical workstation and task setup.
Meanwhile, the policy and optimizer state, which are expensive to rebuild, reside in training workers.
The system therefore separates recovery of experience-producing workers from recovery of the learning state.

\begin{figure}[!tb]
    \centering
    \includegraphics[width=\linewidth]{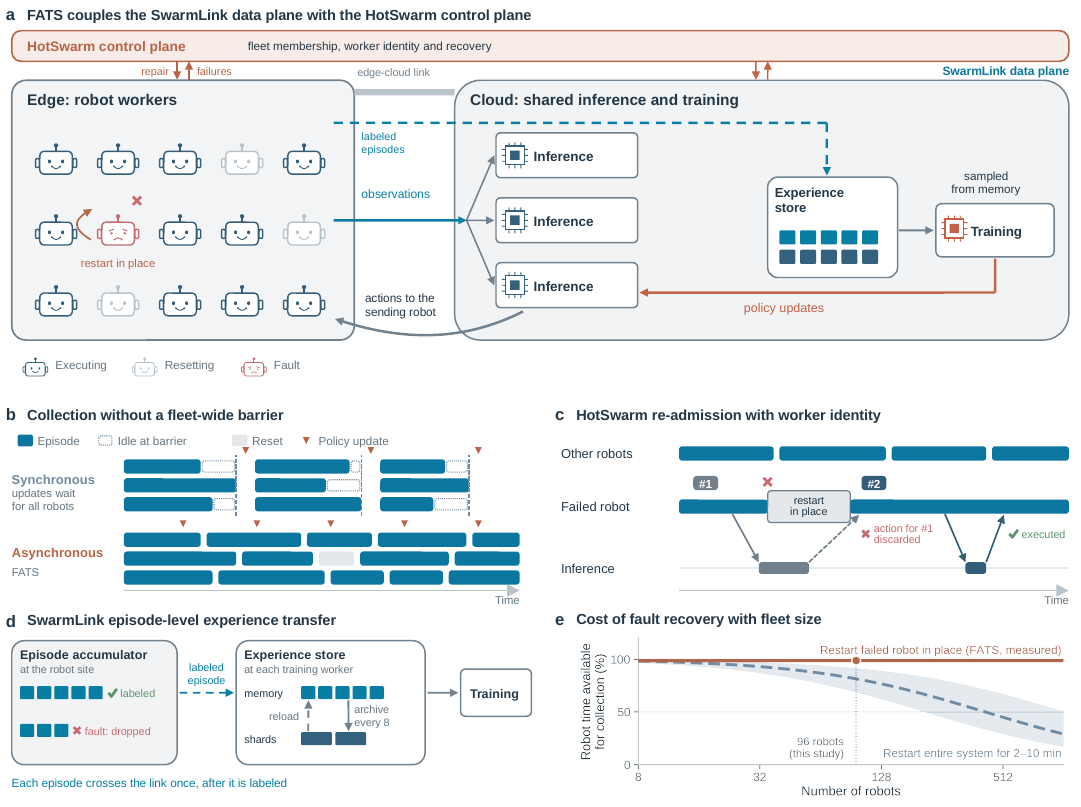}
    \caption{\textbf{Fault-tolerant asynchronous training system (FATS).}
        \textbf{a}, The SwarmLink data plane connects robot workers at the edge to shared cloud inference
        and moves each labeled episode into an experience store for training.
        The HotSwarm control plane manages membership and worker identity, and recovers failed workers.
        \textbf{b}, Synchronous collection makes each policy update wait for every robot, which leaves faster robots idle.
        In asynchronous collection, robots continue independently.
        \textbf{c}, HotSwarm restarts the worker of a failed robot in place while the other robots continue,
        and the action computed for the failed worker (\#1) is discarded.
        \textbf{d}, Each robot accumulates its episode locally and sends it as one record once it is labeled,
        so a fault discards only the partial episode at that robot.
        The experience store serves training from memory and archives shards that a restarted training worker reloads.
        \textbf{e}, Robot time available for collection after fault handling,
        measured for FATS in the 96-robot run and modeled for restarting the entire system after each fault event
        with restart times between 2 and 10\,min (\Cref{app:infra_model}).
        Durations in \textbf{b} and \textbf{c} are illustrative.}
    \label{fig:fats-concept}
\end{figure}

FATS addresses the two properties with a data plane and a control plane.
The data plane, \textit{SwarmLink}, connects robot execution to cloud inference and training.
It pools inference requests through distributed channels and moves each episode to training as one record once it is labeled.
The control plane, \textit{HotSwarm}, decides which stations make up the fleet,
    i.e., the fixed places for one robot or one accelerator,
    and which worker currently occupies each of them.
It detects failures, returns failed robots to their places and repairs the data plane around them.
\Cref{sec:infra-arch} describes how the two planes interact,
    \Cref{sec:infra-swarmlink,sec:infra-hotswarm} describe each plane,
    and \Cref{sec:infra-walkthrough} follows one robot fault through both.

\subsubsection{Architecture}
\label{sec:infra-arch}

FATS separates the learning loop into three kinds of workers.
Environment workers operate robots, i.e., they acquire observations and execute actions.
Inference workers compute actions from observations using the current policy.
Training workers update the policy from collected experience.
Every node describes its accelerators and attached robots in the same way,
    including each robot's communication endpoint and cameras,
    and a placement specification assigns workers to these resources.
In our deployment, environment workers run at the robot site,
    while inference and training workers run on cloud servers (\Cref{app:infra_placement}).

The two planes cannot be designed separately, because every failure involves both.
A crashed process leaves data in channels and connections that the data plane must clean up,
    and a broken connection is often the first sign of a failure that the control plane must handle.
The planes therefore meet at three interfaces (\Cref{fig:fats-concept}a).
\begin{itemize}
    \item \textit{Identity flows down.}
        HotSwarm assigns every worker an identity,
        and SwarmLink attaches it to the data that the worker produces and checks it when replies are delivered.
    \item \textit{Failures flow up.}
        A failed transfer or a lost channel replica in SwarmLink is reported to HotSwarm,
        in addition to HotSwarm's own liveness checks.
    \item \textit{Repair flows down.}
        HotSwarm rebuilds the connections and channel replicas that a failed worker leaves behind,
        including connections that cross the edge-cloud link.
\end{itemize}

\subsubsection{SwarmLink data plane}
\label{sec:infra-swarmlink}

SwarmLink preserves independent robot progress while sharing cloud GPUs and the edge-cloud link.
A robot waits for actions answering its own observation, while training consumes completed experience from the fleet.
Distributed channels connect these two paths without imposing a common collection schedule.

\paragraph{Distributed channels.}
A \textit{channel} is a named producer-consumer queue that lets workers exchange data without fixed peer connections.
Each channel has replicas on participating nodes,
    and each key selects a first-in-first-out queue hosted on its first producer's node.
Producers submit locally when they use keys homed on their own nodes,
    while remote consumers retrieve data directly from the home replica.
Asynchronous operations allow communication to overlap with other work.

Requests carry the robot's station, worker identity, episode and step identifiers.
Replies preserve these identifiers and add the policy version used for inference.
This associates actions with their observations and lets HotSwarm reject stale replies after a worker restart (\Cref{sec:infra-hotswarm}).
HotSwarm also repairs failed replicas.
Experience held only by a failed replica may be lost,
    reducing the available training data without invalidating a completed training step.

\paragraph{Shared inference.}
Cloud inference workers serve a common pool of requests rather than fixed robot groups.
A paused robot therefore leaves inference capacity available to others.
Workers batch ready observations and return each action chunk through its robot's reply channel.
Batch membership is temporary, so robots that share an inference call can subsequently progress at different rates.

Waiting for a full batch would recreate a dependence on slower or unavailable robots.
Each worker therefore dispatches when either its batch target is reached or the oldest request reaches the batching timeout.
The target provides GPU efficiency under load, while the timeout permits smaller batches when arrivals are sparse.
This limits waiting for additional inputs, rather than time spent waiting for a busy GPU.

\paragraph{Episode-level experience transfer.}
Distributed reinforcement learning systems usually stream experience to learners in fragments of steps
    and keep it in replay buffers that overwrite old data~\citep{horgan2018distributed,espeholt2018impala,espeholt2019seed}.
This suits simulation, where experience is cheap to regenerate.
In CEL, an episode becomes usable only after an operator labels it as a success or failure,
    because FERT uses the label to decide whether the episode contributes successful segments or a failure context (\Cref{sec:ret}).
SwarmLink therefore transfers experience as complete, labeled episodes (\Cref{fig:fats-concept}d).
Each environment worker accumulates the frames of its current episode at the robot site
    and, once the episode is labeled, sends it as one record tagged with the worker identity and episode identifier.
The send is asynchronous and uses a channel separate from observations, so it does not delay interaction.
Because a partial episode never leaves the robot site, a robot failure requires no cleanup in the cloud,
    and environment workers hold no state that is costly to lose.
Robots progress independently, so their episodes finish at different times and transfers spread out over the run.

\paragraph{Experience store.}
Each training worker appends arriving episodes to an in-memory store, from which training samples without reading the disk.
Every eight episodes, the worker also writes the new episodes as a shard in the LeRobot format~\citep{cadene2024lerobot} to persistent storage.
A training worker relaunched by HotSwarm reloads its shards in the background and skips any shard left incomplete,
    so a failure loses at most the episodes received since the last shard.

\subsubsection{HotSwarm control plane}
\label{sec:infra-hotswarm}

HotSwarm treats the fleet as a membership that changes during training rather than a fixed set of workers.
It returns a failed robot to its own place in the fleet and keeps every other worker running,
    and it guarantees that nothing left behind by the failed process is mistaken for the work of its replacement.

\paragraph{Stations and workers.}
HotSwarm keeps a registry of \textit{stations}.
A station is a fixed place in a worker group, bound to a placement, i.e., a node, its devices and, for an environment worker, one robot.
Each station is occupied by one \textit{worker}, i.e., one process, at a time,
    and every worker receives a unique identity when it is launched.
The distinction matters because the two have different lifetimes.
A station lasts for the whole training run, because the robot and node it is bound to stay in the fleet,
    whereas a worker lasts only until its process fails.
Whatever should survive a failure, such as the robot's place in the fleet and its reply key, belongs to the station,
    and whatever must not survive it, such as messages in flight and open connections, belongs to the worker.
Membership changes through three operations on stations,
    namely admitting a worker to a station, retiring it, and re-admitting a new worker to the same station.
Recovering a failed robot, taking a robot offline for maintenance and adding a robot whose node has just joined the cluster are therefore one mechanism with three triggers.
HotSwarm periodically refreshes its inventory of nodes, so a newly connected node becomes available for placement while training runs.
An environment worker spends most of its time in a loop that exchanges observations and actions,
    and a message sent through its regular task queue would wait until the loop ends.
Membership changes therefore reach running workers through a separate command channel,
    which a worker receives even while busy in its loop.

\paragraph{Tiered recovery.}
The three kinds of workers differ in how often they fail and in how costly their state is to rebuild.
Environment workers are attached to physical robots and fail most often,
    e.g., when a robot loses power or its network connection,
    but they hold almost nothing that is hard to rebuild.
Their only state is the episode in progress, which is discarded anyway when the robot is reset.
Inference and training workers both run on data-center servers, but their recovery requirements differ.
Inference workers hold a replaceable copy of the policy,
    whereas training workers hold the policy and optimizer state accumulated over hours of training.
A single recovery policy cannot serve both ends.
Restoring the whole job from a checkpoint after every failure, as data-center training systems commonly do,
    would stop the fleet every time a robot fails,
    whereas simply restarting every failed process would lose the optimizer state whenever a training worker fails.
HotSwarm instead matches recovery to each kind of worker.
Environment workers are restarted in place as often as needed, because a restart costs at most one episode.
Inference workers are restarted and load the current policy before resuming service.
Training workers save checkpoints periodically, resume from the latest one and reload their experience stores from shards (\Cref{sec:infra-swarmlink}),
    while the surviving inference workers continue serving their current policy
    and new episodes wait in the channels until training resumes.

HotSwarm detects failures through periodic liveness checks, failing calls and the failure reports of SwarmLink.
A process that stops responding may have crashed or may only be slow, and the two cannot be told apart from outside.
If HotSwarm waited, the station would remain in limbo,
    and a slow process that later resumed could send data alongside its replacement.
HotSwarm therefore terminates a suspected process instead of waiting to learn whether it will recover.
This makes every failure definite, at the cost of at most one episode when the process was only slow.

\paragraph{Re-admission.}
When the environment worker at a station fails, HotSwarm restores the station in five steps (\Cref{fig:fats-concept}c).
\begin{enumerate}
    \item It retires the failed worker.
        The registry keeps the retired record, marked as dead, instead of deleting it,
        so that late references to it can still be recognized as stale.
    \item It launches a new worker on the recorded placement, i.e., the same node and robot.
    \item It rebuilds the connections that the failed worker held with its peers.
        Peers communicate through connection groups that are set up once and then reused,
        and a failed member leaves its group unusable for the others.
        Each connection group carries an \textit{epoch}, i.e., a version number that advances when the group is torn down,
        and peers rejoin only within the new epoch,
        so a new member never binds to an endpoint, i.e., a network address, published by its predecessor.
        Every wait during this rejoin also checks for peer failure and epoch advance, so it cannot block indefinitely.
        Transfers queued against the old group are replayed on the new one in their original order.
    \item It replays the station's part of every call that was in progress when the failure occurred.
        Work is dispatched to a worker group as calls that every station executes, e.g., one rollout step on all robots.
        When one station fails during a call, only its part is executed again, by the new worker,
        and the results already returned by the other stations are kept.
        A single relaunch serves all such calls.
    \item The new worker resumes collection once its robot has been reset.
\end{enumerate}
Concurrent reports of the same failure trigger a single relaunch,
    and restarts and retries follow the limits of each tier, which are unlimited for environment workers.
Recovery restores software only.
Repairing hardware and resetting the scene remain operator tasks.

\begin{table}[!t]
    \centering
    \small
    \caption{Artifacts that can outlive a failed worker, the hazard each poses to its replacement, and the response of HotSwarm.}
    \label{tab:hotswarm-identity}
    \begin{tabularx}{\linewidth}{@{}lXl@{}}
        \toprule
        Artifact & Hazard if the identity is not checked & Response \\
        \midrule
        Membership record & A peer builds a connection to a dead process and waits indefinitely & Refuse and wait for a fresh record \\
        Control message in flight & A late failure report retires the replacement & Ignore \\
        Connection held by a peer & The replacement joins an outdated connection group & Rebuild under a new epoch \\
        Data in the pipeline & Training pairs actions with another worker's observations & Drop and receive again \\
        Queued transfer & The transfer runs on a destroyed connection & Replay on the new connection \\
        \bottomrule
    \end{tabularx}
\end{table}

\paragraph{Worker identity across both planes.}
Restoring a station means that the new worker inherits everything the failed worker left in the system (\Cref{tab:hotswarm-identity}).
HotSwarm therefore attaches the worker identity to every artifact that can outlive its process
    and checks it wherever the artifact is used,
    with a response suited to each kind of artifact.
The registry records which worker currently occupies each station,
    so each check is a local comparison between the identity on the artifact and the registry entry.
No coordination with other workers is needed,
    which matters because the workers that would have to be consulted may themselves be failing.

The data artifacts deserve particular attention.
An action computed for an observation of the failed worker can arrive after the new worker has started.
Executed by the new worker, it would move the robot according to a state that no longer holds,
    and recorded in its episode, it would pair an action with an observation it was not computed from.
Neither case raises an error, and training would silently learn from inconsistent pairs.
SwarmLink prevents both by delivering an action only to the worker that sent the observation,
    and the new worker discards any reply that carries its predecessor's identity.

\paragraph{Decoupled progress.}
Training retrieves experience from channels rather than from particular robots,
    and inference serves whichever robots are waiting.
A failed or paused robot therefore lowers the rate at which experience arrives but stops nothing else,
    and a slow robot delays only its own episodes.
In a synchronous design, each policy update would instead wait for the slowest robot (\Cref{fig:fats-concept}b).
Under the fixed per-robot fault-rate and downtime assumptions in \Cref{app:infra_model}, local recovery makes both expected lost robot time and total robot time grow in proportion to fleet size, so the expected fraction lost to faults stays constant.
A system that restarts entirely after each fault instead stops every robot at every fault.
In a fleet twice as large, faults occur twice as often and each stops twice as many robots,
    so the share of robot time lost roughly doubles (\Cref{fig:fats-concept}e).

\subsubsection{A robot fault across both planes}
\label{sec:infra-walkthrough}

The planes cooperate whenever a robot fails.
Suppose the control process of a robot crashes while one of its observations is waiting in the interaction channel and its episode is half complete.
HotSwarm detects the failure through its liveness check or through the failed delivery reported by SwarmLink,
    terminates any remains of the process and retires the worker.
SwarmLink continues to serve the other robots.
The action computed for the waiting observation is held under the robot's reply key.
The half-complete episode existed only in the accumulator of the failed worker,
    so it disappears with the worker and nothing in the cloud needs cleaning up.
Training continues on the episodes already in its experience store.
HotSwarm launches a new worker at the same station,
    rebuilds its connections under a new epoch and replays its interrupted calls.
The new worker discards the stale action, which carries its predecessor's identity,
    and its first observation starts a fresh episode once the robot has been reset.
In the 96-robot run, for example, robot 20 lost the link to the inertial sensor of its mobile base at 11:24
    and returned to collection at 11:37 after an 11-min restart in place,
    while the rest of the fleet continued (\Cref{fig:error,fig:episode-record}).

FATS is implemented on RLinf~\citep{yu2026rlinf}. Its worker and communication abstractions support both the expert-supervised initialization runs and the subsequent FERT runs; the learning algorithm determines how collected experience is used.
\Cref{app:infra_deployment} describes the hardware, the network and the placement used in our deployment.

\section{Experimental Results}
\label{sec:experiments}

We evaluate FERT and FATS as a concrete realization of CEL for learning from experience generated during deployment. We measure whether fleet-generated experience improves task success without new expert demonstrations or corrective actions, and how learning and its supporting system behave as more robots contribute experience. Task success on held-out robots provides our empirical measure of reliable execution. The controlled experiments address four questions:
\begin{itemize}
    \item Can fleet-generated experience improve task success beyond the initialized policy? (\Cref{sec:exp-main})
    \item How does policy performance change with fleet size when collection time per robot is held fixed? (\Cref{sec:exp-scale})
    \item How do failure-guided retrieval and the other design choices affect FERT? (\Cref{sec:exp-abl})
    \item How does FATS support fleet operation in the presence of local interruptions and communication delays? (\Cref{sec:exp-system})
\end{itemize}
\Cref{sec:exp-setup} describes the shared experimental setup. We then present policy improvement, scaling, ablation and system results.

\subsection{Experimental setup}
\label{sec:exp-setup}

\subsubsection{Tasks and robot setup}
We evaluate five real-world tabletop manipulation tasks spanning flexible and deformable objects, precise alignment, tool use and multi-stage execution. Task durations range from tens of seconds to several minutes. These tasks expose different sources of execution failure, from contact and alignment errors to errors that accumulate across successive stages. \Cref{fig:task-setup} provides an overview of the task setups. Step-by-Step explanation of each task is available in \Cref{app:task_success_criteria}.

\paragraph{Flower Arrangement.}
The robot sequentially inserts two flowers and one leaf into a vase. Flexible stems can oscillate during motion, requiring smooth transport and accurate insertion.

\paragraph{Clothes Folding.}
The robot folds a T-shirt laid flat on a table according to a predefined procedure. The task takes approximately 40 seconds and requires manipulation of deformable fabric and controlled contact with the table.

\paragraph{Phone Charging.}
At the start of each trial, the USB-C connector is already inserted into the phone. The robot inserts a charging plug into a power strip, disconnects the Type-C cable from the phone, and places the phone at a designated location. The task combines sequential manipulation with precise spatial alignment and contact control during plug insertion.

\paragraph{Stool Assembly.}
The robot places two screws at designated locations on a stool, picks up a screwdriver, and tightens the screws. This task combines small-object manipulation and tool use; unsuccessful alignment may require readjustment before execution can continue.

\paragraph{Phone Packaging.}
The robot places a manual, charging cable, charger, phone case and phone into their designated locations in a phone box, then closes the box. Execution takes \textbf{approximately 2-3 minutes}. The task requires a sequence of object placements, with errors in earlier stages potentially affecting later operations.

We use AgiBot G2 robots with a dual-arm configuration. Policy observations include one head-camera view, two wrist-camera views and proprioceptive states. The action space consists of delta joint positions for the two arms. Hardware deployment and network configuration are described in~\Cref{app:infra_deployment}.

\begin{figure}[htbp]
    \centering
    \includegraphics[width=1.0\linewidth]{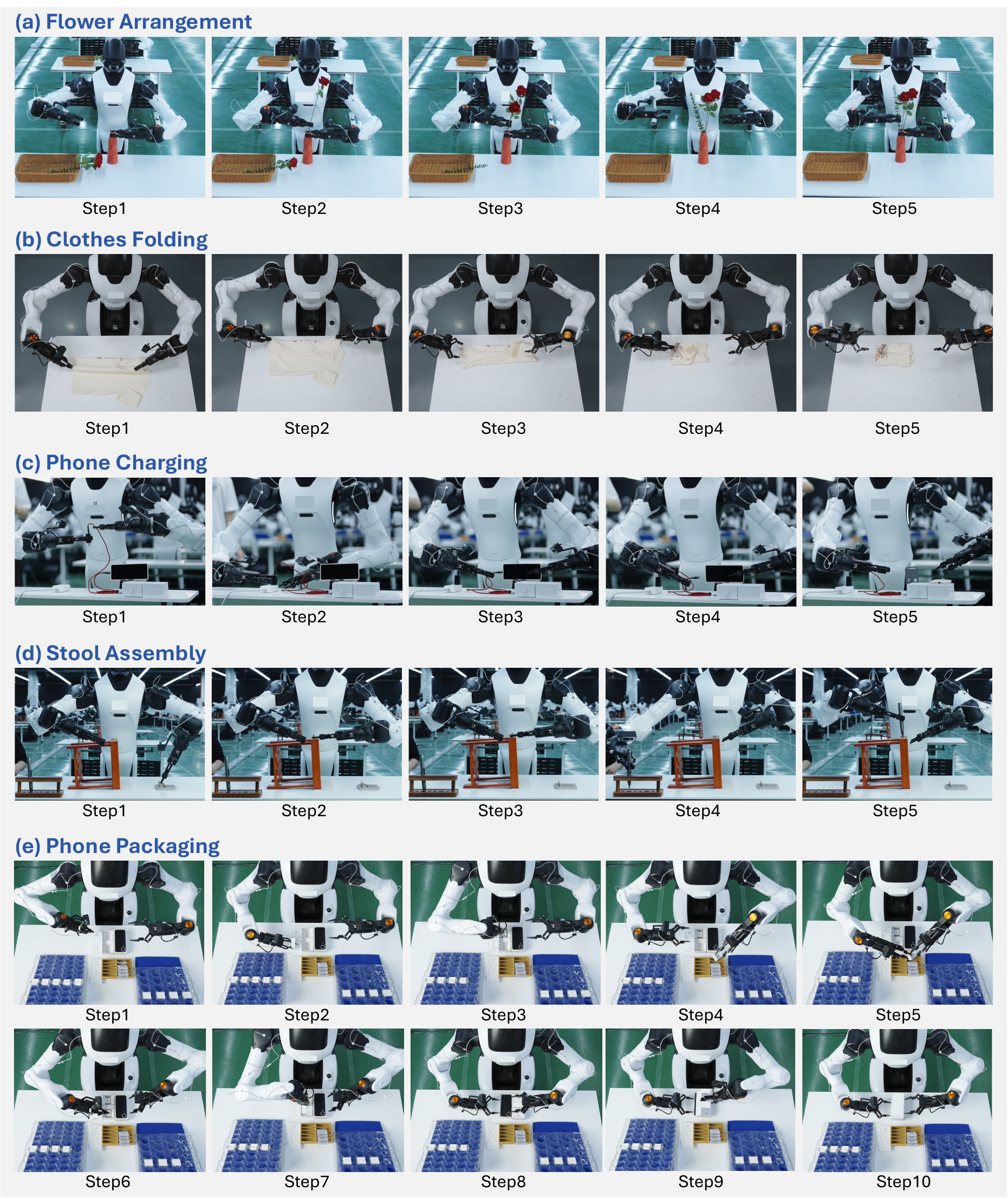}
    \caption{Overview of the five manipulation tasks performed with AgiBot G2 robots: Flower Arrangement, Clothes Folding, Phone Charging, Stool Assembly and Phone Packaging. A detailed step-by-step description is provided in \Cref{app:task_success_criteria}.}
    \label{fig:task-setup}
\end{figure}

\subsubsection{Training setup}

Our experiments use a shared expert-supervised policy initialization for FERT, RECAP and RLT. For each task, the VLA model $\pi_0$~\citep{black2024pi0} is first adapted through supervised fine-tuning (SFT), followed by human-gated dataset aggregation (HG-DAgger)~\citep{kelly2019hg} across robot fleets. We refer to this implementation as multi-robot HG-DAgger, abbreviated as (M)HG-DAgger in \Cref{fig:main-results}: robots collect human intervention data in parallel to train a shared task-specific policy, as detailed in~\Cref{app:policy_preparation}.

For each task, FERT, RECAP and RLT are initialized from the same selected (M)HG-DAgger checkpoint and trained in separate runs. SFT and (M)HG-DAgger serve as performance references for the policy-initialization stages. For Phone Charging, Stool Assembly and Phone Packaging, this checkpoint is obtained using mixed multi-robot HG-DAgger, which combines data collected from the preceding HG-DAgger runs at different fleet sizes for the same task. For Flower Arrangement and Clothes Folding, we use the checkpoint obtained from multi-robot HG-DAgger with the largest fleet.

During subsequent FERT training, operators provide binary task-outcome labels and reset environments, while robots autonomously generate the actions used for learning. No additional expert demonstrations or corrective actions are collected, and the demonstration and intervention datasets used for policy initialization are not transferred to the FERT experience buffer. Collection, optimization and policy distribution proceed asynchronously through FATS, allowing robots to contribute experience and execute updated policy versions during deployment. We evaluate the resulting policies on held-out robots to assess whether the collected experience improves task success.

\subsubsection{Baseline methods}

We compare FERT with two approaches to learning from robot experience: single-robot online reinforcement learning with RLT~\citep{xu2026rltokenbootstrappingonline} and offline reinforcement learning on fleet-collected data with RECAP~\citep{intelligence2025pi06vlalearnsexperience}. In our experiments, both baselines use human corrective actions during subsequent policy learning.

RLT serves as a single-robot online reinforcement learning baseline. It learns a residual actor-critic on top of a frozen VLA policy using sparse task rewards and human corrections collected during interaction. RLT provides a reference for improving a policy through online reinforcement learning on an individual robot, against which we evaluate FERT's learning from fleet experience.

RECAP serves as an offline reinforcement learning baseline trained on experience collected across multiple robots. Our implementation combines autonomous fleet rollouts with demonstrations and human corrections, fits a value function to estimate action advantages, and fine-tunes an advantage-conditioned policy. We use a single round of offline policy improvement on the collected dataset. RECAP therefore provides a comparison with an alternative method that also uses fleet experience but retains expert action supervision during subsequent policy learning.

Detailed baseline implementations, adaptations, and hyperparameters are provided in~\Cref{app:baseline_details}. We additionally report human teleoperation as a separate performance reference.

\subsubsection{Evaluation protocol}
We conduct evaluations on two AgiBot G2 robots excluded from training data collection. All learned-policy evaluation trials are executed autonomously, without human intervention or corrective actions. Even robots of the same model can exhibit instance-level variations in hardware characteristics, calibration, and execution. We therefore use held-out robot instances to test whether policy improvements learned from fleet experience transfer to robots that did not contribute experience during training. This tests transfer across robot instances of the same platform, a requirement for applying fleet-learned policies beyond the robots that supplied their training experience.

For each evaluated policy on each task, we conduct 60 trials per robot, \textbf{giving 120 trials in total}. For the human teleoperation reference, six teleoperators each complete 20 trials per task, also giving 120 trials in total. We report success rates with 95\% Wilson score confidence intervals~\citep{wilson1927probable} in the figures. Task-specific success and failure criteria are detailed in~\Cref{app:task_success_criteria}.

\begin{figure}[t]
    \centering
    \includegraphics[width=\linewidth]{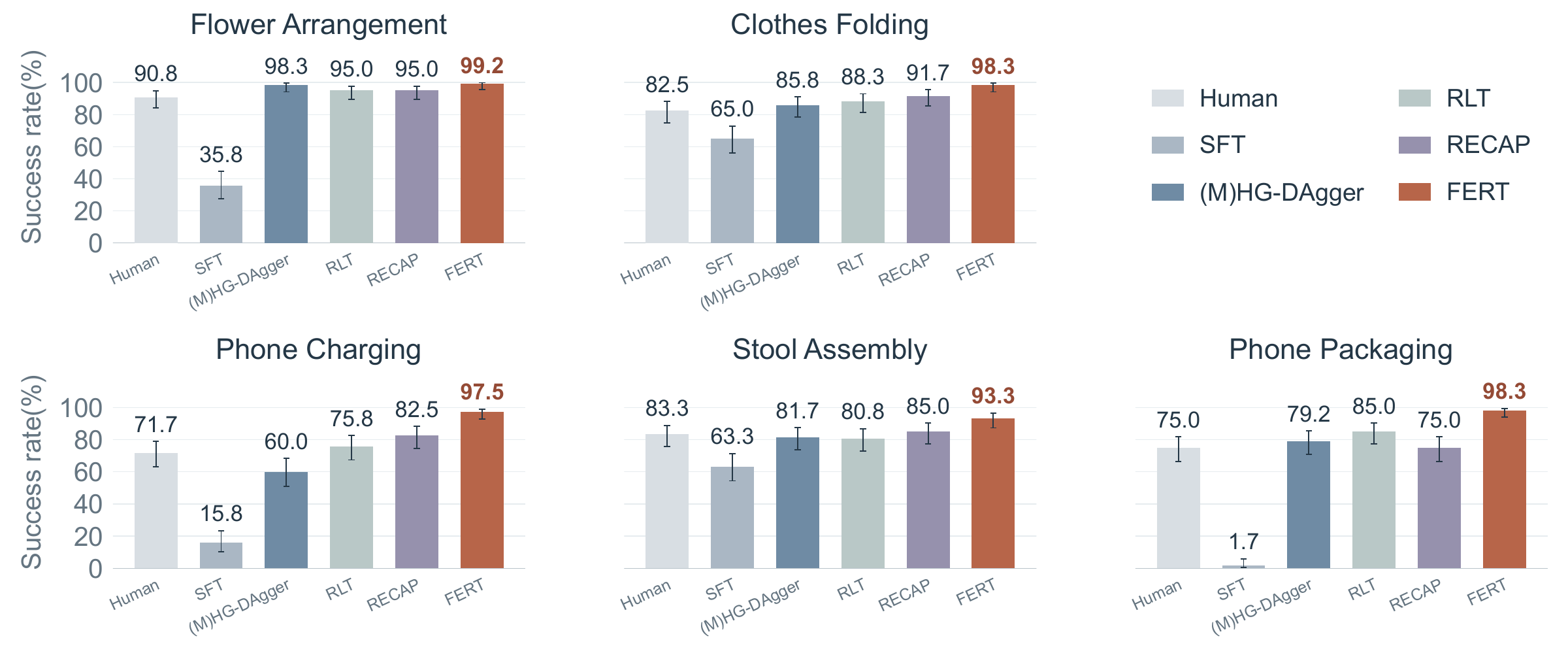}
    \caption{Main results across five tasks. Each subfigure compares the success rates of human teleoperation, SFT, the selected multi-robot HG-DAgger (labelled (M)HG-DAgger), RLT, RECAP, and FERT on one task. For FERT, we report the checkpoint trained with the largest robot fleet. Error bars indicate 95\% Wilson confidence intervals.}
    \label{fig:main-results}
\end{figure}

\subsection{Policy improvement from fleet experience}
\label{sec:exp-main}
\Cref{fig:main-results} summarizes performance across the five manipulation tasks. SFT and multi-robot HG-DAgger, labelled (M)HG-DAgger, provide the starting policies for the subsequent training stages. For FERT, we report results at the largest evaluated training fleet size for each task: 16 robots for Flower Arrangement, Clothes Folding and Phone Charging, 32 for Stool Assembly, and 96 for Phone Packaging. \Cref{tab:detail-succ} provides the success counts for each held-out test robot and the pooled success rates across the evaluated checkpoints and training fleet sizes.

Multi-robot HG-DAgger raises the mean success rate across the five tasks from 36.3\% to 81.0\%, providing stronger starting policies for subsequent learning. Without additional expert demonstrations or corrective actions, FERT further raises the mean success rate to 97.3\%, an improvement of 16.3 percentage points and a relative reduction of 86.0\% in the average failure rate (\Cref{tab:rates-failure-selected}). On Phone Charging, success increases from 60.0\% to 97.5\%, corresponding to a relative failure reduction of 93.8\%. On Phone Packaging, learning with 96 robots increases success from 79.2\% to 98.3\%, corresponding to a relative failure reduction of 92.0\%. These results show that \textbf{FERT improves the reliability of the initialized policies through fleet experience under the evaluated task and training configurations}.

\Cref{tab:detail-human} reports the task-specific success counts for all six teleoperators and their pooled success rates. FERT also exceeds the human teleoperation reference on all five tasks, with an average success-rate advantage of 16.7 percentage points. The largest gains occur on Phone Charging and Phone Packaging, reaching 25.8 percentage points and 23.3 percentage points, respectively. Under the same task-success criteria, the resulting autonomous policies therefore achieve higher observed task-success rates than the participating human teleoperators.

\begin{table}[htbp]
\centering
\begin{minipage}{0.9\linewidth}
\normalsize
\setlength{\tabcolsep}{5pt}
\renewcommand{\arraystretch}{1.25}
\setlength{\heavyrulewidth}{0.7pt}
\setlength{\lightrulewidth}{0.4pt}
\captionsetup{font=small,labelfont=bf,justification=raggedright,singlelinecheck=false,skip=8pt}
\caption{Failure rate table. (M)HG-DAgger reports the performance of the checkpoints used to initialize FERT. FERT reports results obtained with the maximum number of training robots for each task.}
\label{tab:rates-failure-selected}
\begin{tabular}{
  >{\raggedright\arraybackslash}m{\dimexpr0.30\linewidth-2\tabcolsep\relax}
  >{\centering\arraybackslash}m{\dimexpr0.24\linewidth-2\tabcolsep\relax}
  >{\centering\arraybackslash}m{\dimexpr0.20\linewidth-2\tabcolsep\relax}
  >{\centering\arraybackslash}m{\dimexpr0.26\linewidth-2\tabcolsep\relax}}
\toprule
\textbf{Task name} &
\begin{tabular}[c]{@{}c@{}}\textbf{(M)HG-DAgger}\\[-2pt]{\footnotesize (\%) $\downarrow$}\end{tabular} &
\begin{tabular}[c]{@{}c@{}}\textbf{FERT}\\[-2pt]{\footnotesize (\%) $\downarrow$}\end{tabular} &
\begin{tabular}[c]{@{}c@{}}\textbf{Relative failure}\\[-2pt]\textbf{reduction} {\footnotesize (\%) $\uparrow$}\end{tabular} \\
\midrule
Flower Arrangement & 1.67 & \textbf{0.83} & 50.0 \\
Clothes Folding & 14.17 & \textbf{1.67} & 88.2 \\
Phone Charging & 40.00 & \textbf{2.50} & 93.8 \\
Stool Assembly & 18.33 & \textbf{6.67} & 63.6 \\
Phone Packaging & 20.83 & \textbf{1.67} & 92.0 \\\midrule
\rowcolor{black!5}
\textbf{Task avg} & 19.00 & 2.67 & 86.0 \\
\bottomrule
\end{tabular}
\end{minipage}
\end{table}

Compared with the evaluated RLT and RECAP configurations, FERT achieves higher success rates on all five tasks, with average gains of 12.3 and 11.5 percentage points, respectively. Relative to RLT, the largest gain is on Phone Charging, where success increases from 75.8\% to 97.5\%. Relative to RECAP, the largest gain is on Phone Packaging, where success increases from 75.0\% to 98.3\%. These comparisons establish FERT’s performance advantage under the evaluated training configurations. The effect of increasing fleet size for HG-DAgger and FERT is examined separately in \Cref{sec:exp-scale}.

The methods also differ in how they use experience and supervision after initialization. RLT performs online residual reinforcement learning on a single robot, while our RECAP implementation performs one round of offline policy improvement on multi-robot data. Both use additional human corrective actions during subsequent policy learning. FERT instead learns from successful autonomous executions, using binary task outcomes and failure-guided retrieval to prioritize training experience without additional expert demonstrations or corrective actions. It also avoids learning a value function for advantage estimation. Together, these results show that, under the evaluated task and training configurations, FERT uses fleet-generated actions and task-outcome feedback to improve task success beyond the shared expert-supervised initialization without collecting further expert demonstrations or corrective actions.

\begin{figure}[htbp]
    \centering
    \includegraphics[width=\linewidth]{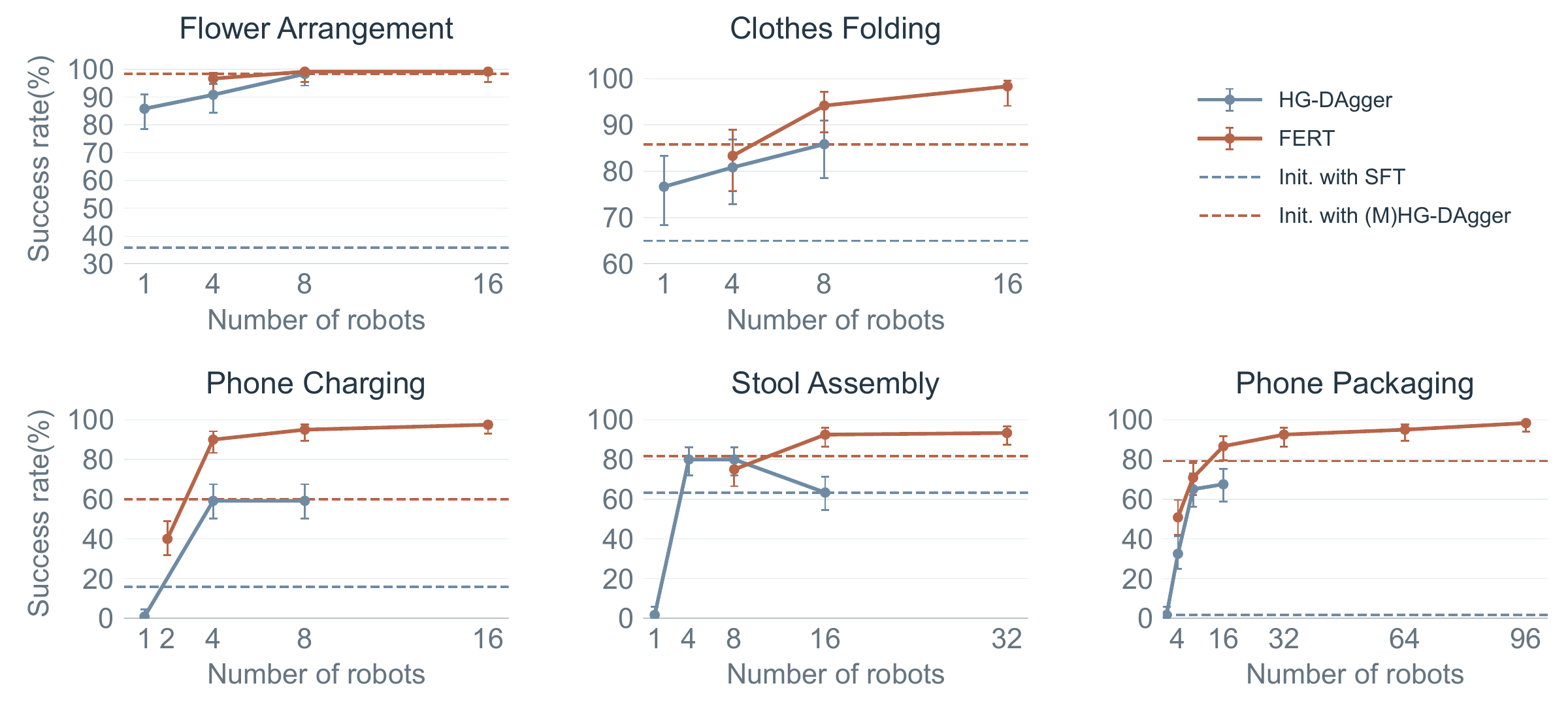}
    \caption{Scaling with fleet size across five tasks. Each panel shows success rate as a function of the number of training robots. The blue and pink curves represent HG-DAgger and FERT, respectively. HG-DAgger starts from SFT, whereas FERT starts from the selected (M)HG-DAgger checkpoint; horizontal dashed lines in different colours indicate the success rates of these initial policies. Error bars indicate 95\% Wilson confidence intervals.}
    \label{fig:scale-task}
\end{figure}

\subsection{Scaling with fleet size}
\label{sec:exp-scale}
We examine how policy performance changes for FERT and our multi-robot implementation of HG-DAgger as the number of robots contributing experience increases (\Cref{fig:scale-task}). Each panel shows success rate as a function of the number of training robots for one task. The two methods start from different policies: HG-DAgger begins from SFT, whereas FERT begins from the selected (M)HG-DAgger checkpoint. The success rates of the two initial policies are indicated by horizontal dashed lines in different colours. Within each task and method, collection time per robot is held fixed across fleet sizes, so larger fleets provide more experience for policy improvement.

For (M)HG-DAgger, single-robot training on the relatively challenging Phone Charging, Stool Assembly and Phone Packaging tasks can yield lower success rates than the SFT initialization. This may reflect overfitting to the environment of an individual robot and the behaviour of a single operator, limiting generalization to other robots. Performance improves as the fleet expands, but further increases in robot numbers lead to saturation on Phone Charging and Phone Packaging and a decline on Stool Assembly. This pattern motivates examining learning with less demanding action supervision as fleets grow. Limited human monitoring and intervention capacity is a possible explanation for the observed saturation, alongside changes in the collected data.

\textbf{For FERT, success rates generally increase with fleet size across all five tasks.} On Phone Packaging, expanding the training fleet from 16 to 96 robots raises success from 86.7\% to 98.3\%, a gain of 11.6 percentage points. At smaller fleet sizes, FERT can still perform below its initialization. For example, training with eight robots reduces success rates relative to the initial policies on both Stool Assembly and Phone Packaging. One possible explanation is that successful experience collected by smaller fleets is insufficiently rich to cover the different failure contexts encountered by the policy, limiting the effectiveness of failure-guided retrieval and subsequent policy updates. 
Together, these results show more consistent performance gains for FERT as fleet size and total collected experience increase across the evaluated tasks and scales. These findings support using experience generated during deployment to improve a shared policy as more robots contribute data. FERT collects this experience with binary task-outcome labels, without requiring additional expert demonstrations or corrective actions.

\begin{figure}[htbp]
    \centering
    \begin{subfigure}[t]{0.21\linewidth}
        \includegraphics[width=\linewidth]{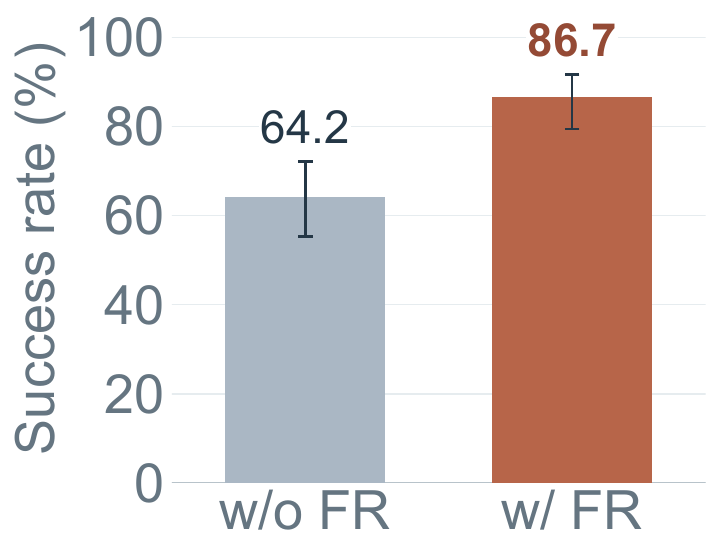}
        \caption{Failure guidance}
        \label{fig:exp-abl-fr}
    \end{subfigure}\hfill%
    \begin{subfigure}[t]{0.28\linewidth}
        \includegraphics[width=\linewidth]{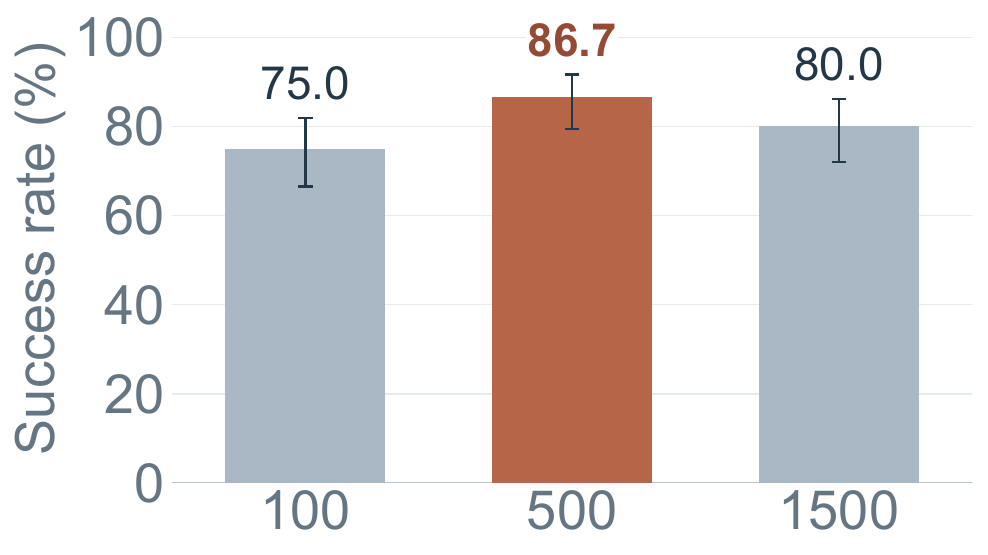}
        \caption{Weight sync interval}
        \label{fig:exp-abl-int}
    \end{subfigure}\hfill%
    \begin{subfigure}[t]{0.21\linewidth}
        \includegraphics[width=\linewidth]{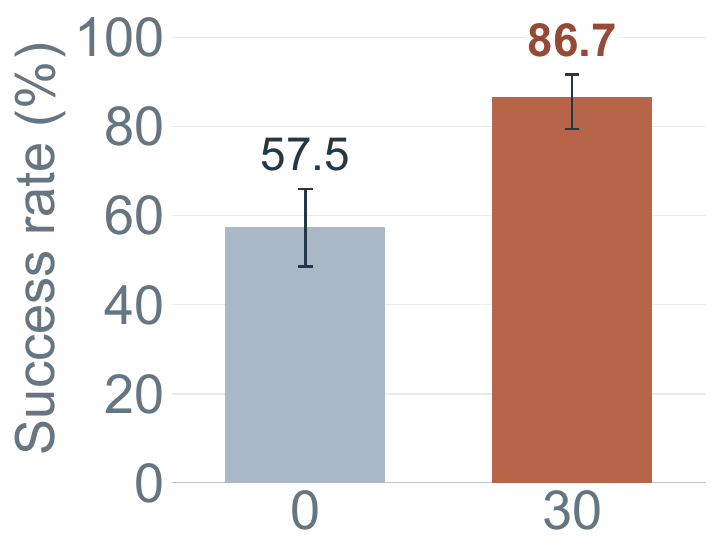}
        \caption{\centering Retrieved-segment extension (frames)}
        \label{fig:exp-abl-extend}
    \end{subfigure}\hfill%
    \begin{subfigure}[t]{0.28\linewidth}
        \includegraphics[width=\linewidth]{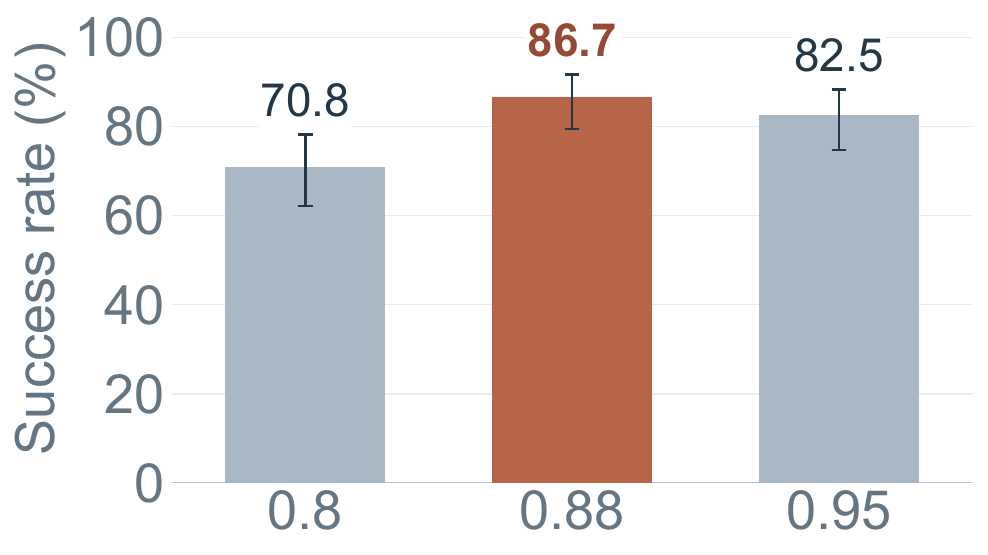}
        \caption{Priority retention factor $\rho$}
        \label{fig:exp-abl-decay}
    \end{subfigure}
    \caption{Ablation studies of FERT on Phone Packaging with 16 training robots. Pink bars indicate the full FERT configuration, while gray bars indicate the corresponding ablated variants. In (a), FR denotes failure-guided retrieval. In (b), the interval is measured in optimizer steps. In (c), each value specifies both $h_{\mathrm{pre}}$ and $h_{\mathrm{post}}$, with $L=90$ frames held fixed. In (d), $\rho$ is the priority retention factor; larger values retain previous priorities for longer. Error bars indicate 95\% Wilson confidence intervals.}
    \label{fig:ablation}
\end{figure}
\subsection{Ablation studies}
\label{sec:exp-abl}

We study FERT's design choices on Phone Packaging with 16 training robots (\Cref{fig:ablation}). The full configuration achieves a success rate of 86.7\%. We examine the contributions of failure-guided retrieval, policy synchronization frequency, retrieved-segment extension and priority retention. \Cref{tab:detail-abl} reports the success counts on each held-out test robot and the pooled success rate for every ablation configuration.

\paragraph{Contribution of failure-guided retrieval.}
We compare the full FERT configuration (w/ FR), which uses failure contexts to retrieve relevant successful segments and increase their sampling priorities, with a no-retrieval variant (w/o FR), which disables failure-guided retrieval and samples uniformly from the collective experience buffer. All other settings are held fixed. Removing failure-guided retrieval reduces success from 86.7\% to 64.2\%, a decrease of 22.5 percentage points. This result demonstrates the benefit of selecting successful experience based on the policy's failure contexts, directing learning towards behaviours relevant to its current failures. This provides empirical evidence for the experience-selection mechanism; \Cref{thm:imp} analyzes its local effect under explicit replay and retrieval-alignment assumptions.

\paragraph{Effect of policy synchronization frequency.}
Policy synchronization determines how quickly experience collection incorporates updated policies. Synchronizing every 500 optimizer steps achieves the highest success rate among the tested settings, 86.7\%. Increasing the interval to 1,500 steps reduces success to 80.0\%, while shortening it to 100 steps reduces success to 75.0\%. Longer intervals delay the use of updated policies, whereas very short intervals may allow less experience to accumulate under each policy version. Within the tested settings, these results favour a moderate synchronization interval for coupling policy refresh with experience collection.

\paragraph{Effect of retrieved-segment extension.}
FERT extends each retrieved interval to include neighbouring actions, providing continuous action context before and after the matched segment (\Cref{sec:ret}). With the maximum query length fixed at $L=90$ frames, setting $h_{\mathrm{pre}}=h_{\mathrm{post}}=30$ frames achieves 86.7\% success. Removing both extensions ($h_{\mathrm{pre}}=h_{\mathrm{post}}=0$) reduces success to 57.5\%, a decrease of 29.2 percentage points. This result highlights the importance of the actions leading into and following a matched segment for learning successful behaviour.

\paragraph{Sensitivity to priority retention.}
The retention factor $\rho$ controls the persistence of historical sampling priorities across retrieval updates (\Cref{sec:weighting}). The default setting $\rho=0.88$ achieves 86.7\% success, compared with 70.8\% at $\rho=0.8$ and 82.5\% at $\rho=0.95$. Within the tested settings, both faster and slower priority decay yield lower success rates than the default, favouring an intermediate retention factor. This pattern is consistent with a trade-off between retaining accumulated retrieval information and responding to recent failures.

\subsection{System support for fleet learning}
\label{sec:exp-system}

For learning to continue during deployment, FATS must keep the CEL loop running on a fleet whose robots fail, pause and progress at different rates, while connecting robot interaction to cloud inference and training (\Cref{sec:infra}).
We examine experience-collection continuity, fault recovery and communication in a 96-robot training run that lasted from 10:37 to 18:30. The robots paused from 12:00 to 13:00 during the operators’ lunch break. Operators provided episode-outcome labels, monitored robot safety and reset environments; they provided no demonstrations or corrective actions.

The evaluation addresses three questions.
\begin{enumerate}
    \item How do robots fail, and does HotSwarm confine each fault to the robots it affects?
    \item How does measured collection compare with trace-replay estimates for synchronous collection and for restarting the system after each fault?
    \item Can SwarmLink sustain cloud inference and experience collection across the fleet?
\end{enumerate}
Fault and collection statistics are computed from the per-robot episode and incident intervals provided in \Cref{app:infra_traces}.

\begin{figure}[t]
    \centering
    \includegraphics[width=\linewidth]{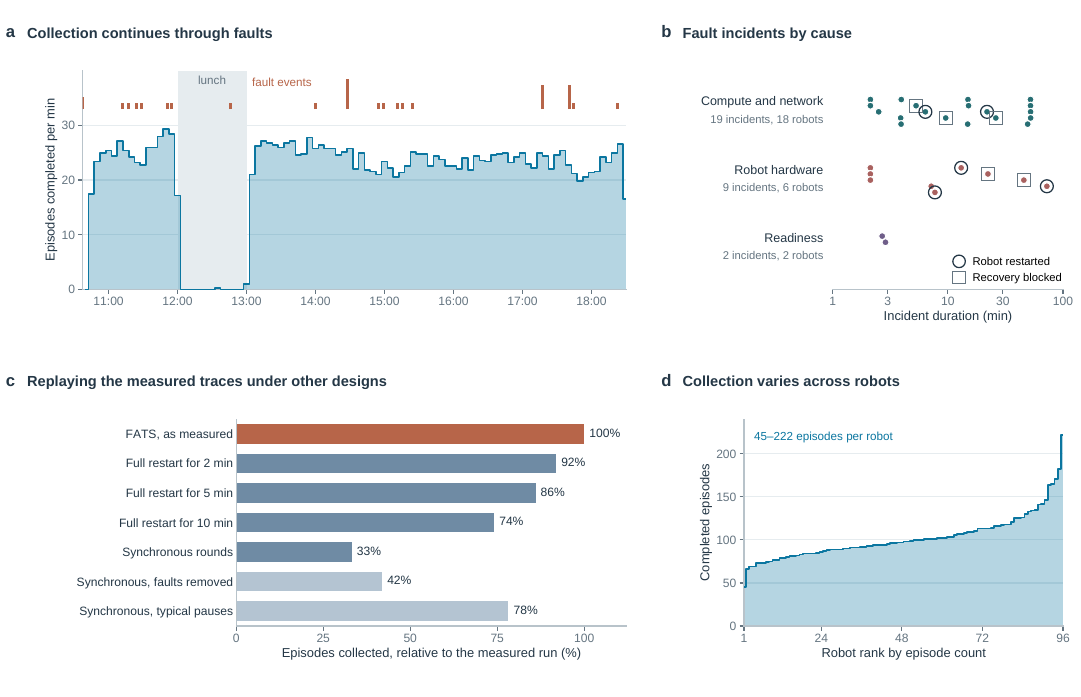}
    \caption{\textbf{FATS in the 96-robot training run.}
        \textbf{a}, Episodes completed per minute in 5-min bins.
        Marks above the curve show fault events, with height proportional to the number of robots affected,
        and the shaded interval is the lunch break.
        \textbf{b}, Fault incidents by initiating cause and duration.
        Circles mark incidents in which the robot was restarted,
        and squares mark incidents in which recovery was blocked.
        \textbf{c}, Episodes collected when each robot's measured episodes and pauses are replayed under other designs,
        relative to the run.
        A full restart stops all robots for 2, 5 or 10 minutes after each fault event,
        and synchronous rounds make each policy update wait for one episode from every robot.
        The two synchronous variants replace the pauses that overlap a fault, or all pauses, with the robot's median pause.
        \textbf{d}, Completed episodes per robot, sorted by count.
        Statistics are computed from the detailed traces in \Cref{app:infra_traces}.}
    \label{fig:fats-trace}
\end{figure}

\subsubsection{Faults and their containment}
\label{sec:exp-profile-overview}
\label{sec:err-detail}
The fleet collected 9,834 episodes during the run.
These executions supply the successful trajectories and failure contexts used by FERT, linking system availability to the supply of fleet experience.

\textbf{Faults arose from several sources and varied in both their scope and recovery time.} 
We recorded 30 fault incidents on 22 of the 96 robots (\Cref{fig:fats-trace}b).
By initiating cause, 19 incidents began with a compute or network fault, e.g., host memory exhaustion or a network interruption,
    9 with a hardware fault of the chassis, joints, battery or sensors,
    and 2 with a robot that was not ready, e.g., powered off or awaiting SDK startup.
Compute and network faults were the most widespread, affecting 18 robots with a median duration of 15\,min.
Hardware faults recurred on fewer robots, with 9 incidents on 6 robots,
    and lasted a median of 8\,min but up to 72\,min.
Some faults interrupted several robots at once.
Host memory exhaustion, for example, stopped robots 44 to 47 within 2\,min of each other.
Grouping the incidents that began within 2\,min of each other yields 19 fault events,
    i.e., one every 22\,min across the fleet.
Five incidents required restarting the robot,
    which took 3 to 11\,min, except for one restart that spanned the lunch break.
In five other incidents, recovery was blocked for 5 to 26\,min,
    and four of these ended together at 15:21.

Fleet-level collection continued through the recorded local fault events (\Cref{fig:fats-trace}a).
Outside the lunch break, 65 robots were executing at a time on average,
    and the fleet completed 24 episodes per minute.
For the 14 fault events that occurred at least 10\,min from the lunch break and from either end of the run,
    we compare the episodes completed in the 10\,min after the event with those in the 10\,min before it.
After a median event, the fleet completed 99\% as many episodes as before it, and never fewer than 88\%.
Collection resumed after 23 of the 30 incidents,
    with a median of 12\,s between the end of an incident and the robot's next episode.
Of the remaining seven, three were still open when the run ended,
    and robots 44 to 47 did not return after their host memory fault at 17:17.
In total, incidents removed 8.3 robot-hours, 1.3\% of the fleet's 660 robot-hours of operation.
Continued collection and subsequent re-admission are consistent with the local recovery design in \Cref{sec:infra-hotswarm}. HotSwarm relaunches the affected worker at its station and repairs the affected connections, while SwarmLink rejects stale data identified with the failed worker. These designs allow available workers to keep supplying experience during local recovery.

\subsubsection{Decoupled progress}
\label{sec:exp-decoupled}
\textbf{Robots progress at different rates even without faults.}
The median episode lasted 3.2\,min, but per-robot medians ranged from 0.6 to 4.2\,min,
    and robots completed between 45 and 222 episodes each (\Cref{fig:fats-trace}d).
Resets and faults add pauses of varying length between episodes.

To quantify the benefits of independent collection and recovery,
    we replay each robot's measured sequence of episodes and pauses under two alternative designs.
The morning and afternoon are replayed as separate sessions,
    because the lunch break stops all robots at the same time of day.

\paragraph{Synchronous rounds.}
Each policy update waits until every robot that is still collecting has finished its next episode,
    and a robot starts its next episode once the round has begun and its own pause has elapsed.
The replay favors this alternative by ignoring the time to update the policy between rounds. In the trace replay, synchronous rounds collect only 33\% of the episodes of the run (\Cref{fig:fats-trace}c).
A round lasts a median of 9.3\,min, more than twice a robot's typical cycle of 4.1\,min,
    because it ends only when the robot with the longest pause or episode finishes.
Faults explain a small part of this loss.
Replacing each pause that overlaps a fault with the robot's median pause raises synchronous collection to 42\%,
    and replacing every pause in this way raises it to 78\%.
The remaining 22\% is lost because episode durations and typical pauses differ across robots and from one episode to the next.
\textbf{These results show that independent collection avoids waiting caused by uneven execution times and pauses, even in the absence of faults.}

\paragraph{Full restart.}
Every fault event stops all robots for a specified restart duration of 2, 5 or 10\,min,
    and the episodes in progress are discarded and repeated.
The replay favors this alternative by ignoring the repair time of the failed robot, which a full restart may also have to wait for. The full-restart replay preserves 92\%, 86\% and 74\% of the episodes for restart durations of 2, 5 and 10\,min, respectively.
The analytical model in \Cref{app:infra_model} predicts corresponding collection-time fractions of 91.6\%, 81.3\% and 68.5\%.
The replay measures completed episodes, whereas the model estimates time available for collection; both indicate larger losses with longer restarts. \textbf{Independent recovery avoids extending these interruptions to robots unaffected by the fault.}

Under the analytical model's fixed per-robot fault-rate and downtime assumptions, whole-system restart loses a larger fraction of collection time as fleet size grows. Local recovery confines the recorded loss to affected robots, 1.3\% in this run, and the model predicts a fleet-size-independent expected fraction for that loss.

\subsubsection{Edge-cloud communication}
\label{sec:exp-communication}
We profile communication during the same run.
Robot interaction runs at the edge, while policy inference and optimization run in the cloud,
    and the two clusters are connected through a dedicated 30\,Gbps link (\Cref{app:infra_deployment}).
We measure the elapsed times of send and receive operations on three channels (\Cref{fig:exp-channel}).
On SwarmLink's interaction path, Env-to-Rollout carries observations, including images and robot states, to inference workers,
    and Rollout-to-Env returns action chunks for execution.
Env-to-Train records experience-transfer operations for training.
These channel-operation durations do not separate network transfer from waiting for data or computation.
They characterize communication demands rather than isolating the benefit of individual mechanisms such as batching.

Interaction messages are frequent and small.
Three $224\times224$ RGB images contain approximately 0.45\,MB of raw uint8 image data per observation,
    excluding proprioceptive state and serialization overhead.
Action chunks sent through Rollout-to-Env require approximately 1.88\,kB/s per robot.
Episodes are larger and less frequent.
A representative Phone Packaging episode contains approximately 3,500 recorded timesteps and occupies 1.58\,GB,
    and robots in the run completed an episode about every 4\,min.

\begin{figure}[t]
    \centering
    \includegraphics[width=\linewidth]{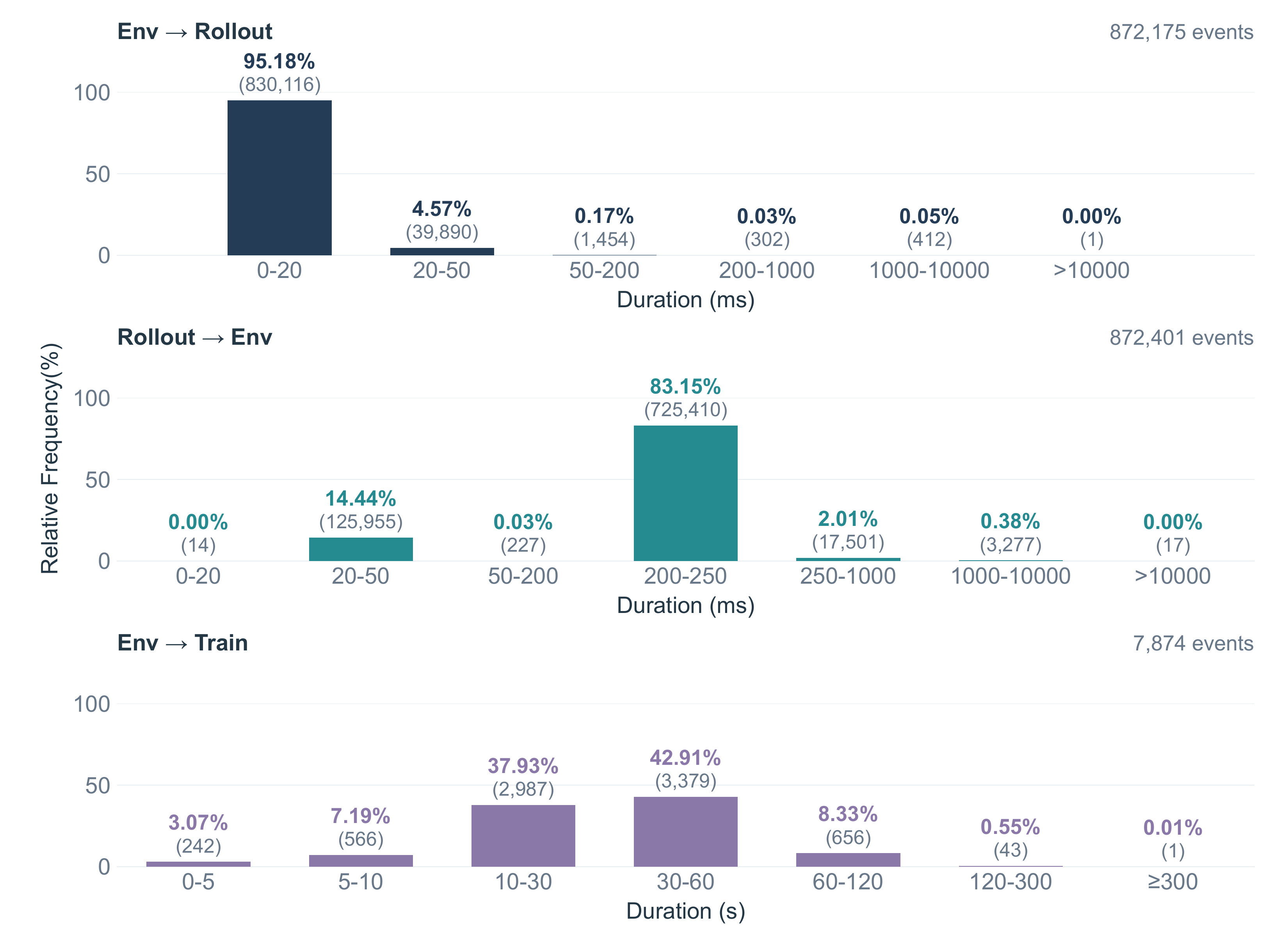}
    \caption{Measured communication-operation durations across observation, action and experience channels. Bars report the fraction and count of operations in each duration interval,
        and each panel shows the total number of events.}
    \label{fig:exp-channel}
\end{figure}

For Env-to-Rollout, 95.18\% of measured operations finish within 20\,ms and more than 99\% within 50\,ms.
For Rollout-to-Env, 14.44\% finish within 50\,ms and approximately 97.6\% within 250\,ms.
Most action-channel operations, 83.15\%, fall in the 200--250\,ms interval.
Although this channel offers room for optimization, operations exceeding 1\,s account for less than 1\% of measurements.

More than 99\% of Env-to-Train operations finish within two minutes.
This difference in timescales motivates serving inference requests independently of experience delivery (\Cref{sec:infra-swarmlink}).
Since a robot completes an episode about every 4\,min, whole-episode transfers keep pace with collection,
    and episodes do not accumulate at the robot site.
\FloatBarrier

\section{Related Works}
\label{sec:related_work}

\subsection{Robot Learning from Real-world Experience}

Real-world interaction provides experience for improving robot policies beyond their initial training. Interactive imitation learning addresses the distribution mismatch between expert demonstrations and policy execution by collecting expert supervision on states encountered by the learner. DAgger~\citep{ross2011reduction} formalizes this approach through iterative dataset aggregation. Its human-gated variant, HG-DAgger~\citep{kelly2019hg}, allows human experts to decide when to take control and collects corrective behaviour during these interventions, making the interaction protocol better suited to real-world human supervision. DART~\citep{laskey2017dart} addresses distribution mismatch by injecting noise during demonstration collection, while ThriftyDAgger~\citep{hoque2021thriftydagger} allocates a limited intervention budget using novelty and risk. Reinforcement learning (RL) methods also combine expert data with autonomous interaction. HIL-SERL~\citep{luo2024hilserl} uses demonstrations, interventions and autonomous experience for manipulation learning via RLPD~\citep{ball2023efficient} algorithm backbone. ConRFT~\citep{chen2025conrft} incorporates human interventions into reinforcement learning for vision-language-action policies. RECAP~\citep{intelligence2025pi06vlalearnsexperience} integrates demonstrations, autonomous rollouts and expert corrections for policy improvement. Related offline-to-online methods include AWAC~\citep{nair2020awac}, which combines prior data with online interaction using advantage-weighted actor updates. 

These approaches differ in how expert-provided behaviour and autonomous experience contribute to policy updates. Our work focuses on policy improvement from experience generated by deployed robots without collecting additional expert actions. Following expert-supervised initialization, CEL connects fleet experience collection, shared policy updates and subsequent execution. Binary task outcomes provide evaluative feedback, while the actions used for policy improvement come from the fleet itself.

\subsection{Fleet-scale Robot Learning}

Fleet-scale learning combines experience from multiple robots to improve shared policies. \citet{ghasemipour2025selfimproving} use steps-to-go predictions to construct rewards and success detectors, enabling autonomous practice across a robot fleet with limited human supervision. SOP~\citep{pan2026sop} connects fleet execution to online learning through an asynchronous system that collects execution experience and human intervention signals and distributes updated policies. LWD~\citep{wang2026lwd} connects deployment and policy improvement through fleet-scale offline-to-online reinforcement learning using autonomous rollouts and human interventions.

These studies establish that parallel robot execution can support shared policy improvement. CEL builds on this direction by connecting fleet-generated experience with shared policy updates during deployment, with FERT providing the concrete learning algorithm evaluated here. We investigate this learning process without additional expert action supervision, using training fleets of up to 96 robots. With collection time per robot held fixed, we evaluate how policy performance changes as the training fleet and the total experience it generates increase.

\subsection{Experience Retrieval and Prioritized Learning}
Retrieval and replay prioritization determine which experience receives greater emphasis during training. Behavior Retrieval~\citep{du2023behavior} uses task-specific expert data to query relevant behaviours from a larger offline dataset. Graph Search and Retrieval~\citep{yin2024offline} uses graph-based value estimates to retrieve actions for imitation. STRAP~\citep{memmel2025strap} uses target demonstrations to retrieve relevant sub-trajectories from an offline dataset. Prioritized experience replay~\citep{schaul2016prioritized} allocates updates according to temporal-difference errors, whereas self-imitation learning~\citep{oh2018selfimitation} reinforces past actions whose returns exceed a learned value estimate. Advantage-weighted regression (AWR)~\cite{peng2019advantage} fits a policy to replayed actions weighted by estimated advantages. AWAC~\citep{nair2020awac} similarly uses advantage-weighted maximum-likelihood actor updates. These methods differ in the information used to select experience and in how the selected data influence policy updates.

FERT uses contexts from failed autonomous executions as queries into successful fleet experience. Retrieved segments receive increased sampling priority during supervised policy updates, while failed actions are not used as imitation targets. The retrieval queries therefore come from the fleet's own execution rather than target demonstrations, and the priorities do not require a learned value estimate. Successful executions supply action targets, while failed attempts guide the allocation of updates among those targets. FERT provides a concrete mechanism for using fleet experience within the broader CEL paradigm.

\subsection{Distributed Systems for Robot Learning}

Distributed learning systems improve throughput by separating experience collection from policy optimization. Ape-X~\citep{horgan2018distributed} aggregates experience from parallel actors in a shared prioritized replay buffer, while IMPALA~\citep{espeholt2018impala} decouples acting and learning and uses off-policy correction to account for differences between the actors' policies and the policy being optimized. SEED RL~\cite{espeholt2019seed} further separates environment interaction from centralized policy inference to improve accelerator utilization. RLlib~\citep{liang2018rllib} provides reusable abstractions for composing and distributing RL computations.

In robotics, \citet{gu2017deep} demonstrated asynchronous off-policy learning across multiple robots, and \citet{yahya2017collective} developed distributed asynchronous guided policy search to train a shared visuomotor policy. QT-Opt~\citep{kalashnikov2018scalable} combined real-robot data collection with distributed replay and asynchronous optimization for large-scale grasping. More recently, SOP~\citep{pan2026sop} connected fleet execution to a centralized learner with asynchronous policy distribution for VLA policy learning. RLinf-USER~\citep{zang2026rlinfuser} provides infrastructure based on Ray~\citep{moritz2018ray} for heterogeneous robots and cloud--edge communication, combining asynchronous learning with persistent experience storage and crash recovery.

These studies establish asynchronous execution and shared learning infrastructure as foundations for robot learning at scale. FATS builds on this direction to sustain CEL across a fleet, allowing policy improvement to continue despite interruptions to individual robots. Within our approach, FERT determines how successful and failed executions contribute to policy updates, while FATS supports the distributed process through which fleet experience is collected and used. Together, they realize a learning loop whose subsequent action supervision comes from robot execution, with binary outcome feedback supplied by operators in our experiments.

\section{Conclusion}
We investigated learning from fleet experience as a basis for improving robot policy reliability beyond expert-supervised initialization. Collective Experience Learning (CEL) connects experience generated by deployed robots to shared policy improvement, shifting subsequent action supervision from newly collected expert actions to robot executions. FERT provides the learning mechanism through failure-guided retrieval of successful experience, while FATS supports asynchronous collection and learning despite local interruptions.
Across five real-world manipulation tasks, the resulting policies achieved success rates of 93.3–99.2\% on robots excluded from training data collection. With collection time per robot held fixed, larger fleets provided more experience and generally improved performance across the evaluated sizes. On phone packaging, learning with 96 robots raised success from 79.2\% to 98.3\%. During the same run, local fault events occurred on average every 22 minutes, yet fleet-level experience collection continued, with fault-related downtime accounting for only 1.3\% of total robot time.
The current study evaluates task-specific policies on a single robot platform and retains human-provided outcome labels and environment resets. Extending CEL to broader task distributions, heterogeneous robots and longer periods of deployment remains future work. Within the evaluated setting, these results support CEL as a scalable learning paradigm in which expanding deployment supplies experience for improving shared policies, without requiring a corresponding expansion of expert action supervision.

\newpage

\bibliographystyle{plainnat}
\bibliography{references}  %

\newpage
\appendix
\crefalias{section}{appendix}
\crefalias{subsection}{appendix}
\crefalias{subsubsection}{appendix}
\crefname{appendix}{Appendix}{Appendices}
\Crefname{appendix}{Appendix}{Appendices}
\section{Author List and Contributions}
\label{sec:contibution}

\subsection*{Author list}
Chao Yu\textsuperscript{1,*}, Shu'ang Yu\textsuperscript{1,*}, Hao Lin\textsuperscript{3,*}, Hongzhi Zang\textsuperscript{1,*}, Yuxiang Yan\textsuperscript{2,*}, Chenxi Jiang\textsuperscript{2,*}, Zhen Guo\textsuperscript{3}, Chuankang Li\textsuperscript{2}, Yihao Chen\textsuperscript{1}, Xin Xu\textsuperscript{1}, Xu Fu\textsuperscript{3}, Jiakai Zhou\textsuperscript{1}, Liangzhi Shi\textsuperscript{1}, Hongyi Ren\textsuperscript{1}, Yiwei Liu\textsuperscript{1}, Jing Wu\textsuperscript{2}, Yuanqing Wang\textsuperscript{3}, Chen Wu\textsuperscript{2}, Bo Dai\textsuperscript{3}, Chunyu Yang\textsuperscript{2}, Hongliang Huang\textsuperscript{2}, Shuoyu Yue\textsuperscript{2}, Weiqi Yin\textsuperscript{2}, Yinuo Chen\textsuperscript{1}, Ruihua Chen\textsuperscript{3}, Chunyang Zhu\textsuperscript{3}, Han Xiao\textsuperscript{2}, Yanghao Li\textsuperscript{3}, Ruipu Wu\textsuperscript{2}, Zhikai Cai\textsuperscript{2}, Xiu Chen\textsuperscript{2}, Xu Tan\textsuperscript{2}, Si Liu\textsuperscript{2}, Yuchen Yan\textsuperscript{3}, Jialin Zhang\textsuperscript{3}, Wen Zeng\textsuperscript{3}, Yiming Zhou\textsuperscript{3}, Zixiao Huang\textsuperscript{3}, Xinlei Chen\textsuperscript{1}, Guanghui Ren\textsuperscript{2}, Quanlu Zhang\textsuperscript{3}, Boxun Li\textsuperscript{3}, Maoqing Yao\textsuperscript{2,\dag}, Yu Wang\textsuperscript{1,\dag}

\medskip
\noindent\textsuperscript{1} Tsinghua University\\
\textsuperscript{2} AgiBot\\
\textsuperscript{3} Infinigence AI

\medskip
\noindent\textsuperscript{*} These authors contributed equally to this work.\\
\textsuperscript{\dag} Corresponding authors: Maoqing Yao and Yu Wang.

\subsection*{Author Contributions}
\paragraph{Project leadership.} Chao Yu served as the overall project leader. Guanghui Ren and Quanlu Zhang served as the project leaders at AgiBot and Infinigence AI, respectively.

\paragraph{Algorithm development.} Shu'ang Yu and Hongzhi Zang contributed to algorithm development.

\paragraph{Infrastructure.} At Infinigence AI, Hao Lin, Zhen Guo, Xu Fu, Yuanqing Wang, Bo Dai, Ruihua Chen, Chunyang Zhu, Yanghao Li, Yuchen Yan, Jialin Zhang, Wen Zeng, Yiming Zhou, and Zixiao Huang contributed to the design, implementation, optimization, and validation of FATS.

At AgiBot, Yuxiang Yan, Chenxi Jiang, Chuankang Li, Chunyu Yang, Han Xiao, Zhikai Cai, Xu Tan and Si Liu contributed to the early system design and optimization of FATS, the development of the AgiBot multi-robot management platform and its integration with FATS, and the development of infrastructure for real-robot experimentation and deployment.

\paragraph{Experiments.} Shu'ang Yu, Hongzhi Zang, Xin Xu, Jiakai Zhou, Liangzhi Shi, Hongyi Ren, Yiwei Liu, Hongliang Huang, Weiqi Yin, Yinuo Chen, and Ruipu Wu contributed to experimental execution and evaluation. Xiu Chen conducted experimental testing. Yihao Chen provided the experimental site.

\paragraph{Baseline implementation.} Chen Wu and Shuoyu Yue contributed to baseline implementation.

\paragraph{Manuscript writing.} Chao Yu, Shu'ang Yu, Hongzhi Zang, Hao Lin, Liangzhi Shi and Quanlu Zhang contributed to writing the manuscript. Xinlei Chen provided advice on the manuscript.

\paragraph{Project management.} Jing Wu managed overall project coordination and execution.

\paragraph{Website and video.} Yinuo Chen, Shu'ang Yu, and Chao Yu developed the project website. Liangzhi Shi and Chao Yu produced the initial version of the project video.

\paragraph{Advisory support.} Boxun Li, Maoqing Yao, and Yu Wang served as project advisors.

\subsection*{Acknowledgements}
We thank Anwei Lin, Xuan Hu, Cheng Ruan, Wei Shan, Yuehan Niu, Xiaofan Zhang, Hui Fang, Yi Xu, Hui Wang, Yao Chen, Shubin Wei, Jinlei Pan, Jia Zeng, Jianbo Mao, Rushuai Yang, Kailun Yuan, Chiming Liu, Zhichao Wu, Rui Zhang, Weihui Zhao, Xiaohan Yan, Houlin Li, Bin Xia, Shiyun Zhang, Qinman Wu, Geng Xi, Hanwen Shi, Hui Zhao, Yujie Chen, Tong Liu, Ying Qiu, Min Qin, Haotian Du, Maoping Li, Wei Shen, Zhaoxiang Zhang, Oghurjan Niyaz, Ziwen Song, Zhuotong Han, Xin Chen, Jie Yang, Wenping Shan, Binbin Qiu, Bowen Jiang, Xiongfeng Cai, Jingyuan Wang, Hui Min, Liyan Zhang, Ken Chen, Haoyu Cao, Hongliang Zhu, Cheng Jing, and Sheng Qu for their contributions and support to overall project site coordination, robot maintenance and calibration, robot data, algorithm design, infrastructure, robot system design, and robot software versions.

\newpage
\section{Theoretical Analysis of FERT}
\label{app:fert_theory}

This appendix provides the setting and proofs supporting \Cref{sec:understanding_fert}. It analyzes a local likelihood-based mean-gradient update as a model of successful-experience fitting and retrieval prioritization. We use the same notation throughout.
All policy gradients are evaluated at $\theta_0$, and task
conditioning is suppressed. The replay distribution and retrieval
priorities remain fixed during each update.

\subsection{Proof of Lemma 1: Successful-Replay Alignment}
\label{app:lemma1}
\begin{proof}
We consider episodes of length $T_c$ in action chunks. Early
termination can be represented by parameter-independent absorbing
steps. The observation $o$ includes any history required by the
policy, and each $A$ is the action chunk sampled at that step.

For the analysis, we model $\mathcal{D}^{+}$ by sampling an episode
from the current policy $\pi_{\theta_0}$, conditioning on
$y(\tau)=1$, and selecting one of its $T_c$ chunks uniformly.
This is a population approximation to the historical successful
buffer. Assume $J(\theta_0)>0$ and $\mathbf{v}\neq0$.

Let $p_\theta(\tau)$ denote the trajectory likelihood. Since the
initial-state distribution and environment dynamics are independent
of $\theta$, only the policy factors contribute to its logarithmic
derivative:
\begin{equation}
    \left.
    \nabla_\theta\log p_\theta(\tau)
    \right|_{\theta=\theta_0}
    =
    \sum_{t=0}^{T_c-1}
    \nabla_\theta\log\pi_{\theta_0}(A_t\mid o_t).
    \label{eq:fert_app_trajectory_score}
\end{equation}
Under the regularity assumptions in the main text, differentiating
the expected return therefore gives
\begin{equation}
    \begin{aligned}
        \mathbf{v}
        &=
        \left.
        \nabla_\theta
        \int y(\tau)p_\theta(\tau)\,d\tau
        \right|_{\theta=\theta_0}\\
        &=
        \mathbb{E}_{\tau\sim\pi_{\theta_0}}
        \left[
            y(\tau)
            \sum_{t=0}^{T_c-1}
            \nabla_\theta\log\pi_{\theta_0}(A_t\mid o_t)
        \right]\\
        &=
        J(\theta_0)
        \mathbb{E}_{\tau\sim\pi_{\theta_0}}
        \left[
            \sum_{t=0}^{T_c-1}
            \nabla_\theta\log\pi_{\theta_0}(A_t\mid o_t)
            \,\middle|\,
            y(\tau)=1
        \right].
    \end{aligned}
    \label{eq:fert_app_policy_gradient}
\end{equation}
The last equality uses the binary outcome $y(\tau)$.

Uniform chunk sampling from these successful episodes gives
\begin{equation}
    \begin{aligned}
        \mathbf{G}_{\mathrm{unif}}
        &=
        \frac{1}{T_c}
        \mathbb{E}_{\tau\sim\pi_{\theta_0}}
        \left[
            \sum_{t=0}^{T_c-1}
            \nabla_\theta\log\pi_{\theta_0}(A_t\mid o_t)
            \,\middle|\,
            y(\tau)=1
        \right]\\
        &=
        \frac{\mathbf{v}}{T_cJ(\theta_0)}.
    \end{aligned}
\end{equation}
Consequently,
\begin{equation}
    \begin{aligned}
        \mu
        &=
        \mathbb{E}_{\mathcal{D}^{+}}
        \left[
            \mathbf{v}^{\top}
            \nabla_\theta\log\pi_{\theta_0}(A\mid o)
        \right]\\
        &=
        \mathbf{v}^{\top}\mathbf{G}_{\mathrm{unif}}
        =
        \frac{\|\mathbf{v}\|_2^2}{T_cJ(\theta_0)}
        >0.
    \end{aligned}
    \label{eq:fert_app_success_alignment}
\end{equation}
This proves Lemma 1.
\end{proof}

\subsection{Proof of Lemma 2: Retrieval Alignment}
\label{app:lemma2}

\begin{proof}
We make the stage-wise setting precise using a robot task with
$T_c$ sequential stages, such as reaching, grasping, and insertion.
Each stage is represented by one action chunk in this model.
Stage $k$ has an independent trainable logit $\theta_k$ and a local failure rate $\phi_k=1-\operatorname{sigmoid}(\theta_k)$, conditional on reaching that stage. The task receives return one exactly when all stages are completed, so
\begin{equation}
    J(\theta)=\prod_{k=1}^{T_c}(1-\phi_k).
\end{equation}
All quantities below are evaluated at $\theta_0$.

Within each stage, the relative likelihoods of successful action
chunks are fixed. Specifically, for a successful chunk at stage
$k$, write
\begin{equation}
    \pi_\theta(A\mid o_k)=(1-\phi_k)q_k(A\mid o_k).
\end{equation}
where $q_k$ is a parameter-independent distribution over successful
chunks. This model isolates learning within each stage without
parameter interference between stages.

The logistic parameterization gives
\begin{equation}
    \frac{\partial\phi_k}{\partial\theta_k}
    =-\phi_k(1-\phi_k),
    \qquad
    \frac{\partial J}{\partial\theta_k}
    =J(\theta_0)\phi_k,
    \qquad
    \frac{\partial\log(1-\phi_k)}{\partial\theta_k}
    =\phi_k.
\end{equation}
Hence the score of a successful chunk from stage $k$ is
$\phi_k\mathbf{e}_k$, where $\mathbf{e}_k$ is the $k$th
coordinate vector. Its first-order contribution is therefore
\begin{equation}
    u_k
    =
    \mathbf{v}^{\top}
    \nabla_\theta\log\pi_{\theta_0}(A\mid o_k)
    =
    J(\theta_0)\phi_k^2.
    \label{eq:fert_app_stage_usefulness}
\end{equation}
Thus, under this model, successful experience from a stage with a
higher local failure rate has a larger first-order contribution
to expected return.

Every successful episode contains all stages, so uniform successful
replay samples stages uniformly. Let $\omega_k$ be the nonnegative
retrieval weight assigned to stage $k$, normalized such that
$T_c^{-1}\sum_k\omega_k=1$. We assume
\begin{equation}
    \phi_k>\phi_\ell
    \quad\Longrightarrow\quad
    \omega_k\geq\omega_\ell,
    \label{eq:fert_app_retrieval_order}
\end{equation}
with at least one pair satisfying
$\phi_k>\phi_\ell$ and $\omega_k>\omega_\ell$.

Using the normalization of the retrieval weights, the covariance becomes
\begin{equation}
\begin{aligned}
    \kappa
    &= \frac{1}{T_c}\sum_{k=1}^{T_c}(\omega_k-1)u_k \\
    &= \frac{J(\theta_0)}{T_c^2}
       \sum_{k<\ell}
       (\omega_k-\omega_\ell)(\phi_k^2-\phi_\ell^2).
\end{aligned}
\end{equation}
The second equality follows by expanding the pairwise differences. Every term in the final sum is nonnegative by \Cref{eq:fert_app_retrieval_order}, and at least one is strictly positive. Since $J(\theta_0)>0$, we obtain $\kappa>0$. This proves Lemma 2. 

\end{proof}

The ordering condition concerns local failure rates conditional on
reaching a stage. It specifies the retrieval quality required by
the result; it is not guaranteed by nonuniform priorities alone.

\subsection{Proof of Theorem 1: Lower bound on the FERT improvement}
\label{app:thm1}

\begin{proof}
We first bound the higher-order effects of a finite update.
For any displacement $\Delta$ whose entire update segment lies
in the smooth neighborhood, the fundamental theorem of calculus
gives
\begin{equation}
    \begin{aligned}
        &J(\theta_0+\Delta)-J(\theta_0)
        -\mathbf{v}^{\top}\Delta\\
        &\qquad=
        \int_0^1
        [\nabla J(\theta_0+t\Delta)-\mathbf{v}]^{\top}
        \Delta\,dt.
    \end{aligned}
\end{equation}
Since $\nabla J$ is $\lambda_J$-Lipschitz, the absolute value of
this remainder is bounded by
\begin{equation}
    \begin{aligned}
        &\left|
            J(\theta_0+\Delta)-J(\theta_0)
            -\mathbf{v}^{\top}\Delta
        \right|\\
        &\qquad\leq
        \int_0^1\lambda_J t\|\Delta\|_2^2\,dt
        =
        \frac{\lambda_J}{2}\|\Delta\|_2^2.
    \end{aligned}
    \label{eq:fert_app_remainder}
\end{equation}

\paragraph{Uniform replay improves expected return.}
By the definition of $\mu$,
\begin{equation}
    \mathbf{v}^{\top}\mathbf{G}_{\mathrm{unif}}
    =
    \mathbb{E}_{\mathcal{D}^{+}}[u]
    =
    \mu.
\end{equation}
Applying the lower bound in \Cref{eq:fert_app_remainder} with
$\Delta=\eta\mathbf{G}_{\mathrm{unif}}$ gives
\begin{equation}
    J(\theta_{\mathrm{unif}})-J(\theta_0)
    \geq
    \eta\mu
    -
    \frac{\lambda_J\eta^2}{2}
    \|\mathbf{G}_{\mathrm{unif}}\|_2^2.
    \label{eq:fert_app_uniform_bound}
\end{equation}
Lemma 1 establishes $\mu>0$. Therefore, this difference is
strictly positive whenever the update remains in the smooth
neighborhood and
\begin{equation}
    0<\eta<
    \frac{2\mu}
         {\lambda_J\|\mathbf{G}_{\mathrm{unif}}\|_2^2}.
    \label{eq:fert_app_uniform_stepsize}
\end{equation}

\paragraph{FERT improves more than uniform replay.}
Because $\mathbb{E}_{\mathcal{D}^{+}}[\omega]=1$,
\begin{equation}
    \begin{aligned}
        \mathbf{v}^{\top}
        (\mathbf{G}_{\mathrm{FERT}}-\mathbf{G}_{\mathrm{unif}})
        &=
        \mathbb{E}_{\mathcal{D}^{+}}[\omega u]
        -
        \mathbb{E}_{\mathcal{D}^{+}}[u]\\
        &=
        \operatorname{Cov}_{\mathcal{D}^{+}}(\omega,u)
        =
        \kappa.
    \end{aligned}
\end{equation}
Apply the lower bound in \Cref{eq:fert_app_remainder} to the FERT
update and the upper bound to the uniform update. Subtracting gives
\begin{equation}
    \begin{aligned}
        J(\theta_{\mathrm{FERT}})
        -
        J(\theta_{\mathrm{unif}})
        \geq\;&
        \eta\kappa\\
        &-
        \frac{\lambda_J\eta^2}{2}
        \left(
            \|\mathbf{G}_{\mathrm{FERT}}\|_2^2
            +
            \|\mathbf{G}_{\mathrm{unif}}\|_2^2
        \right).
    \end{aligned}
    \label{eq:fert_app_retrieval_bound}
\end{equation}
Lemma 2 establishes $\kappa>0$. Hence this difference is strictly
positive whenever both updates remain in the smooth neighborhood
and
\begin{equation}
    0<\eta<
    \frac{2\kappa}
         {\lambda_J
          \left(
              \|\mathbf{G}_{\mathrm{FERT}}\|_2^2
              +
              \|\mathbf{G}_{\mathrm{unif}}\|_2^2
          \right)}.
    \label{eq:fert_app_retrieval_stepsize}
\end{equation}

Combining the two comparisons, a sufficient step-size condition is
\begin{equation}
    0<\eta<
    \min\left\{
        \frac{2\mu}
             {\lambda_J\|\mathbf{G}_{\mathrm{unif}}\|_2^2},
        \frac{2\kappa}
             {\lambda_J
              \left(
                  \|\mathbf{G}_{\mathrm{FERT}}\|_2^2
                  +
                  \|\mathbf{G}_{\mathrm{unif}}\|_2^2
              \right)}
    \right\},
    \label{eq:fert_app_stepsize}
\end{equation}
together with the neighborhood requirement. Both upper bounds are
positive under Lemmas 1 and 2, so a sufficiently small positive
step size satisfies these conditions and yields
\begin{equation}
    J(\theta_{\mathrm{FERT}})
    >
    J(\theta_{\mathrm{unif}})
    >
    J(\theta_0).
\end{equation}

\paragraph{Decomposition of the FERT improvement.}
The normalization of $\omega$ also gives
\begin{equation}
    \mathbf{v}^{\top}\mathbf{G}_{\mathrm{FERT}}
    =
    \mathbb{E}_{\mathcal{D}^{+}}[\omega u]
    =
    \mu+\kappa.
\end{equation}
Applying the lower bound in \Cref{eq:fert_app_remainder} with
$\Delta=\eta\mathbf{G}_{\mathrm{FERT}}$ gives
\begin{equation}
    J(\theta_{\mathrm{FERT}})-J(\theta_0)
    \geq
    \eta\mu+\eta\kappa
    -
    \frac{\lambda_J\eta^2}{2}
    \|\mathbf{G}_{\mathrm{FERT}}\|_2^2.
\end{equation}
This proves the lower bound in \Cref{eq:fert_improvement_decomposition} and completes the proof of Theorem 1.
\end{proof}

The guarantee applies to one mean-gradient update under the
current-policy replay approximation and the stage-wise retrieval
model. Retrieval contributes the positive first-order term
$\eta\kappa$, while the step-size conditions control the
higher-order effects of the update. These conditions define the scope of the local result; repeated learning with historical replay and the flow-matching implementation is evaluated experimentally.

\section{Policy Initialization}
\label{app:policy_init}
\label{app:policy_preparation}

This appendix details the expert-supervised policy initialization used in our experiments. For each task, we adapt the $\pi_0$ policy through supervised fine-tuning (SFT) on human demonstrations, followed by human-gated dataset aggregation (HG-DAgger) with human interventions collected across the robot fleet. The selected task-specific checkpoint is then used as the shared initialization for FERT, RECAP and RLT. For FERT, this expert-supervised stage precedes the CEL learning loop; subsequent action targets come from autonomous fleet executions.

\subsection{Supervised fine-tuning on human demonstrations}
\label{sec:sft}

We first fine-tune the base policy on a fixed dataset of human demonstrations, $\mathcal{D}_{\mathrm{demo}}$. Each training sample contains an observation $o_t$, a task instruction $g$, and a demonstrated action chunk $A_t\in\mathbb{R}^{H\times d}$. Here, $H$ is the action horizon, $d$ is the action dimension, and $o_t$ contains multi-view images and proprioceptive state.

We use the conditional flow-matching objective of the base policy. For each action chunk, we sample $\varepsilon_{\mathrm{noise}}\sim\mathcal{N}(0,I)$ and a flow time $\alpha\in[0,1]$ from the base model's time-sampling distribution.
The interpolated action chunk is
\begin{equation}
    A_t^\alpha=(1-\alpha)A_t+\alpha\varepsilon_{\mathrm{noise}},
    \label{eq:prep_flow_interpolation}
\end{equation}
with target velocity $\varepsilon_{\mathrm{noise}}-A_t$. The training objective is
\begin{equation}
    \mathcal{L}_{\mathrm{SFT}}(\theta)
    =
    \mathbb{E}_{
        (o_t,g,A_t)\sim\mathcal{D}_{\mathrm{demo}},
        \,\alpha,\,\varepsilon_{\mathrm{noise}}
    }
    \left[
        \frac{1}{Hd}
        \left\|
            v_\theta(A_t^\alpha,\alpha\mid o_t,g)
            -(\varepsilon_{\mathrm{noise}}-A_t)
        \right\|_F^2
    \right],
    \label{eq:flow_matching}
\end{equation}
where $\theta$ denotes the policy parameters and $v_\theta$
is the predicted velocity. The noise
$\varepsilon_{\mathrm{noise}}\in\mathbb{R}^{H\times d}$ has independent standard
normal entries. The variable $\alpha$ denotes flow time, whereas
$t$ indexes physical execution steps.
We estimate this objective using mini-batches and update the policy by gradient descent.
The resulting checkpoint initializes HG-DAgger.

\subsection{Multi-robot HG-DAgger with human interventions}
\label{app:hgdagger}

We implement HG-DAgger~\citep{kelly2019hg} across multiple robots, aggregating human intervention data to update a shared task-specific policy. This multi-robot implementation is labelled (M)HG-DAgger in \Cref{fig:main-results}. Each robot executes the current policy while a human operator monitors execution and takes over when correction is needed. After the intervention, autonomous execution can resume. This procedure collects expert actions in states encountered by the policy.

Each successful trajectory is stored in its entirety, together with a per-frame indicator identifying whether each action was generated by the policy or supplied by a human operator. Only eligible expert action chunks collected during human takeovers, hereafter referred to as intervention chunks, are used as action targets for supervised policy updates.

\begin{figure}[t]
    \centering
    \includegraphics[width=0.9\linewidth]
        {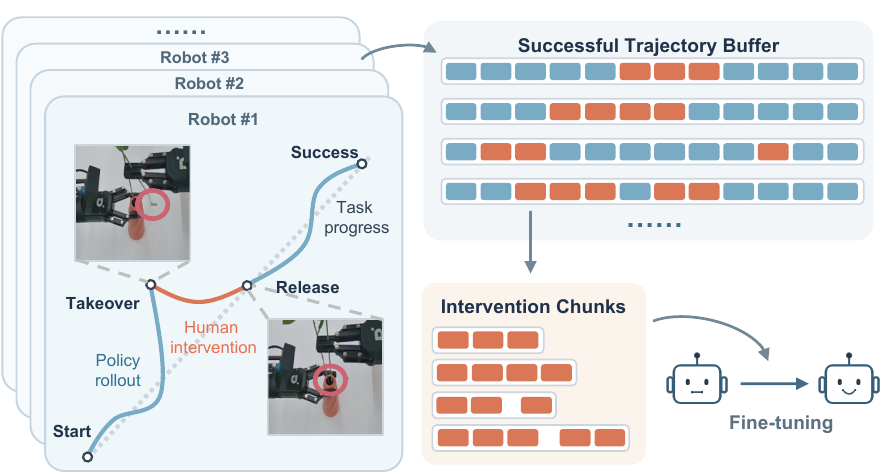}
    \caption{
        \textbf{Multi-robot HG-DAgger with human interventions.}
        Human operators take over execution when correction is needed. Successful trajectories are stored with intervention indicators, and eligible intervention chunks are used for supervised policy updates.
    }
    \label{fig:hgdagger}
\end{figure}

\subsubsection{Selecting Intervention Chunks}
\label{app:intervention_filter}

Let $m_t\in\{0,1\}$ indicate whether the recorded action at timestep $t$ was supplied through human intervention. A candidate chunk is eligible only if the action at every valid timestep within its horizon was supplied by a human operator:
\begin{equation}
    \chi_t
    =
    \prod_{t'=t}^{\min(t+H,T)-1}m_{t'},
    \label{eq:prep_intervention_mask}
\end{equation}
where $T$ is the episode length.
Only action chunks with starting timesteps satisfying $\chi_t=1$ are eligible for sampling. We typically use $H=30$, corresponding to a nominal one-second horizon at a control frequency of $30\,\mathrm{Hz}$.

We additionally discard the first $K_{\mathrm{int}}$ candidate starts in each contiguous intervention segment. This filter is intended to reduce the influence of transient actions around the switch from policy control to human control, including hesitation and abrupt teleoperation adjustments. Note that the intervention-filter parameter $K_{\mathrm{int}}$ is distinct from the number of retrieved trajectories used by FERT.

\subsubsection{Supervised Fine-tuning with Loss Filtering}
\label{app:dagger_loss_filter}

The policy is trained with the same conditional flow-matching objective used for SFT, with eligible intervention chunks as action targets.

For sample $i$ in a mini-batch $\mathcal{B}$, let $(o_i,g_i,A_i)$ denote its observation, instruction and action chunk.
After sampling $\alpha_i$ and $\varepsilon_{\mathrm{noise},i}$, we construct
\begin{equation}
    A_i^{\alpha_i}
    =
    (1-\alpha_i)A_i+\alpha_i\varepsilon_{\mathrm{noise},i}.
\end{equation}
The sample-wise loss is
\begin{equation}
    \ell_i
    =
    \frac{1}{Hd}
    \left\|
        v_\theta(A_i^{\alpha_i},\alpha_i\mid o_i,g_i)
        -(\varepsilon_{\mathrm{noise},i}-A_i)
    \right\|_F^2.
    \label{eq:dagger_sample_loss}
\end{equation}
Here, $i$ indexes training samples rather than trajectory timesteps.

Inspection of training batches revealed a high-loss tail containing approximately $5$-$10\%$ of samples. To limit the influence of these samples, we apply a batch-level loss filter. The threshold combines a loss quantile with an absolute floor:
\begin{equation}
    \ell_{\mathrm{cut},\mathcal{B}}
    =
    \max\left(
        \operatorname{quant}_q
        \left(\{\ell_i:i\in\mathcal{B}\}\right),
        \ell_{\min}
    \right),
    \label{eq:prep_loss_threshold}
\end{equation}
where $\operatorname{quant}_q$ denotes the empirical $q$-quantile, with $q\in(0,1)$, and $\ell_{\min}\geq0$ is a lower bound on the rejection threshold. The retained set is
\begin{equation}
    \mathcal{V}_{\mathcal{B}}
    =
    \left\{
        i\in\mathcal{B}:
        \ell_i\leq\ell_{\mathrm{cut},\mathcal{B}}
    \right\},
    \label{eq:prep_retained_batch}
\end{equation}
and the policy update minimizes
\begin{equation}
    \mathcal{L}_{\mathrm{DAgger}}(\theta;\mathcal{B})
    =
    \frac{1}{|\mathcal{V}_{\mathcal{B}}|}
    \sum_{i\in\mathcal{V}_{\mathcal{B}}}\ell_i.
    \label{eq:dagger_loss}
\end{equation}

Gradients are propagated only through retained samples. The filter acts on the sampled batch and does not remove frames or trajectories from the stored dataset. It deliberately biases optimization towards lower-loss samples within each batch; a large training loss is not, by itself, evidence that a recorded human action is incorrect.

\subsubsection{Online Rollout and Policy Synchronization}
\label{app:dagger_execution}

We run multi-robot HG-DAgger on FATS for asynchronous collection and optimization during initialization. FATS provides the shared infrastructure, while HG-DAgger determines how human interventions contribute to learning. Environment workers execute actions and record intervention indicators, rollout workers provide policy inference, and training workers optimize the policy using eligible intervention chunks.
Collection and optimization proceed concurrently.

Training begins once each training worker has accumulated sufficient valid data.
Data balancing regulates uneven accumulation across workers, and
updated policy parameters are distributed to rollout workers
every $K_{\mathrm{sync}}$ optimizer steps through non-blocking asynchronous
weight synchronization, where $K_{\mathrm{sync}}$ denotes the synchronization interval. During human control, the environment worker skips repeated policy inference to avoid unnecessary inference requests and facilitate teleoperation. 

\section{Training System Deployment}
\label{app:infra_deployment}

This appendix specifies the hardware and placement used to run the CEL collection-and-update loop through FATS.
It covers the hardware and network used in this study, the placement of workers on these resources,
    and the model used to compare recovery strategies in \Cref{fig:fats-concept}e.
The SwarmLink data plane and the HotSwarm control plane are described in \Cref{sec:infra-swarmlink,sec:infra-hotswarm}.
FATS is built on RLinf~\citep{yu2026rlinf},
    which manages distributed processes through Ray~\citep{moritz2018ray}.

\subsection{Robot and computing resources}
\label{app:infra_resources}

The deployment supports 100 robots,
    and the largest training fleet evaluated in this study comprises 96 robots.
Robot interaction takes place at the collection site,
    while cloud resources serve policy inference and optimization.
The cloud computing cluster consists of six servers, each equipped with eight NVIDIA A800 GPUs,
    for a total of 48 cloud GPUs.
Network-attached storage is available at the robot site,
    and a General Parallel File System (GPFS) provides shared storage in the cloud.

\begin{figure}[htbp]
    \centering
    \includegraphics[width=0.9\linewidth]{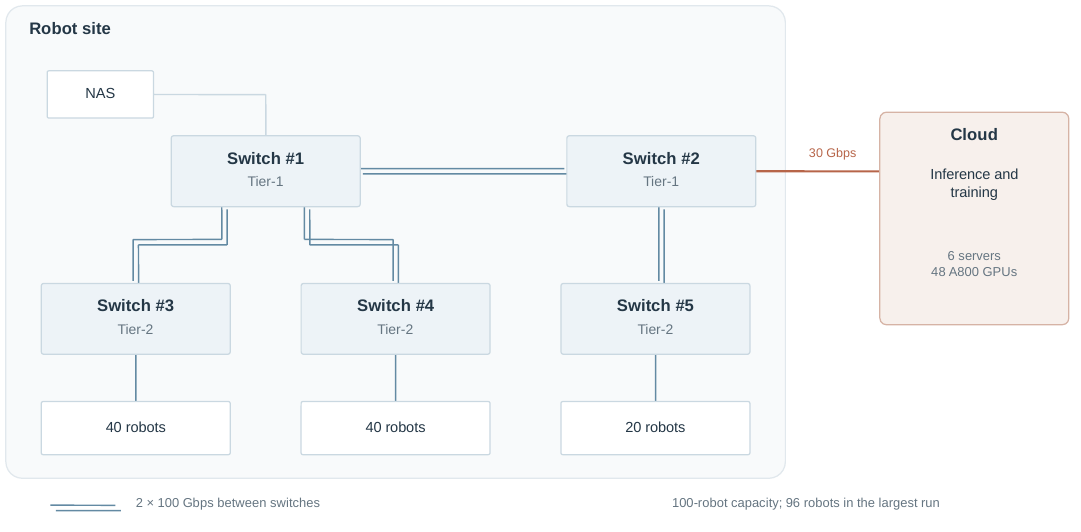}
    \caption{Robot-site network and cloud resources for FATS. The site supports 100 robots, with 96 used in the largest training run. Six cloud servers provide 48 NVIDIA A800 GPUs for inference and policy optimization. The robot site and cloud communicate through a dedicated 30\,Gbps link.}
    \label{fig:hardware-link}
\end{figure}

\subsection{Network topology}
\label{app:infra_network}

The robot site and the cloud are connected through a dedicated 30\,Gbps link.
The site network contains two tier-1 switches and three tier-2 switches,
    with dual 100\,Gbps fiber links between interconnected switches.
Cloud servers are located in the provider's data center,
    while the robots operate at a single collection site.
The dedicated cloud connection terminates at switch~2.
\Cref{fig:hardware-link} summarizes the hardware deployment and network topology.

\subsection{Worker roles and placement}
\label{app:infra_placement}

\paragraph{Resource description.}
Each node reports the hardware that workers can use.
Accelerators are discovered through their runtimes.
Robots are registered explicitly,
    together with their communication endpoints and camera identifiers,
    and their connectivity can optionally be checked before use.
Resources are organized into named groups,
    and a node may belong to several groups.

\paragraph{Placement.}
A placement specification assigns environment, inference and training workers to resource groups.
Each worker is bound to its accelerators or robot
    and to the software environment of the node it runs on.
The same specification supports inference co-located with robot control or placed on a separate cluster.

\paragraph{Deployment placement.}
Environment workers acquire observations and execute actions at the robot site.
Inference and training workers run in the cloud,
    so observation requests and action replies cross the dedicated link.
Episode accumulators run inside the environment workers at the robot site,
    and experience stores run inside the training workers in the cloud,
    so each completed episode crosses the link once.

\subsection{Recovery cost and fleet size}
\label{app:infra_model}

\paragraph{Model.}
We compare two recovery strategies through a model of robot time available for experience collection.
Let $N$ be the number of robots and $\lambda$ the rate of fault events per robot,
    so that the fleet experiences fault events at rate $N\lambda$.
If the whole system restarts after each event and a restart takes time $t_{\mathrm{restart}}$,
    every event stops all robots for time $t_{\mathrm{restart}}$.
The mean interval of collection between fleet-level fault events is $1/(N\lambda)$.
Adding a restart duration $t_{\mathrm{restart}}$ to each such interval gives the share of time available for collection:
\begin{equation}
    A_{\mathrm{sys}}=\frac{1}{1+N\lambda t_{\mathrm{restart}}}.
    \label{eq:recovery_sys}
\end{equation}
If only the affected robot restarts,
    each incident removes that robot from collection while the other robots continue.
The share of robot time available for collection is then
\begin{equation}
    A_{\mathrm{robot}}=1-\frac{\sum_{i} \Delta t_i}{N T_{\mathrm{run}}},
    \label{eq:recovery_robot}
\end{equation}
where $T_{\mathrm{run}}$ is the observation-window duration excluding the lunch break,
    and $\Delta t_i$ is the recorded robot downtime due to incident $i$ within that window.
Under fixed per-robot incident-rate and downtime assumptions, the expected number of incidents grows in proportion to $N$, so the expected value of $A_{\mathrm{robot}}$ does not depend on fleet size.

\paragraph{Parameters.}
We estimate the parameters from the 96-robot training run shown in \Cref{fig:episode-record,fig:error}.
The run lasted from its first episode at 10:37 to 18:30.
Excluding the operators' lunch break from 12:00 to 13:00 gives $T_{\mathrm{run}}=6.88$\,h.
During this period, 22 robots experienced 30 fault incidents.
Several incidents had a shared cause,
    e.g., host memory exhaustion that interrupted four robots at once.
We therefore group incidents that began within 2\,min of each other into one fault event,
    which gives 19 events, i.e., $\lambda=0.0288$ per robot-hour,
    or one event every 22\,min across the fleet.
The recorded incidents removed 8.3 robot-hours from the fleet's approximately 660 robot-hours of operation,
    giving a fault loss of 1.3\% and $A_{\mathrm{robot}}=98.7\%$.
\Cref{fig:error} draws short incidents with a minimum width of about 2\,min,
    so these losses are upper bounds.
The run did not include whole-system restarts,
    so $t_{\mathrm{restart}}$ is not measured and we evaluate $t_{\mathrm{restart}}=2$, 5 and 10\,min.

\paragraph{Results.}
For 96 robots, \Cref{eq:recovery_sys} gives $A_{\mathrm{sys}}=91.6\%$, 81.3\% and 68.5\% for $t_{\mathrm{restart}}=2$, 5 and 10\,min, respectively.
With the same per-robot event rate,
    a fleet of 1,024 robots would have $A_{\mathrm{sys}}=50.4\%$, 28.9\% and 16.9\%,
    whereas $A_{\mathrm{robot}}$ remains 98.7\%.
\Cref{fig:fats-concept}e shows these values as a function of fleet size.

\paragraph{Assumptions.}
The model assumes that per-robot fault rates and downtime distributions do not depend on fleet size, and that faults on different robots are independent apart from the grouping described above.
It favors whole-system restart in two respects.
It ignores the repair time of the failed robot, which a whole-system restart may also have to wait for,
    and it ignores the episodes in progress on other robots, which a whole-system restart discards.
The model considers only the recorded robot-local incidents,
    rather than shared failures of cloud inference or the edge-cloud link.
Environment resets affect both strategies and are excluded.

\section{Task Setup and Evaluation Criteria}
\label{app:task_success_criteria}
This appendix describes the observation preprocessing and action representation shared by SFT, (M)HG-DAgger, FERT, RECAP and RLT, followed by the task-specific evaluation criteria. These specifications define the common robot interface and task-success measurements used to evaluate the CEL realization and its baselines.

\subsection{Observation preprocessing and action representation}
\label{app:shared_observation_action}

\paragraph{Observation preprocessing.}
The policy input consists of an observation $o_t$ and a task instruction $g$. The observation $o_t$ comprises three RGB images and a 16-dimensional proprioceptive state. The images come from one head-mounted camera with resolution $400\times640$ and two wrist-mounted cameras with resolution $1056\times1280$ each (height $\times$ width). The state contains seven joint measurements per arm and two gripper measurements.

All methods use the same geometric image transformations. Each image is cropped to a task-relevant region according to \Cref{tab:task-prompts-crops}, anisotropically resized to a square whose side length equals the crop's shorter spatial dimension, and finally resized to $224\times224$, without padding. 

\paragraph{Action representation.}
The VLA predicts action chunks
\begin{equation}
    A_t=(a_t,\ldots,a_{t+H-1})\in\mathbb{R}^{H\times d},\qquad H=30,\quad d=16,
\end{equation}
where $d$ is the action dimension and each $a_{t+k}$ is a $d$-dimensional action containing seven joint commands per arm and two gripper commands. Joint-position deltas within a chunk are relative to the robot's joint configuration at the start of that chunk. Each action dimension is normalized using training-dataset statistics, which are also used to transform predictions back to the robot's action space at inference time. RLT executes a prefix of the VLA reference action chunk during its RL phase; its residual policy and execution horizon are described in \Cref{app:rlt_baseline}.

\subsection{Task success and failure criteria}

All learned-policy evaluation trials are executed autonomously, without human intervention or corrective actions. This section specifies the task completion requirements and failure conditions used to score these trials.

\begin{longtable}{@{}
  >{\raggedright\arraybackslash}p{\dimexpr.26\linewidth-\tabcolsep\relax}
  >{\raggedright\arraybackslash}p{\dimexpr.37\linewidth-\tabcolsep\relax}
  >{\raggedright\arraybackslash}p{\dimexpr.37\linewidth-2\tabcolsep\relax}@{}}
\caption{Task prompts and observation crop parameters. Wrist crops apply to both the left and right cameras.}\label{tab:task-prompts-crops}\\
\toprule

\endfirsthead
\multicolumn{3}{@{}l}{\footnotesize\itshape Table \thetable\ continued}\\[4pt]
\toprule

\endhead
\midrule
\multicolumn{3}{r@{}}{\footnotesize\itshape Continued on next page}\\
\endfoot
\bottomrule

\endlastfoot
\rowcolor{black!5}
\textbf{Task} & \multicolumn{2}{l@{}}{\textbf{Language prompt}}\\*
Flower arrangement & \multicolumn{2}{l@{}}{\lq\lq Put the flowers into the vase.\rq\rq} \\
Clothes folding & \multicolumn{2}{l@{}}{\lq\lq Fold the cloth.\rq\rq} \\
Phone charging & \multicolumn{2}{l@{}}{\lq\lq Charge the phone.\rq\rq} \\
Stool assembly & \multicolumn{2}{l@{}}{\lq\lq Assemble the stool with a screwdriver.\rq\rq} \\
Phone packaging & \multicolumn{2}{l@{}}{\lq\lq Pack the phone.\rq\rq} \\
\addlinespace[6pt]
\rowcolor{black!5}
\textbf{Task} & \textbf{Head-camera crop box} & \textbf{Two wrist-cameras crop box}\\*
Flower arrangement & $(120,180,330,460)$ & $(100,304,772,976)$ \\
Clothes folding & $(120,50,400,590)$ & $(200,400,856,1056)$ \\
Phone charging & $(140,60,400,580)$ & $(300,304,748,976)$ \\
Stool assembly & $(130,60,350,580)$ & $(100,304,772,976)$ \\
Phone packaging & No cropping & $(192,304,864,976)$ \\
\end{longtable}

\paragraph{Common evaluation rules.}
A trial is successful only if the robot completes the prescribed task instruction and satisfies all task-specific requirements below. Skipping a required step or executing the steps in the wrong order constitutes failure. For Flower Arrangement, Phone Charging and Phone Packaging, remaining idle or making no substantive progress for more than $8\,\mathrm{s}$ at any step constitutes failure; the corresponding threshold for Stool Assembly is $25\,\mathrm{s}$. These are limits on lack of progress within a step, rather than total task-duration limits.

An empty grasp is a grasp attempt that fails to engage the target object. Where an empty grasp is explicitly listed as a step-level failure below, a single occurrence causes immediate failure. For other grasping steps in Flower Arrangement, Phone Charging, Stool Assembly and Phone Packaging where retries are permitted, three consecutive empty grasps cause failure; a fourth attempt is not made.

\paragraph{Flower Arrangement.}
The following steps must be completed in order under the common evaluation rules:
\begin{enumerate}
    \item The left gripper grasps the vase horizontally at its mouth. An empty grasp, an incorrect grasp position, displacement of the vase by more than $2\,\mathrm{cm}$, or knocking it over constitutes failure.
    \item The right gripper picks up rose A and inserts it into the vase. An empty grasp, grasping multiple items, failing to insert the rose, or pulling it out while withdrawing the arm constitutes failure.
    \item The right gripper picks up rose B and inserts it into the vase, subject to the same grasping and insertion criteria as rose A.
    \item The right gripper picks up the leaf and inserts it into the vase, subject to the same grasping and insertion criteria as the roses.
    \item The left gripper releases the vase and returns to its resting configuration. Displacing the vase by more than $2\,\mathrm{cm}$ or knocking it over constitutes failure. Both roses and the leaf must remain inside the upright vase.
\end{enumerate}

\paragraph{Clothes Folding.}
The following steps must be completed in order under the common evaluation rules:
\begin{enumerate}
    \item Grasp the sleeve and hem corner on one side with both grippers and fold that side inward.
    \item Grasp one side of the garment with both grippers and drag it inward to the edge of the table.
    \item Grasp the sleeve and hem corner on the opposite side with both grippers and fold that side inward, mirroring Step 1.
    \item Grasp the bottom of the garment with the right gripper and fold it upward in half.
    \item Smooth out the garment with the right gripper.
\end{enumerate}

\textit{Garment allocation.} During training, each robot cycles through pink, green and blue garments, switching to the next colour after $5$--$10$ trials with the current colour. During evaluation, each held-out robot performs $20$ consecutive trials with each colour, yielding $60$ trials per robot across the three colours and $120$ trials across the two evaluation robots.

\paragraph{Phone Charging.}
The following steps must be completed in order under the common evaluation rules:
\begin{enumerate}
    \item The left arm picks up the charging cable from the table, and the right arm picks up the charging adapter.
    \item The left arm moves behind the power socket to support alignment, and the right arm inserts the adapter plug into the socket.
    \item The left arm moves behind the USB-C port, while the right gripper grasps the charging cable near the USB-C connector.
    \item The left arm grasps the phone, and the right arm unplugs the USB-C connector.
    \item The left arm places the phone on the phone box and pushes it into the required position.
\end{enumerate}

\paragraph{Stool Assembly.}
The following steps must be completed in order under the common evaluation rules. Poking the table at any step constitutes failure. Apart from empty-grasp attempts, which follow the common rule and the immediate-failure conditions below, attempting the same action more than three times constitutes failure.
\begin{enumerate}
    \item The right arm rests diagonally against the stool in the designated position near screw holes A and B. Incorrect positioning constitutes failure.
    \item The left arm grasps the upper part of screw B and seats it correctly in hole B. An empty grasp, an incorrect grasp position, dropping the screw, or incorrect seating constitutes failure.
    \item The left arm grasps the upper part of screw A and seats it correctly in hole A, subject to the same grasping and placement criteria as screw B.
    \item The right arm grasps the middle of the screwdriver. An empty grasp, an incorrect grasp position, or dropping the screwdriver constitutes failure.
    \item The right arm aligns the screwdriver with screw A \& B and tightens them separately. Dropping the screwdriver, misalignment, or failure to tighten the screw constitutes failure. Both screws must be correctly seated and tightened at completion.
\end{enumerate}

\paragraph{Phone Packaging.}
The following steps must be completed in order under the common evaluation rules. Every specified grasp, transfer and placement must be completed; empty grasps may be retried, but the third consecutive empty grasp ends the trial.
\begin{enumerate}
    \item Pick up the top phone case from the stack and place it to the right of the phone box.
    \item Pick up the first manual from its rack and place it on the phone case.
    \item Pick up the first charging cable from the rack on the right and place it in the designated cable compartment inside the phone box.
    \item Use the left arm to pick up the first charging adapter from the rack on the left.
    \item Transfer the adapter from the left gripper to the right gripper in mid-air.
    \item Use the right arm to place the adapter upright in its designated compartment, then tip it into its final position.
    \item Pick up the phone case together with the manual, place both inside the box, and arrange them flat.
    \item Brace the box with the right arm while the left arm pushes the phone from the box lid into the box and arranges it flat.
    \item Continue bracing the box with the right arm while the left arm moves the lid clear, grasps it, and places it over the box at an angle.
    \item Keep the box braced with the right arm while the left arm pushes the lid into alignment and presses it down to close the box. All specified contents must be correctly placed and the box fully closed; otherwise, the trial is unsuccessful.
\end{enumerate}

\section{FERT Implementation Details}
\label{app:fert_details}
We initialize FERT from the selected task-specific (M)HG-DAgger checkpoint described in \Cref{app:policy_init}. Successful autonomous trajectories are stored in an in-memory rolling buffer, while failed trajectories provide temporary retrieval queries. Retrieval embeddings are cached using a frozen visual encoder, and sampling priorities are updated between policy updates. We optimize the policy with AdamW and cosine learning-rate decay.

For a flow-based policy, evaluating action log-likelihoods in \Cref{eq:fert_population_loss} requires tracking density changes along the generative flow. We instead retain the conditional flow-matching objective in \Cref{eq:flow_matching} as a tractable surrogate for likelihood-based fitting. Denoising objectives admit variational interpretations in terms of log-likelihood lower bounds under appropriate noise-level weighting. This motivates the use of flow matching to fit the prioritized successful experience without explicitly evaluating action likelihoods.

A short derivation follows from expressing velocity prediction as noise prediction. Given $A_i^\alpha=(1-\alpha)A_i + \alpha\epsilon_{\mathrm{noise}}$, define
\begin{equation}
\hat{\epsilon}_{\theta} = A_i^\alpha + (1-\alpha)v_\theta(A_i^\alpha,\alpha\mid o_i,g_i).
\end{equation}
Then
\begin{equation}
\left\|\hat{\epsilon}_{\theta}-\epsilon_{\mathrm{noise}}\right\|_F^2 = (1-\alpha)^2
\left\|
v_\theta(A_i^\alpha,\alpha\mid o_i,g_i) - (\epsilon_{\mathrm{noise}}-A_i)
\right\|_F^2.
\end{equation}
Thus, the flow-matching loss can be expressed as a noise-prediction loss with noise-level weighting. For monotone weightings, the corresponding weighted denoising objective has a variational form:
\begin{equation}
\begin{aligned}
C_i-\mathcal{L}_{\mathrm{den}}^{\zeta}(\theta;A_i,o_i,g_i)
&=
\mathbb{E}_{\alpha\sim p_\zeta,\,\epsilon_{\mathrm{noise}}}
\left[
\operatorname{ELBO}_{\theta,\alpha}
(A_i^\alpha\mid o_i,g_i)
\right] \\
&\leq
\mathbb{E}_{\alpha\sim p_\zeta,\,\epsilon_{\mathrm{noise}}}
\left[
\log p_{\theta,\alpha}
(A_i^\alpha\mid o_i,g_i)
\right].
\end{aligned}
\end{equation}
Here, $\zeta$ denotes the noise-level weighting, $p_\zeta$ is its induced distribution over noise levels, $p_{\theta,\alpha}$ is the corresponding diffusion model's marginal density, and $C_i$ is independent of $\theta$. Minimizing the weighted denoising loss therefore maximizes an expected log-likelihood lower bound for noise-perturbed actions. FERT applies this likelihood-related fitting principle through the base policy's flow-matching loss on priority-sampled successful chunks. For a detailed derivation, see Appendix A.1 of \citep{mcallister2026flow}.

\Cref{tab:fert-settings} summarizes the key training parameters of FERT. Each FERT run involves approximately 7 to 8 hours of online training. We use the checkpoint obtained at convergence for the final evaluation.

\captionsetup[longtable]{font=small,labelfont=bf,justification=raggedright,singlelinecheck=false,skip=7pt}
\begin{longtable}{@{}>{\raggedright\arraybackslash}p{.58\linewidth}>{\raggedright\arraybackslash}p{.36\linewidth}@{}}
\caption{Key FERT settings. Batch sizes are global.}\label{tab:fert-settings}\\
\toprule
\textbf{Parameter} & \textbf{Setting} \\
\midrule
\endfirsthead
\toprule
\textbf{Parameter} & \textbf{Setting} \\
\midrule
\endhead
\bottomrule
\endfoot
Max learning rate & $2.5\times10^{-5}$ \\
Global batch size & 1024 \\
Learning-rate schedule & Cosine \\
Learning-rate warmup & 80 optimizer steps \\
Learning-rate schedule horizon & 10,000 optimizer steps \\
Gradient norm clipping & 4.0 \\
Rolling buffer window per training worker & 350,000 frames \\
Minimum training data per worker & 60,000 frames \\
Failure-query length $L$ & 90 frames \\
Failure endpoint offset $\delta$ & 10 frames \\
Best retrieval segment number $K$ & 10 \\
Retrieval extensions $h_{\mathrm{pre}},h_{\mathrm{post}}$ & 30 frames \\
Priority retention factor $\rho$ & 0.88 \\
Initial sampling priority $w_{\mathrm{init}}$ & 0.5 \\
Down-sampling ratio & 3 \\
Weight Sync. Interval & 500 \\
Training worker to robot ratio & 0.25
\end{longtable}

\section{Baseline Details}
\label{app:baseline_details}
The following sections specify the evaluated RECAP and RLT implementations. Both start from the same selected expert-supervised checkpoint as FERT, but differ in their use of experience, human corrections and policy updates. These details define the comparison configurations used in \Cref{sec:exp-main}.

\subsection{RECAP}
\label{app:recap_baseline}

RECAP~\cite{intelligence2025pi06vlalearnsexperience} serves as our multi-robot offline reinforcement-learning baseline. We retain its value-based advantage labeling and advantage-conditioned policy training, with the adaptations below. This baseline evaluates policy improvement from collected fleet data with continued use of human corrective actions. Observation preprocessing and action representation follow \Cref{app:shared_observation_action}.

\paragraph{Initialization and training data.}
We initialize the policy from the same selected (M)HG-DAgger checkpoint used by FERT and perform one round of offline policy improvement. The training set combines data accumulated during (M)HG-DAgger with newly collected rollouts and human takeover data; the newly collected data include successful and failed trajectories. \Cref{tab:training-data} reports the data composition and frame-level intervention fractions. The original method performs multiple rounds of policy improvement, whereas our implementation uses a single round. The initial policy has undergone supervised training on demonstrations and corrections, without prior advantage-conditioned training.

\paragraph{Task-specific value function.}
We replace the original multitask, language-conditioned value model learned from broad data with a separate vision-and-state value network for each task. It uses a frozen SigLIP encoder and a lightweight Transformer head, and is trained on the collected task data without multimodal web data. We retain distributional value prediction with two-hot targets and undiscounted Monte Carlo returns. Nonterminal frames receive a reward of $-1$, successful termination receives $0$, and failed termination receives a task-specific negative penalty. Returns are divided by a scale no smaller than the maximum episode length plus the magnitude of the terminal failure reward, keeping the targets within the value support $[-1,0]$.

\paragraph{Advantage conditioning.}
We estimate advantages over the action-chunk horizon and label frames above the dataset's 70th advantage percentile as positive. Human-intervention frames are always labeled positive, overriding both the threshold and indicator dropout. We replace the original natural-language advantage indicators with dedicated positive, negative, and null tokens. Non-intervention samples use indicator dropout with probability 0.3. The policy is fine-tuned with the conditional flow-matching objective and evaluated with the positive indicator at guidance scale 1. \Cref{tab:shared-settings} summarizes the key settings.

\begingroup
\small
\setlength{\tabcolsep}{5pt}
\renewcommand{\arraystretch}{1.16}
\setlength{\LTleft}{0pt plus 1fill}
\setlength{\LTright}{0pt plus 1fill}
\setlength{\LTcapwidth}{\linewidth}
\captionsetup[longtable]{font=small,labelfont=bf,justification=raggedright,singlelinecheck=false,skip=7pt}

\begin{longtable}{@{}
  >{\raggedright\arraybackslash}p{\dimexpr.23\linewidth-2\tabcolsep\relax}
  >{\raggedright\arraybackslash}p{\dimexpr.25\linewidth-2\tabcolsep\relax}
  *{3}{>{\raggedleft\arraybackslash}p{\dimexpr.115\linewidth-2\tabcolsep\relax}}
  >{\raggedleft\arraybackslash}p{.175\linewidth}@{}}
\caption{RECAP training data sources and composition. Intervention fractions are measured over frames.}\label{tab:training-data}\\
\toprule
\textbf{Task} & \textbf{Data source} & \multicolumn{3}{c}{\textbf{Trajectories}} & \textbf{Intervention} \\
\cmidrule(lr){3-5}
 & & Total & Successful & Failed & \textbf{fraction} \\
\midrule
\endfirsthead
\multicolumn{6}{@{}l}{\footnotesize\itshape Table \thetable\ continued}\\[4pt]
\toprule
\textbf{Task} & \textbf{Data source} & \multicolumn{3}{c}{\textbf{Trajectories}} & \textbf{Intervention} \\
\cmidrule(lr){3-5}
 & & Total & Successful & Failed & \textbf{fraction} \\
\midrule
\endhead
\midrule
\multicolumn{6}{r@{}}{\footnotesize\itshape Continued on next page}\\
\endfoot
\bottomrule

\endlastfoot
\textbf{Flower arrangement} & (M)HG-DAgger & 160 & 160 & 0 & 74.6\% \\*
 & Newly collected data & 48 & 13 & 35 & 2.4\% \\
\addlinespace[5pt]
\textbf{Clothes folding} & (M)HG-DAgger & 148 & 148 & 0 & 68.0\% \\*
 & Newly collected data & 118 & 39 & 79 & 18.0\% \\
\addlinespace[5pt]
\textbf{Phone charging} & (M)HG-DAgger & 716 & 716 & 0 & 39.9\% \\*
 & Newly collected data & 358 & 211 & 147 & 7.4\% \\
\addlinespace[5pt]
\textbf{Stool assembly} & (M)HG-DAgger & 328 & 328 & 0 & 55.8\% \\*
 & Newly collected data & 559 & 347 & 212 & 21.0\% \\
\addlinespace[5pt]
\textbf{Phone packaging} & (M)HG-DAgger & 978 & 978 & 0 & 58.4\% \\*
 & Newly collected data & 458 & 371 & 87 & 4.5\% \\
\end{longtable}

\begin{longtable}{@{}>{\raggedright\arraybackslash}p{.58\linewidth}>{\raggedright\arraybackslash}p{.36\linewidth}@{}}
\caption{Key RECAP settings. Batch sizes are global.}\label{tab:shared-settings}\\
\toprule
\textbf{Parameter} & \textbf{Setting} \\
\midrule
\endfirsthead
\toprule
\textbf{Parameter} & \textbf{Setting} \\
\midrule
\endhead
\bottomrule
\endfoot
Value peak learning rate & $10^{-4}$ \\
Value global batch size & 2,048 \\
Policy peak learning rate & $5\times10^{-5}$ \\
Policy global batch size & 256 \\
Number of value bins & 201 \\
Value support & $[-1,0]$ \\
Advantage lookahead & 30 frames \\
Advantage threshold & 70th percentile \\
Indicator dropout probability & 0.3 \\
\end{longtable}
\endgroup

\subsection{RLT}
\label{app:rlt_baseline}

RLT~\cite{xu2026rltokenbootstrappingonline} serves as our single-robot online reinforcement-learning baseline. It provides a comparison with online policy improvement on an individual robot using task rewards and human corrections. We retain its two-stage procedure: learning an RL-token representation from frozen VLA features, followed by online actor--critic training. Observation preprocessing and action representation follow \Cref{app:shared_observation_action}.

\paragraph{Initialization and representation learning.}
We freeze the same selected (M)HG-DAgger $\pi_0$ checkpoint used to initialize FERT. An encoder--decoder is trained on the SFT dataset to compress the VLA's image-prefix features into an RL token through feature reconstruction. After this stage, the decoder is discarded and the encoder remains frozen throughout online learning.

\paragraph{Residual action generation.}
The online policy uses the RL token, proprioceptive state, and VLA reference actions. Our actor predicts a residual around the VLA reference, whereas the original actor directly predicts an action mean conditioned on that reference. The residual output layer is initialized to zero, and a soft trust-region mapping bounds deviations from the reference actions in normalized action coordinates. Gripper deviations are disabled. \Cref{tab:rlt-common-settings} summarizes the key settings, with task-specific overrides in \Cref{tab:rlt-task-overrides}. For Phone Charging, residual corrections are enabled only for the right arm; the left arm and grippers follow the VLA reference.

\paragraph{Online data and human supervision.}
Experience is collected on one robot using a sparse reward of 1 at successful termination and 0 otherwise. During training-data collection, the operator switches between VLA execution and RLT execution; only RLT-phase transitions enter the replay buffer. Human corrections replace the executed actions and behavior-cloning targets at intervention frames, and intervention transitions are also stored in a demonstration buffer. Training batches draw equally from replay and demonstrations when both are available. Operators provide success/failure labels and temporarily intervene during difficult task segments for 5--10 trajectories, then return to outcome labeling, repeating this process as needed. \Cref{tab:rlt-task-specific-settings} reports task-specific episode counts and intervention fractions.

\paragraph{Policy optimization and evaluation.}
We use twin critics, delayed actor updates, and RLT's chunk-level backup, together with behavior-cloning regularization. The VLA and RL-token encoder remain frozen during online learning. At evaluation time, switching between VLA and RLT execution is automatic, without operator phase selection or corrective actions. The actor uses its mean without exploration noise, followed by the same trust-region mapping used during training.

\begingroup
\small
\setlength{\tabcolsep}{5pt}
\renewcommand{\arraystretch}{1.16}
\setlength{\LTleft}{0pt plus 1fill}
\setlength{\LTright}{0pt plus 1fill}
\setlength{\LTcapwidth}{\linewidth}
\captionsetup[longtable]{font=small,labelfont=bf,justification=raggedright,singlelinecheck=false,skip=7pt}
\begin{longtable}{@{}>{\raggedright\arraybackslash}p{.58\linewidth}>{\raggedright\arraybackslash}p{.36\linewidth}@{}}
\caption{Key RLT settings shared across tasks, except for the overrides in \Cref{tab:rlt-task-overrides}. Deviation budgets are in normalized action coordinates.}\label{tab:rlt-common-settings}\\
\toprule
\textbf{Parameter} & \textbf{Setting} \\
\midrule
\endfirsthead
\toprule
\textbf{Parameter} & \textbf{Setting} \\
\midrule
\endhead
\bottomrule
\endfoot
Actor learning rate & $3\times10^{-4}$ \\
Critic learning rate & $3\times10^{-4}$ \\
Global batch size & 256 \\
Discount factor & 0.99 \\
Update-to-data ratio & 5 \\
Critic-to-actor update ratio & $2:1$ \\
Behavior-cloning weight & 2 \\
Actions executed per VLA-phase chunk & 30 \\
Actions executed per RLT-phase chunk & 10 \\
Default arm deviation budget & 0.2 \\
Gripper deviation budget & 0 \\
Default exploration standard deviation & 0.01 \\
\end{longtable}

\begin{longtable}{@{}>{\raggedright\arraybackslash}p{\dimexpr.25\linewidth-\tabcolsep\relax}>{\raggedright\arraybackslash}p{\dimexpr.27\linewidth-2\tabcolsep\relax}>{\raggedright\arraybackslash}p{\dimexpr.48\linewidth-\tabcolsep\relax}@{}}
\caption{Task-specific RLT settings. Unlisted settings follow \Cref{tab:rlt-common-settings}. Deviation budgets are in normalized action coordinates.}\label{tab:rlt-task-overrides}\\
\toprule
\textbf{Task} & \textbf{Parameter} & \textbf{Setting} \\
\midrule
\endfirsthead
\toprule
\textbf{Task} & \textbf{Parameter} & \textbf{Setting} \\
\midrule
\endhead
\bottomrule
\endfoot
\textbf{Phone Charging} & Deviation budget & 0.2 for the 7 right-arm dimensions; 0 for the left arm and grippers \\
\addlinespace[5pt]
\textbf{Flower Arrangement} & Deviation budget & 0.5 for the 14 arm dimensions; 0 for the grippers \\*
 & Exploration standard deviation & 0.005 \\
\addlinespace[5pt]
\textbf{Clothes Folding} & Deviation budget & 0.5 for the 14 arm dimensions; 0 for the grippers \\*
 & Exploration standard deviation & 0.005 \\
\end{longtable}

\begin{longtable}{@{}
  >{\raggedright\arraybackslash}p{\dimexpr.4\linewidth-2\tabcolsep\relax}
  >{\raggedleft\arraybackslash}p{\dimexpr.3\linewidth-2\tabcolsep\relax}
  >{\raggedleft\arraybackslash}p{\dimexpr.3\linewidth-2\tabcolsep\relax}@{}}
\caption{Online RLT intervention fractions and complete-episode counts by task. All values are approximate.}\label{tab:rlt-task-specific-settings}\\
\toprule
\textbf{Task} & \textbf{Intervention fraction} & \textbf{Complete episodes} \\
\midrule
\endfirsthead
\multicolumn{3}{@{}l}{\footnotesize\itshape Table \thetable\ continued}\\[4pt]
\toprule
\textbf{Task} & \textbf{Intervention fraction} & \textbf{Complete episodes} \\
\midrule
\endhead
\midrule
\multicolumn{3}{r@{}}{\footnotesize\itshape Continued on next page}\\
\endfoot
\bottomrule
\endlastfoot
Flower Arrangement & 30\% & 30 \\
\addlinespace[5pt]
Clothes Folding & 20\% & 30 \\
\addlinespace[5pt]
Phone Packaging & 40\% & 40 \\
\addlinespace[5pt]
Stool Assembly & 80\% & 100 \\
\addlinespace[5pt]
Phone Charging & 40\% & 70 \\
\end{longtable}
\endgroup

\section{Extended Experimental Results}

This appendix reports the task-level measurements, system traces and retrieval examples supporting the evaluation of FERT and FATS. The policy results measure improvement on held-out robots, the traces characterize experience collection during a training run, and the retrieval examples illustrate how failed executions select successful experience.

\subsection{Detailed Success Rate}
\Cref{tab:detail-succ} reports the detailed evaluation results for SFT, (M)HG-DAgger, RLT, RECAP and FERT across the five tasks and evaluated training fleet sizes.
\begin{table}[!htbp]
\centering
\small
\setlength{\tabcolsep}{1.5pt}
\renewcommand{\arraystretch}{1.18}
\setlength{\heavyrulewidth}{0.7pt}
\setlength{\lightrulewidth}{0.4pt}
\setlength{\cmidrulewidth}{0.3pt}
\setlength{\arrayrulewidth}{0.3pt}
\makeatletter
\def\rlinfHeaderRuleTop{\rule[-\dp\@arstrutbox]{\arrayrulewidth}{\dimexpr(\ht\@arstrutbox+\dp\@arstrutbox)*4/5\relax}}
\def\rlinfHeaderRuleBottom{\rule[\dimexpr(\ht\@arstrutbox+\dp\@arstrutbox)/5-\dp\@arstrutbox\relax]{\arrayrulewidth}{\dimexpr(\ht\@arstrutbox+\dp\@arstrutbox)*4/5\relax}}
\def\rlinfCline#1{\rlinf@cline#1\@nil}
\def\rlinf@cline#1-#2\@nil{\omit\@multicnt#1\advance\@multispan\m@ne\ifnum\@multicnt=\@ne\@firstofone{&\omit}\fi\@multicnt#2\advance\@multicnt-#1\advance\@multispan\@ne\hskip0pt plus0.5fill\leaders\hrule\@height\arrayrulewidth\hskip0pt plus9fill\hskip0pt plus0.5fill\cr\noalign{\vskip-\arrayrulewidth}}
\makeatother
\captionsetup{font=small,labelfont=bf,justification=raggedright,singlelinecheck=false,skip=8pt}
\caption{Task success rates for SFT, (M)HG-DAgger, RLT, RECAP, and FERT. Results are reported for two held-out test robots, with 60 trials per robot for each evaluated checkpoint. Robot 1 and Robot 2 denote the held-out test robots; their rows report successful trials out of 60. Avg. (\%) pools successes over 120 trials. Bold values indicate the highest average success rate for each task.}
\label{tab:detail-succ}
\captionsetup[subtable]{font=small,labelfont=bf,textfont=bf,justification=raggedright,singlelinecheck=false,skip=5pt}
\begin{subtable}{\linewidth}
\centering
\caption{Flower Arrangement}
\label{tab:exp-flower-horizontal}
\begin{tabularx}{\linewidth}{>{\raggedright\arraybackslash}p{1.5cm}|>{\centering\arraybackslash}X|>{\centering\arraybackslash}X>{\centering\arraybackslash}X>{\centering\arraybackslash}X|>{\centering\arraybackslash}X>{\centering\arraybackslash}X>{\centering\arraybackslash}X>{\centering\arraybackslash}X>{\centering\arraybackslash}X}
\toprule
 & \multirow{2}{=}{\centering\adjustbox{max width=\linewidth}{\footnotesize\bfseries SFT}} & \multicolumn{3}{c|}{\textbf{(M)HG-DAgger} {\scriptsize\textcolor{black!60}{(num of robots)}}} & \multicolumn{1}{c!{\rlinfHeaderRuleTop}}{\multirow{2}{*}{\footnotesize\bfseries RLT}} & \multirow{2}{=}{\centering\adjustbox{max width=\linewidth}{\footnotesize\bfseries RECAP}} & \multicolumn{3}{!{\rlinfHeaderRuleTop}c}{\textbf{FERT} {\scriptsize\textcolor{black!60}{(num of robots)}}} \\
\rlinfCline{3-5}\rlinfCline{8-10}
 &  & 1 & 4 & 8 & \multicolumn{1}{c!{\rlinfHeaderRuleBottom}}{} &  & \multicolumn{1}{!{\rlinfHeaderRuleBottom}c}{4} & 8 & 16 \\
\midrule
Robot 1 & 23/60 & 47/60 & 53/60 & 60/60 & 56/60 & 58/60 & 59/60 & 60/60 & 60/60 \\
Robot 2 & 20/60 & 56/60 & 56/60 & 58/60 & 58/60 & 56/60 & 57/60 & 59/60 & 59/60 \\
\rowcolor{black!5}
Avg. (\%) & 35.83 & 85.83 & 90.83 & 98.33 & 95.00 & 95.00 & 96.67 & \textbf{99.17} & \textbf{99.17} \\
\bottomrule
\end{tabularx}
\end{subtable}
\par\vspace{10pt}
\begin{subtable}{\linewidth}
\centering
\caption{Clothes Folding}
\label{tab:exp-clothes-horizontal}
\begin{tabularx}{\linewidth}{>{\raggedright\arraybackslash}p{1.5cm}|>{\centering\arraybackslash}X|>{\centering\arraybackslash}X>{\centering\arraybackslash}X>{\centering\arraybackslash}X|>{\centering\arraybackslash}X>{\centering\arraybackslash}X>{\centering\arraybackslash}X>{\centering\arraybackslash}X>{\centering\arraybackslash}X}
\toprule
 & \multirow{2}{=}{\centering\adjustbox{max width=\linewidth}{\footnotesize\bfseries SFT}} & \multicolumn{3}{c|}{\textbf{(M)HG-DAgger} {\scriptsize\textcolor{black!60}{(num of robots)}}} & \multicolumn{1}{c!{\rlinfHeaderRuleTop}}{\multirow{2}{*}{\footnotesize\bfseries RLT}} & \multirow{2}{=}{\centering\adjustbox{max width=\linewidth}{\footnotesize\bfseries RECAP}} & \multicolumn{3}{!{\rlinfHeaderRuleTop}c}{\textbf{FERT} {\scriptsize\textcolor{black!60}{(num of robots)}}} \\
\rlinfCline{3-5}\rlinfCline{8-10}
 &  & 1 & 4 & 8 & \multicolumn{1}{c!{\rlinfHeaderRuleBottom}}{} &  & \multicolumn{1}{!{\rlinfHeaderRuleBottom}c}{4} & 8 & 16 \\
\midrule
Robot 1 & 44/60 & 48/60 & 50/60 & 51/60 & 50/60 & 54/60 & 49/60 & 56/60 & 58/60 \\
Robot 2 & 34/60 & 44/60 & 47/60 & 52/60 & 56/60 & 56/60 & 51/60 & 57/60 & 60/60 \\
\rowcolor{black!5}
Avg. (\%) & 65.00 & 76.67 & 80.83 & 85.83 & 88.33 & 91.67 & 83.33 & 94.17 & \textbf{98.33} \\
\bottomrule
\end{tabularx}
\end{subtable}
\par\vspace{10pt}
\begin{subtable}{\linewidth}
\centering
\caption{Phone Charging}
\label{tab:exp-charge-horizontal}
\begin{tabularx}{\linewidth}{>{\raggedright\arraybackslash}p{1.5cm}|>{\centering\arraybackslash}X|>{\centering\arraybackslash}X>{\centering\arraybackslash}X>{\centering\arraybackslash}X>{\centering\arraybackslash}X|>{\centering\arraybackslash}X>{\centering\arraybackslash}X>{\centering\arraybackslash}X>{\centering\arraybackslash}X>{\centering\arraybackslash}X>{\centering\arraybackslash}X}
\toprule
 & \multirow{2}{=}{\centering\adjustbox{max width=\linewidth}{\footnotesize\bfseries SFT}} & \multicolumn{4}{c|}{\textbf{(M)HG-DAgger} {\scriptsize\textcolor{black!60}{(num of robots)}}} & \multicolumn{1}{c!{\rlinfHeaderRuleTop}}{\multirow{2}{*}{\footnotesize\bfseries RLT}} & \multirow{2}{=}{\centering\adjustbox{max width=\linewidth}{\footnotesize\bfseries RECAP}} & \multicolumn{4}{!{\rlinfHeaderRuleTop}c}{\textbf{FERT} {\scriptsize\textcolor{black!60}{(num of robots)}}} \\
\rlinfCline{3-6}\rlinfCline{9-12}
 &  & 1 & 4 & 8 & {\scriptsize Mixed} & \multicolumn{1}{c!{\rlinfHeaderRuleBottom}}{} &  & \multicolumn{1}{!{\rlinfHeaderRuleBottom}c}{2} & 4 & 8 & 16 \\
\midrule
Robot 1 & 7/60 & 1/60 & 39/60 & 43/60 & 41/60 & 46/60 & 45/60 & 20/60 & 52/60 & 58/60 & 58/60 \\
Robot 2 & 12/60 & 0/60 & 32/60 & 28/60 & 31/60 & 45/60 & 54/60 & 28/60 & 56/60 & 56/60 & 59/60 \\
\rowcolor{black!5}
Avg. (\%) & 15.83 & 0.83 & 59.17 & 59.17 & 60.00 & 75.83 & 82.50 & 40.00 & 90.00 & 95.00 & \textbf{97.50} \\
\bottomrule
\end{tabularx}
\end{subtable}
\par\vspace{10pt}
\begin{subtable}{\linewidth}
\centering
\caption{Stool Assembly}
\label{tab:exp-stool-horizontal}
\begin{tabularx}{\linewidth}{>{\raggedright\arraybackslash}p{1.5cm}|>{\centering\arraybackslash}X|>{\centering\arraybackslash}X>{\centering\arraybackslash}X>{\centering\arraybackslash}X>{\centering\arraybackslash}X>{\centering\arraybackslash}X|>{\centering\arraybackslash}X>{\centering\arraybackslash}X>{\centering\arraybackslash}X>{\centering\arraybackslash}X>{\centering\arraybackslash}X}
\toprule
 & \multirow{2}{=}{\centering\adjustbox{max width=\linewidth}{\footnotesize\bfseries SFT}} & \multicolumn{5}{c|}{\textbf{(M)HG-DAgger} {\scriptsize\textcolor{black!60}{(num of robots)}}} & \multicolumn{1}{c!{\rlinfHeaderRuleTop}}{\multirow{2}{*}{\footnotesize\bfseries RLT}} & \multirow{2}{=}{\centering\adjustbox{max width=\linewidth}{\footnotesize\bfseries RECAP}} & \multicolumn{3}{!{\rlinfHeaderRuleTop}c}{\textbf{FERT} {\scriptsize\textcolor{black!60}{(num of robots)}}} \\
\rlinfCline{3-7}\rlinfCline{10-12}
 &  & 1 & 4 & 8 & 16 & {\scriptsize Mixed} & \multicolumn{1}{c!{\rlinfHeaderRuleBottom}}{} &  & \multicolumn{1}{!{\rlinfHeaderRuleBottom}c}{8} & 16 & 32 \\
\midrule
Robot 1 & 41/60 & 2/60 & 50/60 & 52/60 & 39/60 & 50/60 & 48/60 & 53/60 & 41/60 & 57/60 & 57/60 \\
Robot 2 & 35/60 & 0/60 & 46/60 & 44/60 & 37/60 & 48/60 & 49/60 & 49/60 & 49/60 & 54/60 & 55/60 \\
\rowcolor{black!5}
Avg. (\%) & 63.33 & 1.67 & 80.00 & 80.00 & 63.33 & 81.67 & 80.83 & 85.00 & 75.00 & 92.50 & \textbf{93.33} \\
\bottomrule
\end{tabularx}
\end{subtable}
\par\vspace{10pt}
\begin{subtable}{\linewidth}
\centering
\caption{Phone Packaging}
\label{tab:packaging-horizontal}
\begin{tabularx}{\linewidth}{>{\raggedright\arraybackslash}p{1.5cm}|>{\centering\arraybackslash}X|>{\centering\arraybackslash}X>{\centering\arraybackslash}X>{\centering\arraybackslash}X>{\centering\arraybackslash}X>{\centering\arraybackslash}X|>{\centering\arraybackslash}X>{\centering\arraybackslash}X>{\centering\arraybackslash}X>{\centering\arraybackslash}X>{\centering\arraybackslash}X>{\centering\arraybackslash}X>{\centering\arraybackslash}X>{\centering\arraybackslash}X}
\toprule
 & \multirow{2}{=}{\centering\adjustbox{max width=\linewidth}{\footnotesize\bfseries SFT}} & \multicolumn{5}{c|}{\textbf{(M)HG-DAgger} {\scriptsize\textcolor{black!60}{(num of robots)}}} & \multicolumn{1}{c!{\rlinfHeaderRuleTop}}{\multirow{2}{*}{\footnotesize\bfseries RLT}} & \multirow{2}{=}{\centering\adjustbox{max width=\linewidth}{\footnotesize\bfseries RECAP}} & \multicolumn{6}{!{\rlinfHeaderRuleTop}c}{\textbf{FERT} {\scriptsize\textcolor{black!60}{(num of robots)}}} \\
\rlinfCline{3-7}\rlinfCline{10-15}
 &  & 1 & 4 & 8 & 16 & {\scriptsize Mixed} & \multicolumn{1}{c!{\rlinfHeaderRuleBottom}}{} &  & \multicolumn{1}{!{\rlinfHeaderRuleBottom}c}{4} & 8 & 16 & 32 & 64 & 96 \\
\midrule
Robot 1 & 0/60 & 0/60 & 14/60 & 37/60 & 42/60 & 47/60 & 54/60 & 47/60 & 32/60 & 45/60 & 53/60 & 56/60 & 57/60 & 59/60 \\
Robot 2 & 2/60 & 2/60 & 25/60 & 41/60 & 39/60 & 48/60 & 48/60 & 43/60 & 29/60 & 40/60 & 51/60 & 55/60 & 57/60 & 59/60 \\
\rowcolor{black!5}
Avg. (\%) & 1.67 & 1.67 & 32.50 & 65.00 & 67.50 & 79.17 & 85.00 & 75.00 & 50.83 & 70.83 & 86.67 & 92.50 & 95.00 & \textbf{98.33} \\
\bottomrule
\end{tabularx}
\end{subtable}
\end{table}

\Cref{tab:detail-abl} reports the detailed results for the FERT ablation configurations on Phone Packaging with 16 training robots.

\begin{table}[!htbp]
\centering
\small
\setlength{\tabcolsep}{1.5pt}
\renewcommand{\arraystretch}{1.18}
\setlength{\heavyrulewidth}{0.7pt}
\setlength{\lightrulewidth}{0.4pt}
\setlength{\cmidrulewidth}{0.3pt}
\setlength{\arrayrulewidth}{0.3pt}
\makeatletter
\def\rlinfHeaderRuleTop{\rule[-\dp\@arstrutbox]{\arrayrulewidth}{\dimexpr(\ht\@arstrutbox+\dp\@arstrutbox)*4/5\relax}}
\def\rlinfHeaderRuleBottom{\rule[\dimexpr(\ht\@arstrutbox+\dp\@arstrutbox)/5-\dp\@arstrutbox\relax]{\arrayrulewidth}{\dimexpr(\ht\@arstrutbox+\dp\@arstrutbox)*4/5\relax}}
\def\rlinfCline#1{\rlinf@cline#1\@nil}
\def\rlinf@cline#1-#2\@nil{\omit\@multicnt#1\advance\@multispan\m@ne\ifnum\@multicnt=\@ne\@firstofone{&\omit}\fi\@multicnt#2\advance\@multicnt-#1\advance\@multispan\@ne\hskip0pt plus0.5fill\leaders\hrule\@height\arrayrulewidth\hskip0pt plus9fill\hskip0pt plus0.5fill\cr\noalign{\vskip-\arrayrulewidth}}
\makeatother
\captionsetup{font=small,labelfont=bf,justification=raggedright,singlelinecheck=false,skip=8pt}
\caption{Success rates for the FERT ablation study on Phone Packaging with 16 training robots. FERT denotes the full configuration; w/o FR disables failure-guided retrieval; No extension sets $h_{\mathrm{pre}}=h_{\mathrm{post}}=0$. Retention $\rho$ is the priority retention factor, and Sync interval is the policy synchronization interval in optimizer steps. Each configuration is evaluated over 60 trials on each of two held-out test robots. The robot rows report successful trials out of 60, and Avg. (\%) pools successes over 120 trials. Bold indicates the highest average success rate.}
\label{tab:detail-abl}
\begin{tabularx}{\linewidth}{>{\raggedright\arraybackslash}p{1.5cm}>{\centering\arraybackslash}X>{\centering\arraybackslash}X>{\centering\arraybackslash}X>{\centering\arraybackslash}X>{\centering\arraybackslash}X>{\centering\arraybackslash}X>{\centering\arraybackslash}X}
\toprule
 & \multirow{2}{*}{\shortstack{FERT\\(16 robots)}} & \multirow{2}{*}{w/o FR} & \multirow{2}{*}{\shortstack{No\\extension}} & \multicolumn{2}{c}{Retention $\rho$} & \multicolumn{2}{c}{Sync interval} \\
\cmidrule(lr){5-6}\cmidrule(lr){7-8}
 & & & & 0.8 & 0.95 & 100 & 1500 \\
\midrule
Robot 1 & 53/60 & 33/60 & 34/60 & 44/60 & 50/60 & 46/60 & 48/60 \\
Robot 2 & 51/60 & 44/60 & 35/60 & 41/60 & 49/60 & 44/60 & 48/60 \\
\rowcolor{black!5}
Avg. (\%) & \textbf{86.67} & 64.17 & 57.50 & 70.83 & 82.50 & 75.00 & 80.00 \\
\bottomrule
\end{tabularx}
\end{table}

\Cref{tab:detail-human} and \Cref{fig:exp-detail-human} report human teleoperation performance. Performance varies substantially across operators on some tasks. For example, on Phone Packaging, individual success rates range from 25\% (5/20) to 95\% (19/20), with an overall average of 75\%.

\begin{table}[!htbp]
\centering
\small
\setlength{\tabcolsep}{1.5pt}
\renewcommand{\arraystretch}{1.18}
\setlength{\heavyrulewidth}{0.7pt}
\setlength{\lightrulewidth}{0.4pt}
\setlength{\cmidrulewidth}{0.3pt}
\setlength{\arrayrulewidth}{0.3pt}
\makeatletter
\def\rlinfHeaderRuleTop{\rule[-\dp\@arstrutbox]{\arrayrulewidth}{\dimexpr(\ht\@arstrutbox+\dp\@arstrutbox)*4/5\relax}}
\def\rlinfHeaderRuleBottom{\rule[\dimexpr(\ht\@arstrutbox+\dp\@arstrutbox)/5-\dp\@arstrutbox\relax]{\arrayrulewidth}{\dimexpr(\ht\@arstrutbox+\dp\@arstrutbox)*4/5\relax}}
\def\rlinfCline#1{\rlinf@cline#1\@nil}
\def\rlinf@cline#1-#2\@nil{\omit\@multicnt#1\advance\@multispan\m@ne\ifnum\@multicnt=\@ne\@firstofone{&\omit}\fi\@multicnt#2\advance\@multicnt-#1\advance\@multispan\@ne\hskip0pt plus0.5fill\leaders\hrule\@height\arrayrulewidth\hskip0pt plus9fill\hskip0pt plus0.5fill\cr\noalign{\vskip-\arrayrulewidth}}
\makeatother
\captionsetup{font=small,labelfont=bf,justification=raggedright,singlelinecheck=false,skip=8pt}
\caption{Human teleoperation success rates across five tasks. Each of the six operators completes 20 trials per task, yielding 120 trials per task in total. ID denotes the operator; each entry reports successful trials out of 20. Avg. (\%) pools successes over all 120 trials per task.}
\label{tab:detail-human}
\begin{tabularx}{\linewidth}{>{\raggedright\arraybackslash}p{1.5cm}>{\centering\arraybackslash}X>{\centering\arraybackslash}X>{\centering\arraybackslash}X>{\centering\arraybackslash}X>{\centering\arraybackslash}X}
\toprule
\textbf{ID} & \textbf{\shortstack{Flower\\Arrangement}} & \textbf{\shortstack{Clothes\\Folding}} & \textbf{\shortstack{Phone\\Charging}} & \textbf{\shortstack{Stool\\Assembly}} & \textbf{\shortstack{Phone\\Packaging}} \\
\midrule
1 & 20/20 & 20/20 & 12/20 & 20/20 & 16/20 \\
2 & 19/20 & 16/20 & 18/20 & 20/20 & 13/20 \\
3 & 20/20 & 19/20 & 12/20 & 17/20 & 19/20 \\
4 & 12/20 & 15/20 & 12/20 & 15/20 & 19/20 \\
5 & 18/20 & 14/20 & 19/20 & 14/20 & 18/20 \\
6 & 20/20 & 15/20 & 13/20 & 14/20 & 5/20 \\
\rowcolor{black!5}
Avg. (\%) & 90.83 & 82.50 & 71.67 & 83.33 & 75.00 \\
\bottomrule
\end{tabularx}
\end{table}

\begin{figure}[h]
    \centering
    \includegraphics[width=\linewidth]{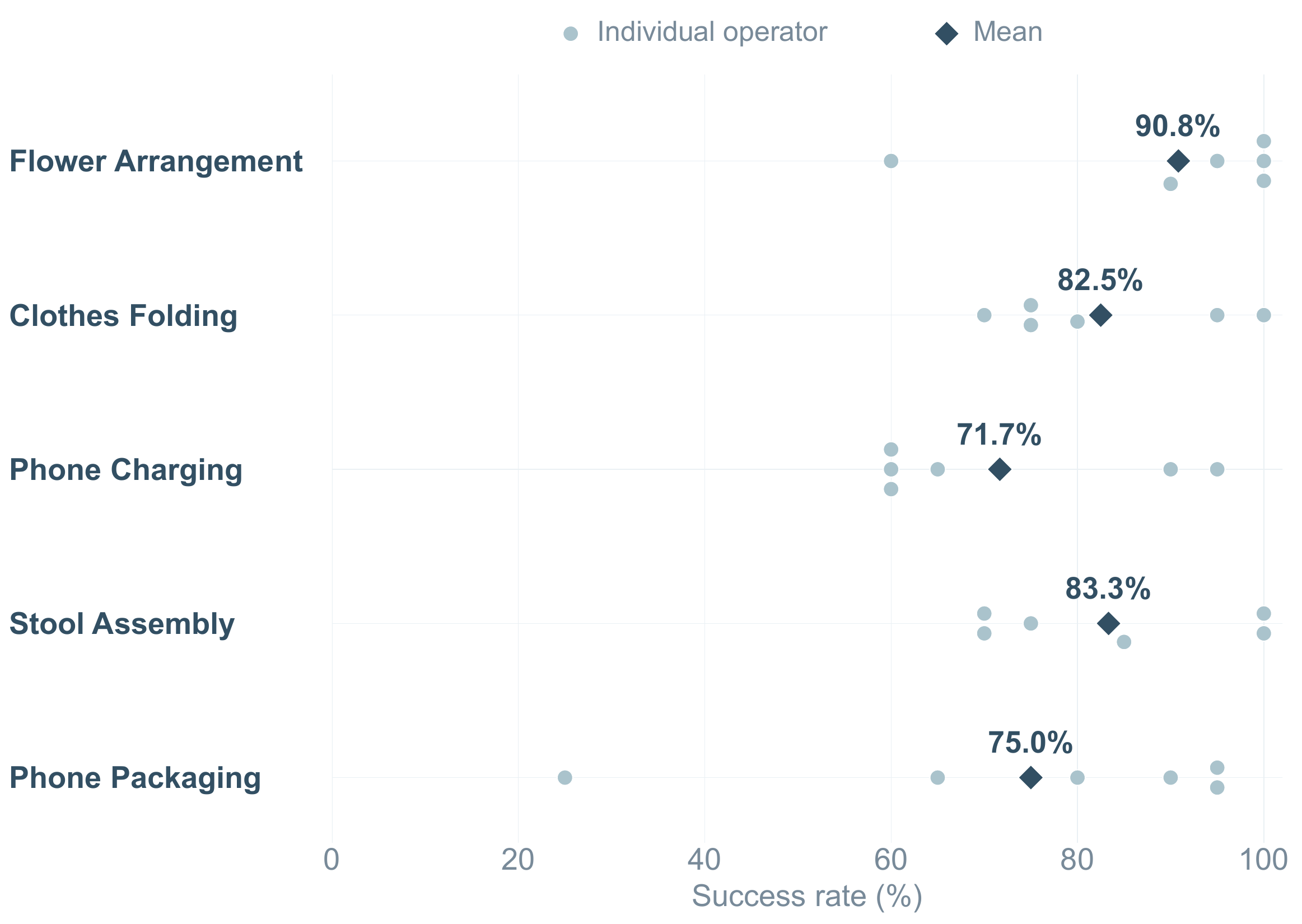}
    \caption{Human teleoperation success rates across tasks and operators. Each operator is evaluated over 20 trials per task.}
    \label{fig:exp-detail-human}
\end{figure}

\subsection{Detailed System Traces}
\label{app:infra_traces}

This section provides the detailed traces underlying the system evaluation in \Cref{sec:exp-system}. They document robot participation, local interruptions and recovery of experience collection, and supply the episode and incident intervals used for the trace-replay comparisons.
\Cref{fig:episode-record} shows the run's 9,834 episodes for individual robots, distinguishing successful trials, failed trials and system errors. Blank intervals include environment resets and operator breaks. Outcomes vary across robots: Robots 3 and 4 succeed in most of their episodes, whereas Robots 6 and 7 fail in most of theirs.

\Cref{fig:exp-active} summarizes the number of robots executing actions over time. The light curve shows the raw count and the dark curve applies Gaussian smoothing with a standard deviation of 60 seconds. Outside the lunch break, 65 robots were executing at a time on average.

\Cref{fig:error} gives the timing, affected robot identities and types of the 30 fault incidents on 22 robots. These detailed intervals support the fault-event grouping and recovery statistics reported in the main text.

\begin{figure}
    \centering
    \includegraphics[width=0.95\linewidth]{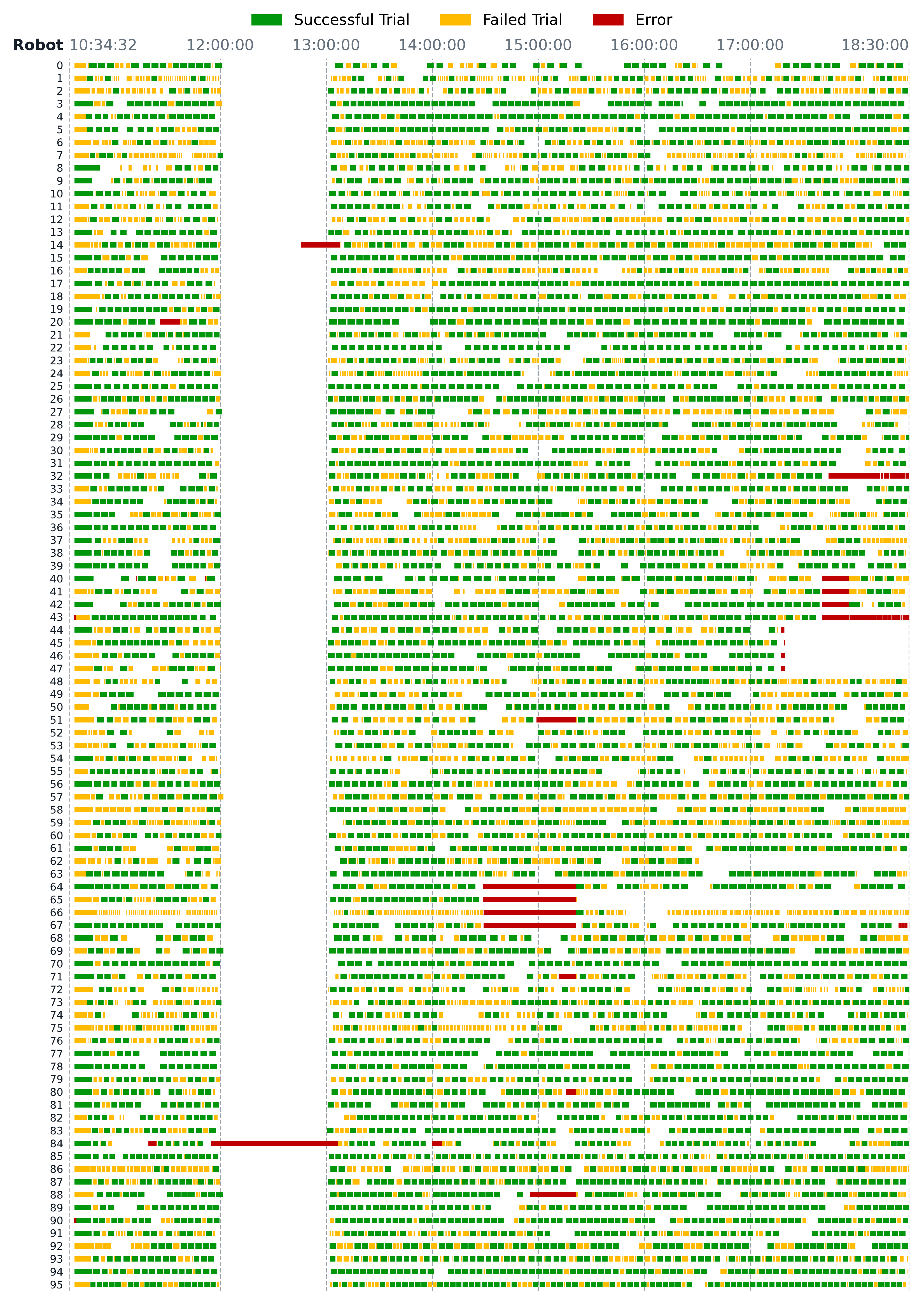}
    \caption{Detailed per-robot trial timelines during a 96-robot training run. Successful trials, failed trials and system errors are marked separately; blank intervals include environment resets and operator breaks.}
    \label{fig:episode-record}
\end{figure}

\begin{figure}
    \centering
    \includegraphics[width=\linewidth]{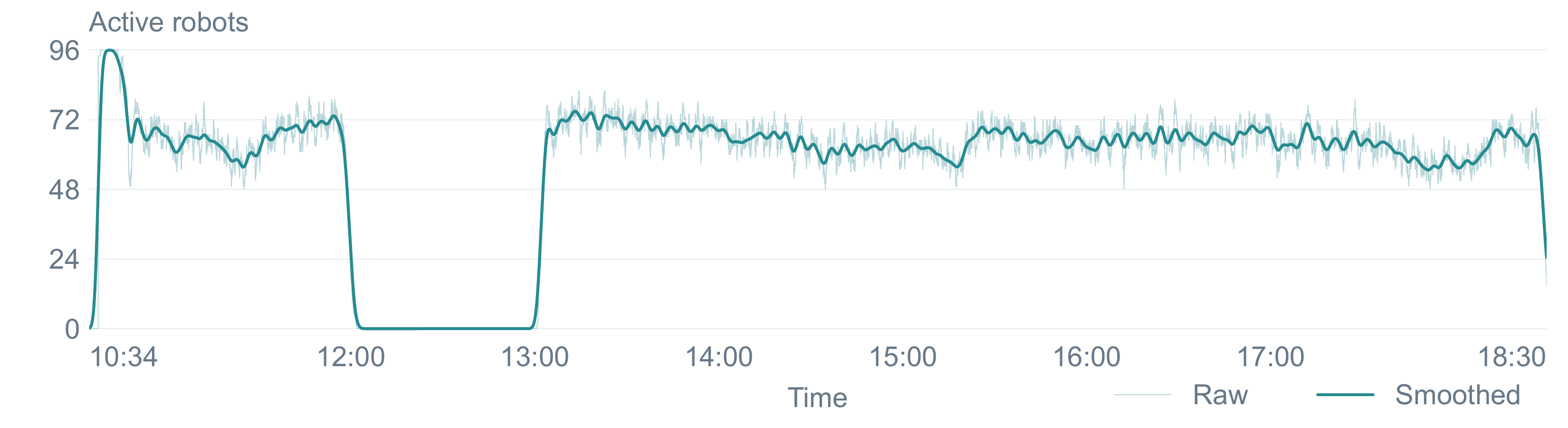}
    \caption{Active robot count during a 96-robot training run. A robot is active while executing actions; inactive intervals include environment resets, operator breaks and robot errors. Light curves show raw measurements, and dark curves show Gaussian-smoothed measurements ($\sigma = 60$~s).}
    \label{fig:exp-active}
\end{figure}

\begin{figure}
    \centering
    \includegraphics[width=\linewidth]{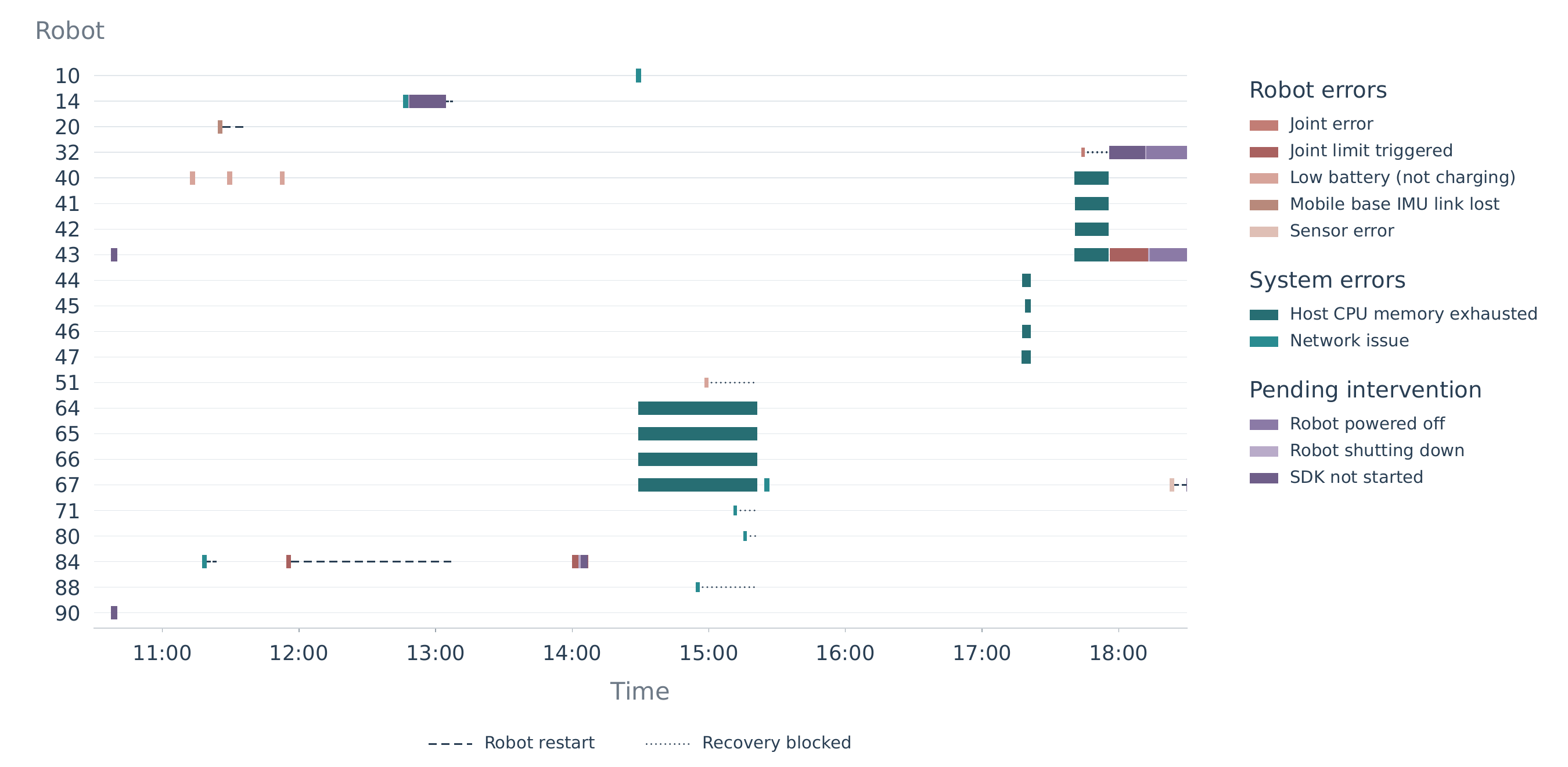}
    \caption{Error Types and Occurrence Times During a 96-Robot Training Run. The x-axis indicates time, and the y-axis lists the IDs of robots that encountered at least one error. Colors distinguish different error types.}
    \label{fig:error}
\end{figure}

\subsection{Visualization of failure-guided retrieval}
For each task, we provide a visual example of how FERT uses a failed execution to select related successful fleet experience. In each figure (\Cref{fig:FR-vis-flower}-\Cref{fig:FR-vis-phonepack}), failure contexts are shown against a red background, followed by retrieved successful segments against green backgrounds.

\begin{figure}
    \includegraphics[width=0.86\linewidth]{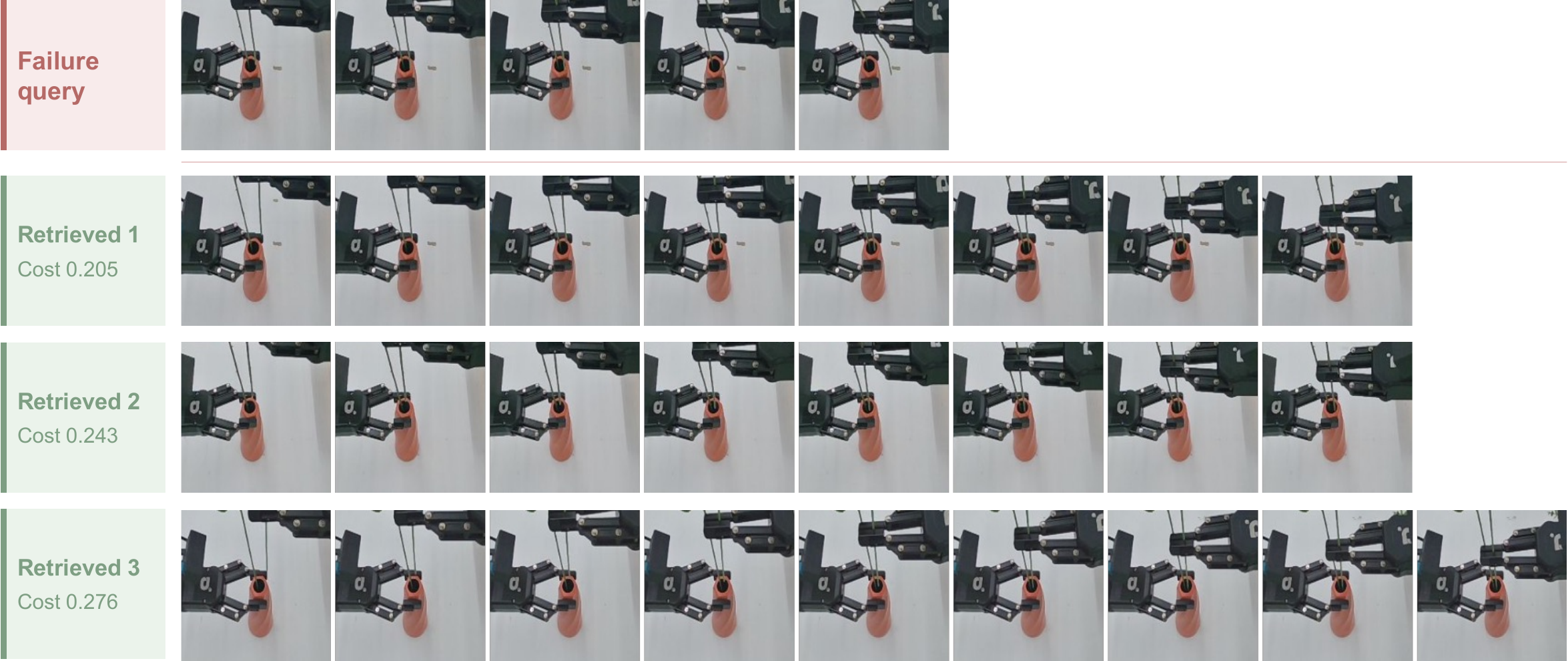}
    \caption{Visualize the retrieval effect in Flower Arrangement. The first row shows a clip of the robot failing to insert the second flower, while the next three rows show three matched successful clips. FERT effectively matched the clip of successfully inserting the second flower.}
    \label{fig:FR-vis-flower}
\end{figure}
\begin{figure}
    \includegraphics[width=0.68\linewidth]{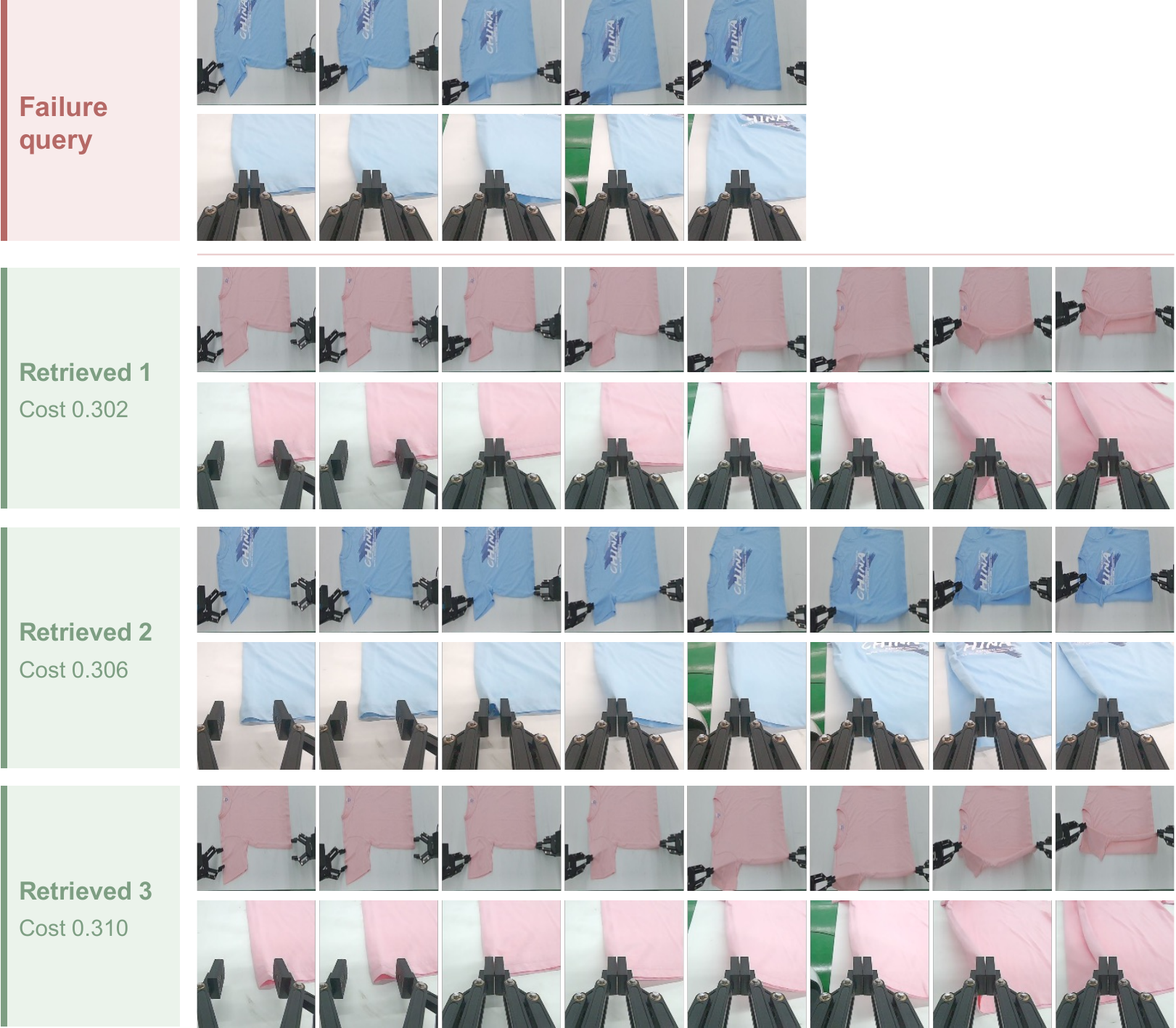}
    \caption{Visualize the retrieval effect in Clothes Folding. The top row shows a failure segment in which the robot attempts to fold a T-shirt in half but fails to grasp its right corner with the right gripper. The following rows below show the corresponding retrieved segments. With an appropriate feature representation, FERT retrieves the same critical step across clothes of different colors.}
    \label{fig:FR-vis-clothes}
\end{figure}
\begin{figure}
    \includegraphics[width=1\linewidth]{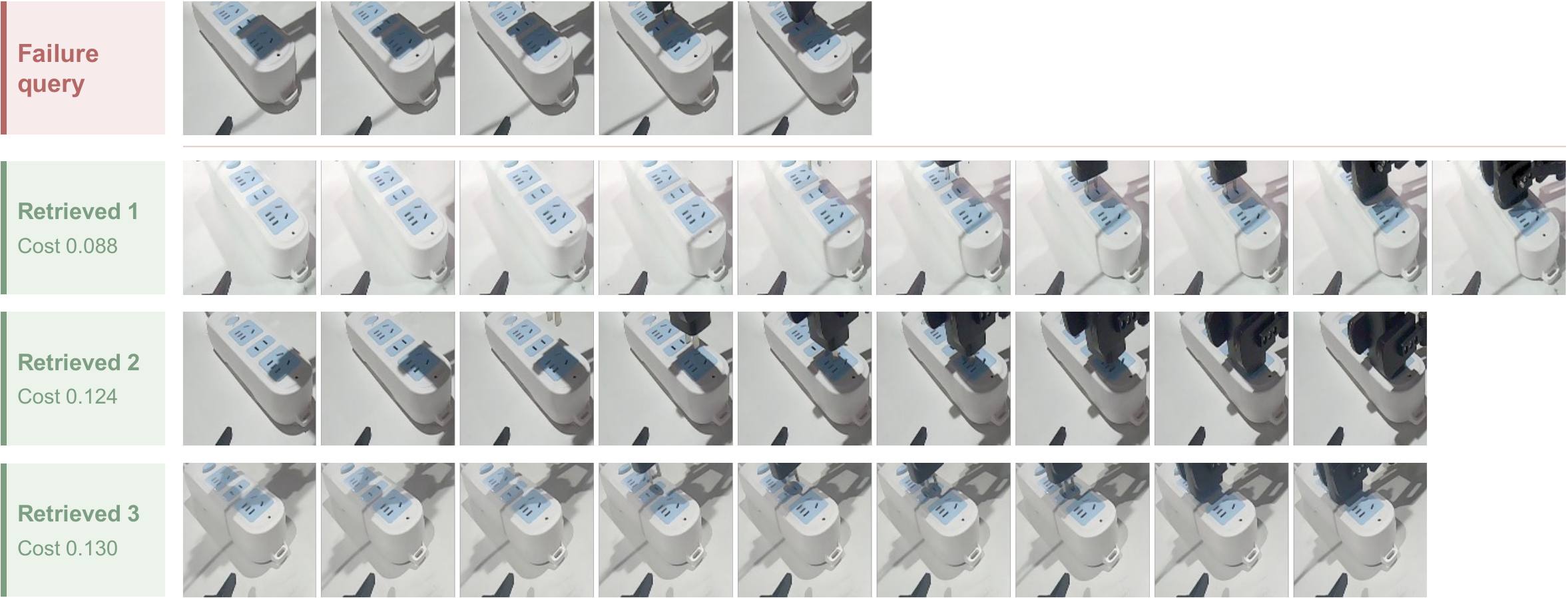}
    \caption{Visualize the retrieval effect in Phone Charging. It illustrates the most easily mistaken step in the phone charging: inserting the charger plug into the socket, which requires high precision.}
    \label{fig:FR-vis-charger}
\end{figure}
\begin{figure}
    \includegraphics[width=0.94\linewidth]{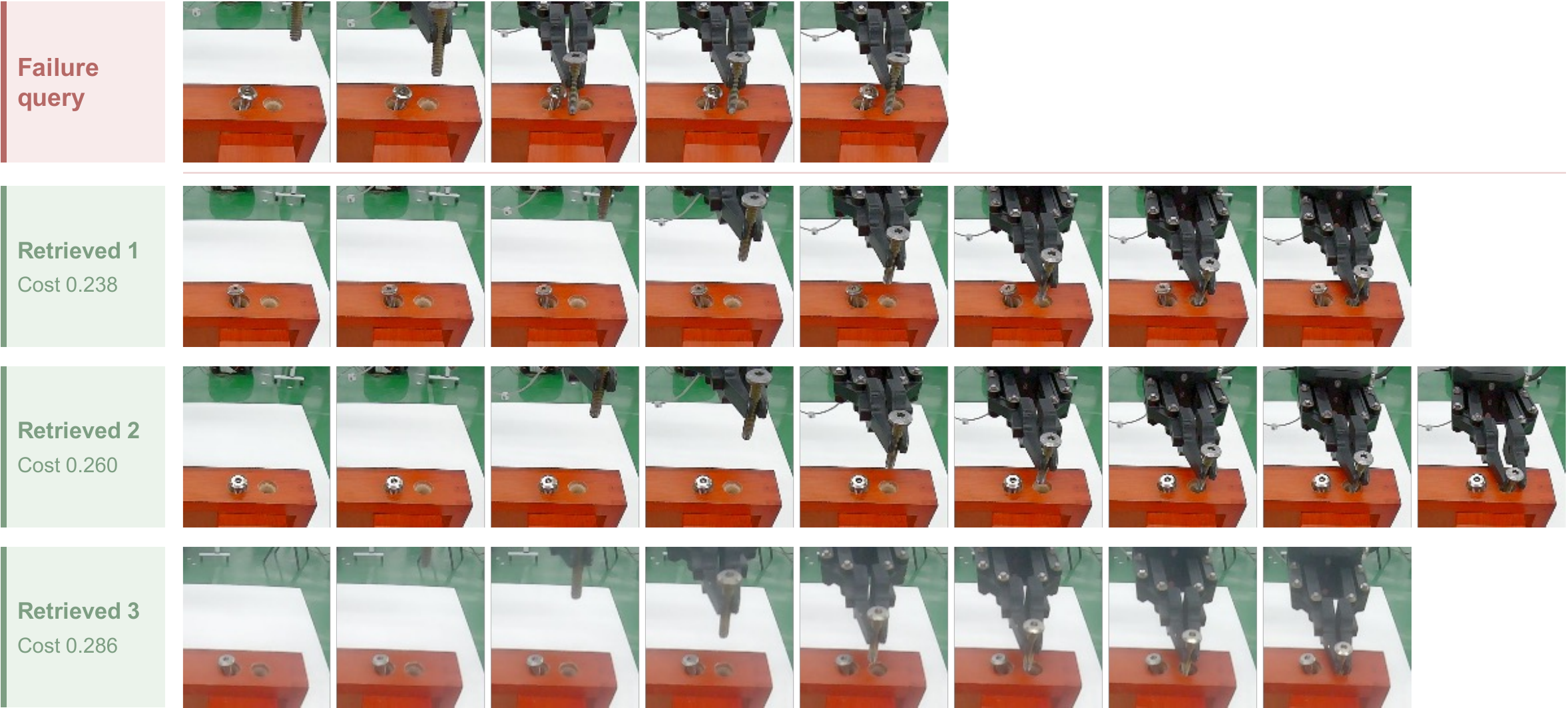}
    \caption{Visualize the retrieval effect in Stool Assembly. The figure shows a failed attempt to place a screw into a screw hole alongside the corresponding retrieved successful segments.}
    \label{fig:FR-vis-tooluse}
\end{figure}
\begin{figure}
    \includegraphics[width=0.75\linewidth]{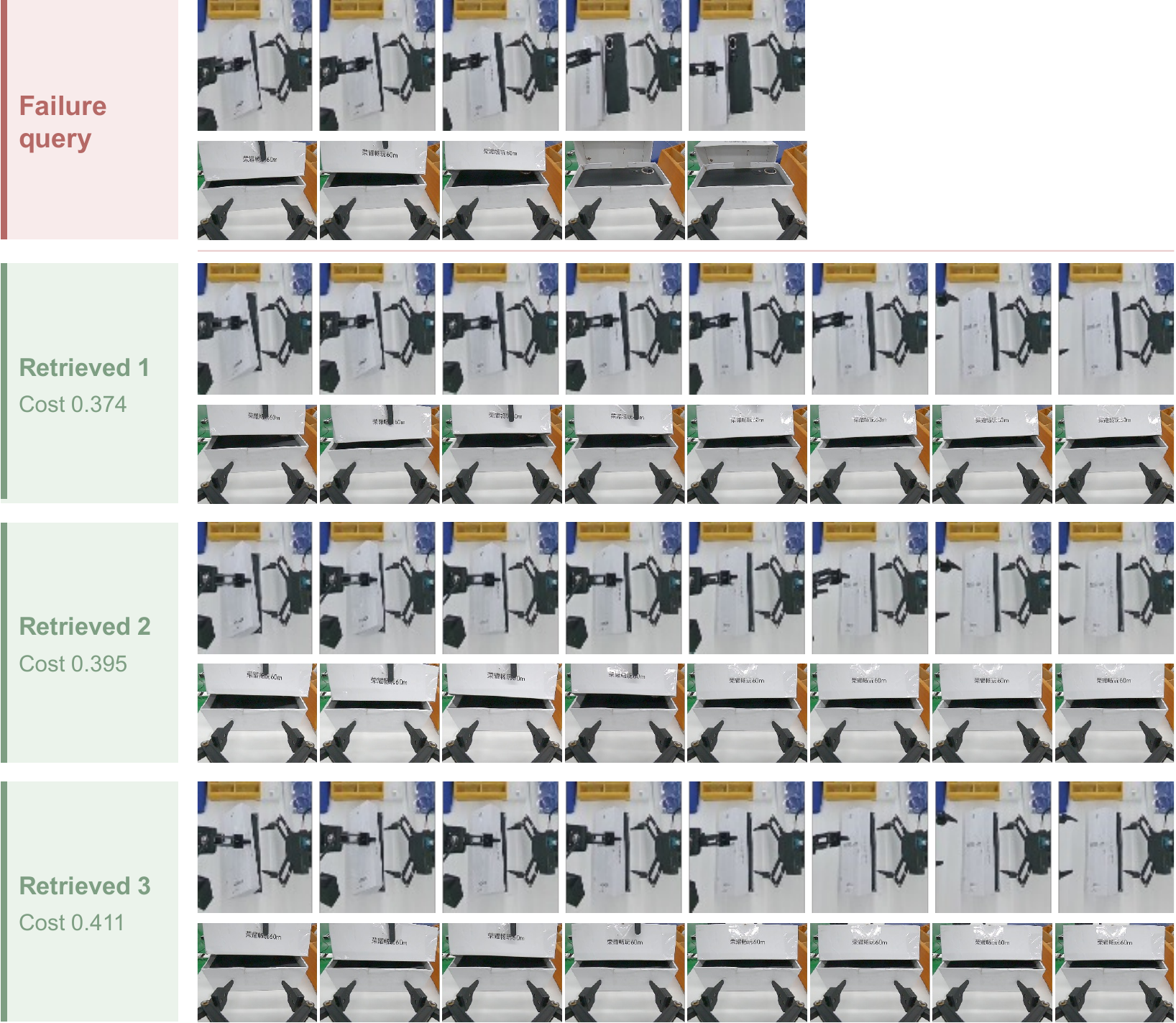}
    \caption{Visualize the retrieval effect in Phone Packaging. The figure illustrates the lid-placement stage: placing the lid over the box at an angle. This step requires high precision despite limited visibility. The top two rows show a failed attempt in which the left arm repeatedly adjusts the lid’s angle, but the lid still topples over. The three examples below show successful attempts.}
    \label{fig:FR-vis-phonepack}
\end{figure}

\end{document}